\documentclass{article}

\PassOptionsToPackage{numbers, compress}{natbib}
\usepackage[preprint]{main_arxiv_neurips}
\usepackage[utf8]{inputenc} % allow utf-8 input
\usepackage{amsmath}
\usepackage{amssymb}
\usepackage{amsthm}
\usepackage[T1]{fontenc}    % use 8-bit T1 fonts
\usepackage{hyperref, cleveref} % hyperlinks
\crefname{assumption}{assumption}{assumptions}
\usepackage{url}            % simple URL typesetting
\usepackage{booktabs}       % professional-quality tables
\usepackage{amsfonts}       % blackboard math symbols
\usepackage{nicefrac}       % compact symbols for 1/2, etc.
\usepackage{microtype}      % microtypography
\usepackage{xcolor}   
\usepackage{mathtools}
\usepackage{mdframed}
\usepackage{tcolorbox}
\usepackage{subfigure}
\mdfsetup{hidealllines=true} 

\usepackage{comment}
\usepackage{tabularx}
\usepackage{algorithm}
\usepackage{algorithmic}

\newcommand{\nsum}{\sum_{i=1}^N}

\newcommand{\knsum}{\sum_{k,i}}
\newcommand{\pknsum}{\sum_{k',i'}}
\newcommand{\pdl}[1]{#1-\tau(#1, i')}
\newcommand{\psum}{\sum_{i'=1}^N}

\newcommand{\dl}[1]{#1-\tau(#1, i)}
\newcommand{\w}[1]{w_{#1}}
\newcommand{\wopt}{w_*}
\newcommand{\sg}{\tilde{\nabla}}
\newcommand{\hessian}{\nabla^2 f}
\newcommand{\norm}[1]{\left\| #1 \right\|}
\newcommand{\EE}[2][]{\mathbb{E}_{#1}\left[#2\right]}
\newcommand{\PP}{\mathbb{P}}
\newcommand{\calA}[1]{\mathcal{A}_{#1}}
\newcommand{\calB}[1]{\mathcal{B}_{#1}}
\newcommand{\calC}[1]{\mathcal{C}_{#1}}
\newcommand{\calD}[1]{\mathcal{D}_{#1}}
\newcommand{\calF}[1][t]{\mathcal{F}_{#1}}
\newcommand{\calH}[1][]{\mathcal{H}_{#1}}
\newcommand{\calI}[1][]{\mathcal{I}_{#1}}
\newcommand{\calQ}[1][]{\mathcal{Q}_{#1}}
\newcommand{\calS}[1][]{\mathcal{S}_{#1}}
\newcommand{\calT}[1][]{\mathcal{T}_{#1}}
\newcommand{\calU}[1][]{\mathcal{U}_{#1}}
\newcommand{\calV}[1][]{\mathcal{V}_{#1}}
\newcommand{\dotp}[2]{\left<#1, #2\right>}
\newcommand{\dmaxt}[1][T]{\bar{d}_{\max, #1}}

\newcommand{\dbarmaxt}[1][T]{\bar{d}_{\max, {#1}}}
\newcommand{\taubar}{\bar{\tau}}
\newcommand{\taubart}[1][T]{\bar{\tau}_{#1}}

\newcommand{\taubarmaxt}[1][T]{\bar{\tau}_{\max, {#1}}}

\newcommand{\lmaxt}[1][T]{l_{\max,#1}}
\newcommand{\teta}[1][t]{\tilde{\eta}_{#1}}

\newcommand{\cO}{\mathcal{O}}
\newcommand{\cS}{\mathcal{S}}
\newcommand{\E}{\mathbb{E}}
\newcommand{\Ac}{\mathcal{A}}
\newcommand{\Dc}{\mathcal{D}}
\newcommand{\Bc}{\mathcal{B}}
\newcommand{\tDc}{\mathcal{\tilde D}}
\newcommand{\ber}{\text{Ber}}
\newcommand{\KL}{\text{KL}}

\renewcommand{\ss}{\text{\textbf{s}}}

\newcommand{\xinran}[1]{\textcolor{brown}{[Xinran: #1]}}

\newcommand{\algoname}[1][ ]{\texttt{MIFA#1}}
\newcommand{\fedavg}[1][ ]{\texttt{FedAvg#1}}
\newcommand{\fedprox}[1][ ]{\texttt{FedProx#1}}
\newcommand{\scaffold}[1][ ]{\texttt{SCAFFOLD#1}}
\newtheorem{theorem}{Theorem}[section]
\newtheorem{definition}{Definition}[section]
\newtheorem{assumption}{Assumption}
\newtheorem{lemma}{Lemma}[section]
\newtheorem{corollary}{Corollary}[section]
\newtheorem{proposition}{Proposition}[section]
\newtheorem{remark}{Remark}[section]

\DeclarePairedDelimiterX{\inp}[2]{\langle}{\rangle}{#1, #2}
\DeclarePairedDelimiterX{\abs}[1]{\lvert}{\rvert}{#1}
\DeclarePairedDelimiterX{\cbr}[1]{\{}{\}}{#1} % curly bracket
\DeclarePairedDelimiterX{\rbr}[1]{(}{)}{#1} % round bracket
\DeclarePairedDelimiterX{\sbr}[1]{[}{]}{#1} % square bracket

\allowdisplaybreaks
\title{Fast Federated Learning in the Presence of\\
Arbitrary Device Unavailability}

\author{%
Xinran Gu\thanks{Equal contribution.}\\
  Department of Industrial Engineering\\
  Tsinghua University\\
  \texttt{gxr17@mails.tsinghua.edu.cn} \\
  \And
  Kaixuan Huang\footnotemark[1]\\
  ECE\\
  Princeton University\\
  \texttt{kaixuanh@princeton.edu} \\
  \AND
  Jingzhao Zhang\\
  EECS\\
  Massachusetts Institute of Technology \\
  \texttt{jzhzhang@mit.edu} \\
  \And
  Longbo Huang\\
  IIIS\\
  Tsinghua University\\
  \texttt{longbohuang@tsinghua.edu.cn} \\
}

\begin{document}

\maketitle

\begin{abstract}
Federated Learning (FL) coordinates with numerous heterogeneous devices to collaboratively train a shared model while preserving user privacy. Despite its multiple advantages, FL faces new challenges. One challenge arises when devices drop out of the training process beyond the control of the central server. In this case, the convergence of popular FL algorithms such as FedAvg is severely influenced by the straggling devices. To tackle this challenge, we study federated learning algorithms under arbitrary device unavailability and propose an algorithm named Memory-augmented Impatient Federated Averaging (MIFA). Our algorithm efficiently avoids excessive latency induced by inactive devices, and corrects the gradient bias using the memorized latest updates from the devices. We prove that MIFA achieves minimax optimal convergence rates on non-i.i.d.~data for both strongly convex and non-convex smooth functions. 
We also provide an explicit characterization of the improvement over baseline algorithms through a case study, and validate the results by numerical experiments on real-world datasets.
\end{abstract} 

\section{Introduction}
Federated learning is a machine learning setting in which a central server coordinates with a large number of devices to collectively train a shared model \citep{mcmahan2017communication,smith2017federated,konevcny2015federated,sattler2019robust,MLSYS2020_38af8613, Li2020On}. Practical advantages of this training scheme are mainly twofold. First, each device keeps the private data locally and hence preserves its data privacy. Second, federated learning can make use of idle computing resources and lower computation costs. Although federated learning successfully scales up with data sizes and accelerates training via more affordable computing power \citep{yu2019linear, wang2019cooperative,stich2019local}, the collaborative setup leads to new challenges due to  large variations among individual computing devices. Our work aims to formulate and investigate the impact of device variations on FL from an optimization perspective.

In FL, a device can differ from its peers in multiple aspects \citep{kairouz2019advances, MLSYS2020_38af8613}. \emph{First}, the data distribution and local task can be different among devices. To address the data variation, non-i.i.d. objective models were proposed and analyzed by~\citep{Li2020On, pmlr-v119-karimireddy20a, khaled2020tighter,yang2021achieving, MLSYS2020_38af8613}. We follow this line of work and formulate our optimization objective as a sum of stochastic functions on individual devices (See Eqn.~\eqref{eq:obj}).

A \emph{second} variation among devices is caused by different computing and communication speeds. One natural way to formulate the variation in computation speeds is to allow asynchronous updates and model the  updates as delayed responses. Lots of novel research has studied the problem with different delay models, e.g.,
\citep{stich2020error,arjevani2020tight,NIPS2015_452bf208,feyzmahdavian2016asynchronous,zhang2020taming,agarwal2012distributed, aytekin2016analysis, glasgow2020asynchronous}. However, the delayed setup assumes that all devices make roughly the same number of (delayed) responses in the end. This behavior may deviate largely from the FL practice, where each device, e.g., personal cell phones, can have very different active duration when participating in the FL training, and hence make different numbers of responses. For this reason, our work aims to address this \emph{third}  discrepancy among devices caused by individual availability patterns.

The \emph{third }device heterogeneity caused by different availability patterns is less studied in optimization for federated learning problems. In this model, instead of making a delayed response, devices can abort the training halfway, e.g., due to battery level, incoming calls, etc, and fail to return their responses upon the central server's requests \citep{mcmahan2017communication, MLSYS2019_bd686fd6, kairouz2019advances}. To handle missing responses, researchers propose algorithms where the central server may collect responses from only a fraction of the devices and make updates \citep{pmlr-v119-karimireddy20a, mcmahan2017communication, yang2021achieving, MLSYS2020_38af8613, Li2020On, kairouz2019advances}.

Previous works on collecting responses from a fraction of devices can be divided into two categories. When the response distribution is known, one could collect only the fastest responses and re-weight according to their response probability \citep{kairouz2019advances,Li2020On}. This model can be restrictive, as in practice, the exact distribution may not be available and may evolve. Another line of work  assumes that the server can arbitrarily decide and sample a set of devices to collect responses accordingly in every communication round  \citep{pmlr-v119-karimireddy20a, mcmahan2017communication, yang2021achieving, MLSYS2020_38af8613}. This model does not require knowing the response possibility. However, the response time can be very long if the selected subset contains unavailable devices.

In this work, we address the above limitations by studying federated learning in the presence of arbitrary device unavailability. Within this practical setup, we propose an algorithm that automatically adapts to the underlying unavailability and allows patterns of the device unavailability to be non-stationary and even adversarial. Furthermore, our algorithm can achieve optimal convergence rates in the presence of device inactivity and automatically reduce to best-known rates if all devices are active. Our contributions are summarized as follows.

\begin{itemize}
    \item  We investigate the federated learning problem with a practical formulation of device participation, which does not require each device to be online according to an  (either known or unknown) distribution. 
    \item We propose the \emph{Memory-augmented Impatient Federated Averaging} (\algoname[]) algorithm that is agnostic to the availability pattern. It efficiently avoids excessive latency induced by inactive devices, successfully exploits the information about the descent direction in stale and noisy gradients, and corrects the gradient bias using the memorized latest updates. 
    \item We prove that \algoname[] achieves minimax optimal convergence rates $\mathcal{O}\left(\tfrac{\taubart+1}{N KT}\right)$ for smooth, strongly convex functions, and $\mathcal{O}\left(\textstyle\sqrt{\tfrac{\bar{\nu}+1}{NKT}}\right)$ for smooth, non-convex functions
    (see definitions in Sections~\ref{sec:setup}, \ref{sec:cvx} and \ref{sec:non_cvx}), and establish matching lower bounds.
    \algoname[] also achieves optimal convergence rates in the degenerated case when all devices are active.
    \item We provide an explicit characterization of the improvement over baseline algorithms through a case study and empirically verify our results on real-world datasets.
\end{itemize}

\section{Related work}\label{sec:related}

\paragraph{Federated learning.}

Federated Averaging (\fedavg[]) was first proposed in \citep{mcmahan2017communication}. \citep{Li2020On,khaled2020tighter,pmlr-v119-karimireddy20a,yang2021achieving} provided convergence analysis for \fedavg[] on non-i.i.d. data and quantified how data heterogeneity degrades the convergence rate. Several variants of \fedavg[] were designed to deal with data heterogeneity. \fedprox[] \citep{MLSYS2020_38af8613} adds a proximal term to local objective functions, while \texttt{FSVRG} \citep{konevcny2016federated} and \scaffold[] \citep{pmlr-v119-karimireddy20a} employ variance reduction techniques.

One line of work focused on variations in computation capabilities among devices~\citep{wang2020tackling,reisizadeh2020straggler,xie2019asynchronous}. These models assume that responses are delayed but not missing. To address the missing response, some work assumes that the server can actively sample a subset of devices to respond \citep{pmlr-v119-karimireddy20a, mcmahan2017communication, yang2021achieving, MLSYS2020_38af8613} or that the pattern of device availability is known \citep{Li2020On, kairouz2019advances,eichner2019semi}. These results do not generalize to adversarial inactive patterns. \citep{ruan2021towards} discussed the impact of device inactivity on convergence but their proposed algorithm diverges if there exists an inactive device in each round of communication. 
However, our setup allows adversarial patterns under certain non-distributional assumptions (see Section~\ref{sec:cvx}) while our proposed algorithm still achieves convergence.

\paragraph{Asynchronous distributed optimization.} Our work is related to literature in the field of traditional asynchronous distributed optimization in that our proposed algorithm uses stale gradients. The problem setup for asynchronous distributed algorithms can be divided into two categories \citep{glasgow2020asynchronous}. One is the shared-data (i.i.d.) setting, where all workers can access the whole dataset. In this setting, the local gradient is an unbiased estimator of the global gradient \citep{stich2020error,arjevani2020tight,NIPS2015_452bf208,feyzmahdavian2016asynchronous,zhang2020taming,agarwal2012distributed}. In contrast, we assume each worker has non-i.i.d. data, and hence the local stochastic gradient can not be viewed as an unbiased estimator of the global gradient.

The other less studied setting in distributed optimization is the distributed-data setting (non-i.i.d.), where data are partitioned among workers. Specifically, \citep{aytekin2016analysis} proposed an asynchronous incremental aggregated gradient algorithm that uses buffered gradients to update the global model. Unlike our setup, this algorithm evaluates full local gradients, performs only one local step, and was analyzed under the bounded delay assumption. \citep{glasgow2020asynchronous} models the delay as stochastic and assumes that the server has knowledge of the distribution, but our formulation is distribution-free. \citep{basu2019qsparse} allows workers to perform multiple local steps and communicate with the server at different times, but the authors assume that all workers are available and compute at the same rate.

\paragraph{Comparison with an independent work.}
While preparing the manuscript, we were unaware of
 an {independent} work \citep{yan2020distributed} that investigated the same setup and proposed a similar algorithm called \texttt{FedLaAvg}. 
 Their main theorem established the convergence rate of $\cO \left(\sqrt{\tfrac{\nu_{\max}}{N^{0.5} T}}(G^2 + \sigma^2)\right)$ for smooth and non-convex problems, where $G^2$ is the uniform upper bound for the squared norm of stochastic gradients and $\nu_{max}$ is the maximum number of inactive rounds.
In comparison, we prove the minimax optimal rate of $\cO(\sqrt{\tfrac{\bar{\nu}}{N K T}}\sigma^2)$ without the bounded gradient assumption, also improving $\nu_{\max}$ to $\bar{\nu}$. Furthermore, our result achieves a linear speedup in $N$ and $K$.

Apart from non-convex functions, we also derive the minimax optimal rates for strongly convex smooth functions under the mild assumption that allows for arbitrary and unbounded number of inactive rounds. Both of our results achieve linear speedups in terms of $N$ and $K$, and automatically recover the best-known rates of $\fedavg[]$ when all devices are active. We also show that our proposed algorithm achieves acceleration over unbiased baseline algorithms in the presence of stragglers.

\section{Problem Setup}\label{sec:setup}

We consider optimizing the following problem in a Federated Learning setting:
\begin{align}\label{eq:obj}
        \min_{w\in \mathbb{R}^d} f(w) := \textstyle \frac{1}{N}  \nsum f_i(w) := \frac{1}{N}  \nsum \E_{\xi_i} [f_i(w, \xi_i)], 
\end{align}
where $w$ is the optimization variable, e.g., parameters of a machine learning model, 
$N$ is the number of participating devices,  $f_i$ is the local loss function on device $i$, and $\xi_i$ describes the randomness in local data distribution.

In the ideal federated learning setup (see Figure~\ref{fig:setup} (a)), all devices return responses within similar time, and hence the central server collects all the local updates. In this case, the computation cost is usually measured by the number of local stochastic oracle evaluations, which is proportional to the number of rounds.  In a delayed FL setup  (see Figure~\ref{fig:setup} (b)), devices are always active upon the central server's request but may return responses with a delay. Here, all devices return almost the same number of responses in the long term.

As we discussed, the above setups do not depict a real-world scenario in which a device can have a longer inactive duration than active duration. In such cases, the communication interval is much longer than the local computation time required for each update, and each device generates an unequal number of responses \citep{mcmahan2017communication, kairouz2019advances}. This motivates our setup in Figure~\ref{fig:setup} (c).

\begin{figure}
\vspace{-0.15in}
\begin{center}
    \subfigure[ Ideal setup]{
        \includegraphics[width=0.315\textwidth]{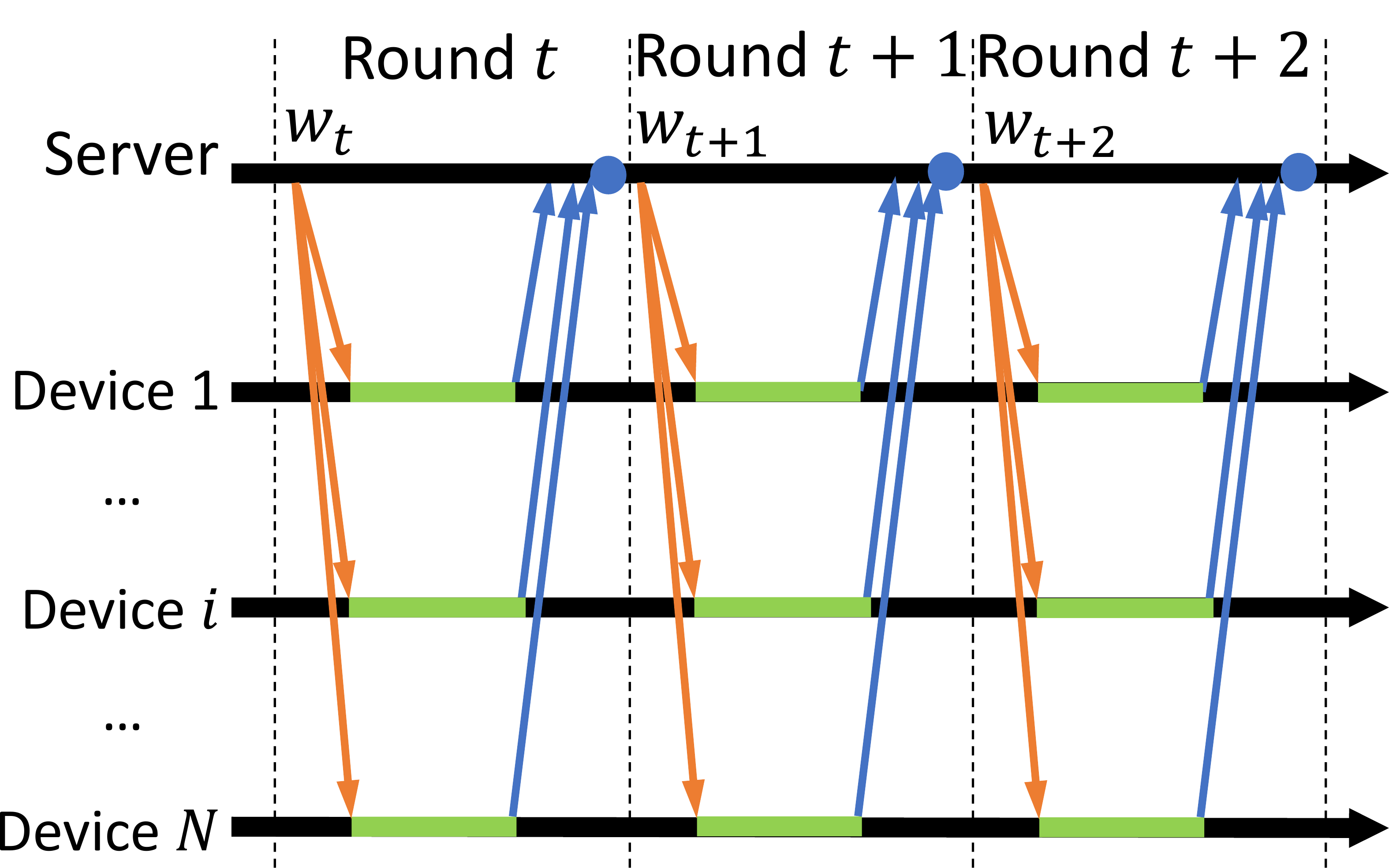}
        \vspace{-0.125in}
    }
    \subfigure[ Delayed setup ]{
        \includegraphics[width=0.315\textwidth]{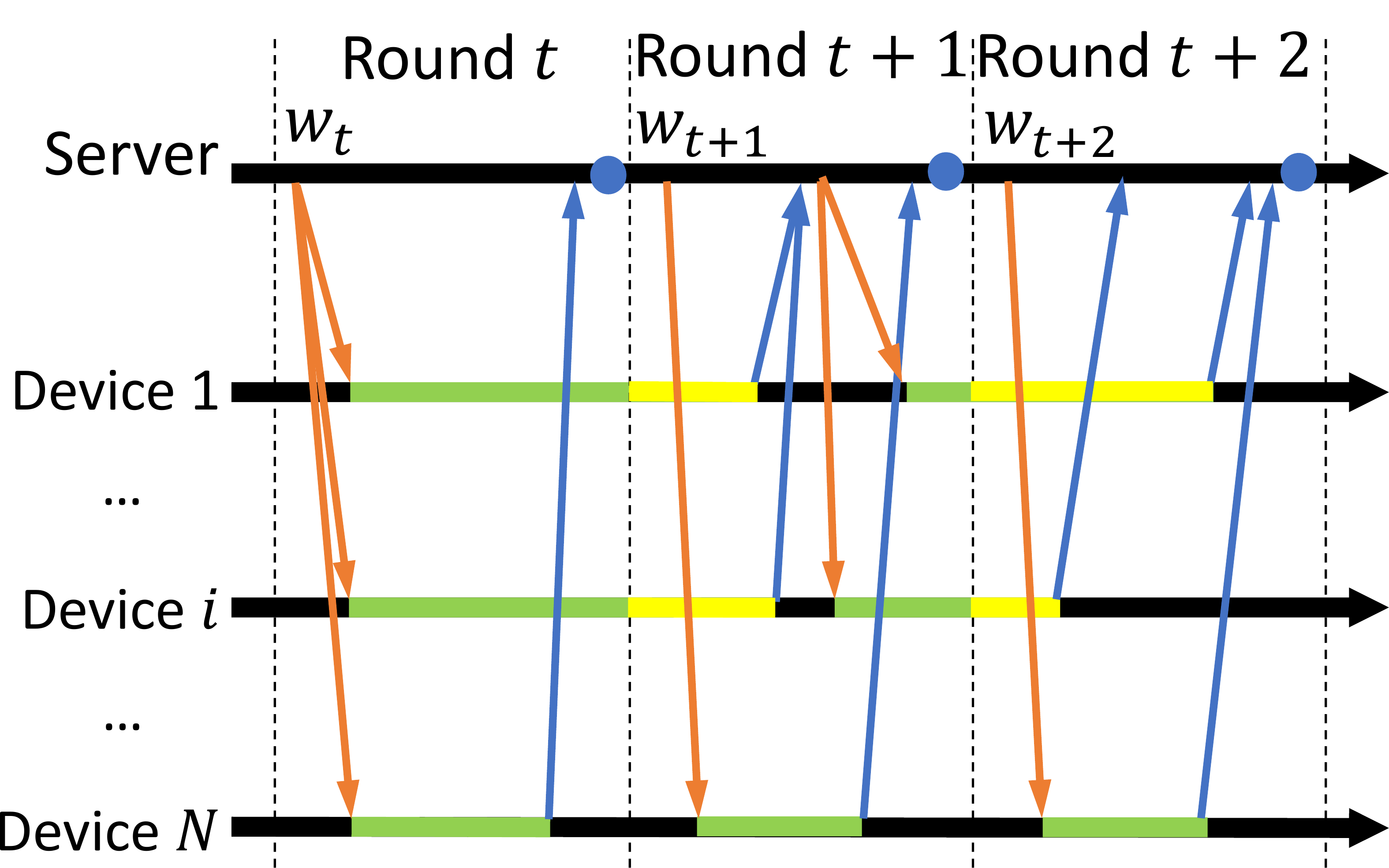}
        \vspace{-0.125in}
    }
    \subfigure[Our setup]{
        \includegraphics[width=0.315\textwidth]{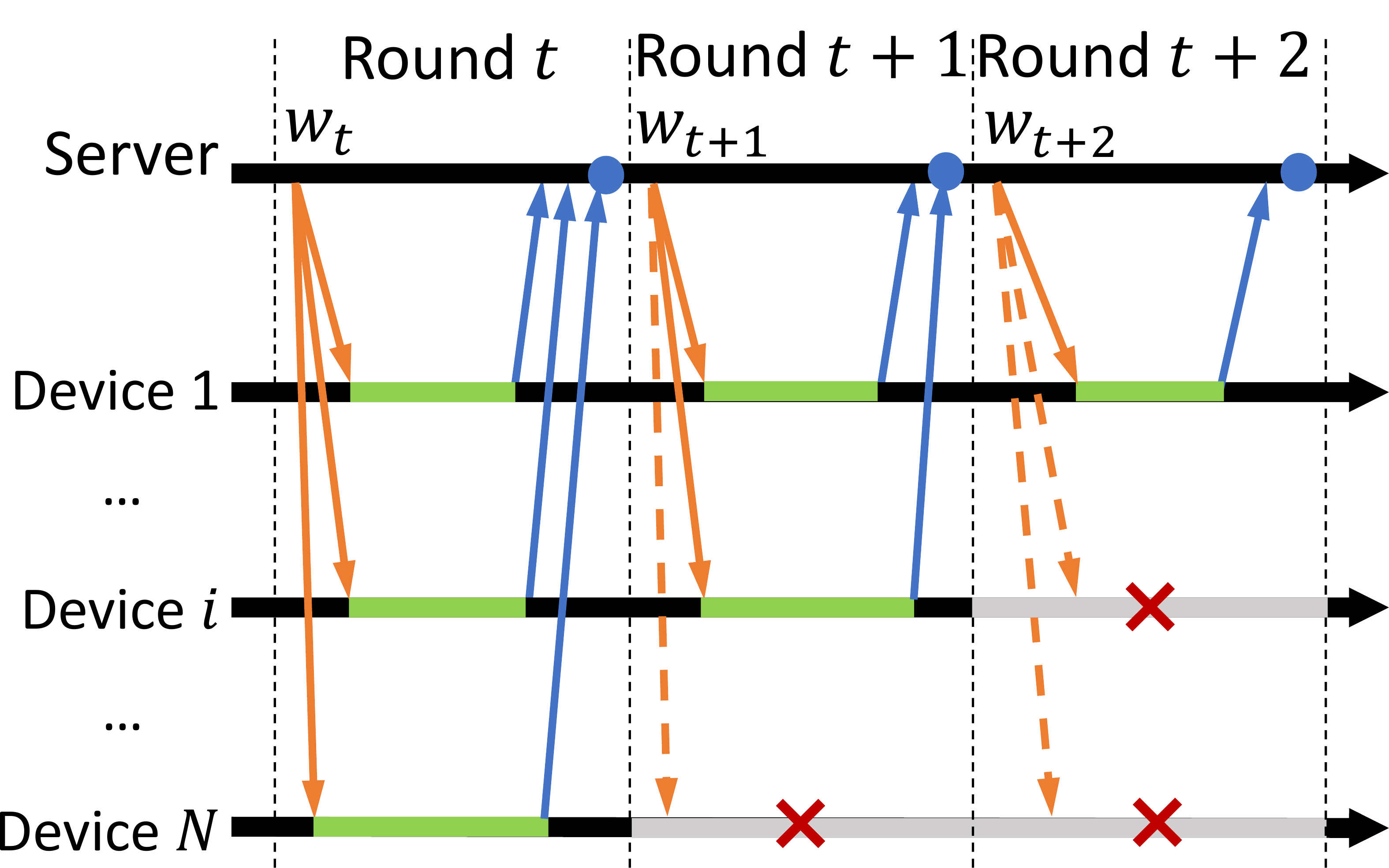}
        \vspace{-0.125in}
    }
        \includegraphics[width=1\textwidth]{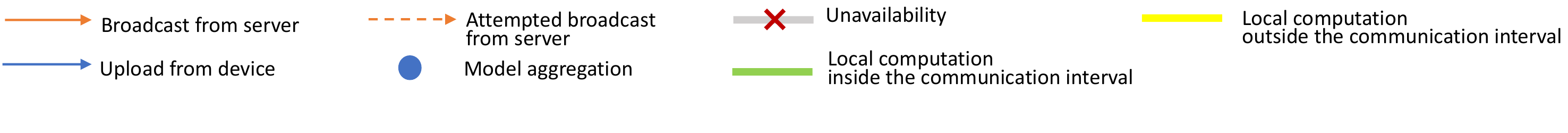}
    \vspace{-0.4in}
    \caption{ An illustration of setup. (a) Ideal setup: all devices return their responses within similar time. (b) Delayed setup: all devices are available, but may return responses with a delay. (c) Our setup: devices can be unavailable arbitrarily, and the communication interval is long enough for active devices to return responses. }
        \label{fig:setup}
    \vspace{-0.3in}
\end{center}
\end{figure}
 
In our proposed setup, we use $t$ to index the global communication rounds. We say a device \emph{participates} or is \emph{active} at round $t$ if it can complete the computation task and send back the update at the end of round $t$. We define $\mathcal{A}(t)$ as the set of all active devices at round $t$. Notice that we make no assumptions on the distribution of the participation patterns of devices and allow them to be arbitrary.

Directly applying  $\fedavg[]$ to the proposed setup can be problematic due to the existence of inactive devices. To accommodate for inactive devices, we discuss three natural variants of $\fedavg[]$ and their limitations. The detailed algorithms can be founded in Appendix~\ref{sec:baseline}.

\begin{itemize}
    \item Biased \fedavg[]. At each communication round, the global model is updated with a direct average of local updates from the active devices. This naive approach induces bias when data distribution and response patterns vary among devices.
    \item \fedavg[] with device sampling. The server selects a subset of $S$ devices randomly without replacement, and then waits till all devices in the subset $\cS$ respond. This is how original \fedavg[] \citep{mcmahan2017communication} addresses device unavailability. Note that over $T$ communication rounds, the global model is updated less than $T$ times due to waiting. This approach is prone to stragglers and we refer the readers to \Cref{sec:case_study} for a detailed discussion.
    \item \fedavg[] with importance sampling \citep{Li2020On,glasgow2020asynchronous,kairouz2019advances}. The local updates from the active devices are weighted by the reciprocal of the participation probabilities to avoid bias. This approach is only applicable when the response of each device is i.i.d. over rounds and it requires the knowledge of participation probabilities. 
\end{itemize}

\section{Memory-augmented Impatient Federated Averaging (MIFA)}\label{sec:alg}

In this section, we introduce our algorithm --- \emph{Memory-augmented Impatient Federated Averaging} (\algoname[]). \algoname[] maintains an update-array $\{G^i\}$ in the memory that stores the latest updates for all devices. As the name suggests, \algoname[] has two components. First, the algorithm is impatient and avoids waiting for any specific device when facing heterogeneous devices with arbitrary availability. Second, the algorithm augments the received updates of the active devices with the stored updates of the inactive devices to perform averaging.

Specifically, at the beginning of round $t$, the server broadcasts the latest model parameter $w_t$ to all active devices $\mathcal{A}(t)$. After receiving $w_t$, each active device, say, the $i$-th device, sets $\w{t, 0}^i = \w{t}$ and performs $K$ steps of SGD with respect to the local objective function to get $\w{t, K}^i$:
\begin{align*}
    \w{t, k+1}^i = \w{t, k}^i - \eta_t \sg f_i(\w{t, k}^i), \  k=0, \cdots, K-1,
\end{align*}
where $\eta_t$ is the learning rate and $\sg f_i(\w{t,k}^i)$ is the stochastic gradient evaluated on device $i$. Next, the server stores the received update $\frac{1}{\eta_t}(\w{t}-\w{t, K}^i )$ in $G^i$. Denote by $\{ G^i_t \}$ the update-array after round $t$, then we have
\[
    G^i_t =\left\{ 
    \begin{array}{ll}
        G^i_{t-1}, & \text{ if } i \notin \mathcal{A}(t),\\
        \frac{1}{\eta_t}(\w{t}-\w{t, K}^i ), & \text{ if } i \in \mathcal{A}(t).
    \end{array}
    \right.
\]
At the end of round $t$, the server updates the global model with the average of $\{G_t^i\}$ (line 9). In other words, our algorithm \algoname[] updates the model with the latest available accumulated gradients for all devices.

\begin{algorithm}
    \caption{ Memory-augmented Impatient Federated Averaging ($\mathtt{MIFA})$}\label{algo: our algorithm}
    \begin{minipage}{.5\textwidth}
    \begin{algorithmic}[1]
    \STATE \textbf{Input}: initial $w_1$, learning rate $\{\eta_t\}$
        \STATE  \textbf{Server executes}:
        \STATE initialize $ G^i\leftarrow 0, i \in [N]$
        \FOR{$t=1, \cdots, T-1$}
            \STATE broadcast $w_t$ to all active devices $i\in \mathcal{A}(t)$
            \FOR{each active device $i$}
            \STATE $G^i \leftarrow $ DeviceUpdate($i,\w{t}, \eta_t$)
            \ENDFOR
            \STATE $\w{t+1}\leftarrow \w{t} - \frac{\eta_t}{N}\nsum G^i$
        \ENDFOR
    \end{algorithmic}
    \end{minipage}
        \begin{minipage}{.48\textwidth}
    \begin{algorithmic}[1]
        \STATE \textbf{DeviceUpdate}($i, w_t, \eta_t$):
            \STATE  $\w{t, 0}^i\leftarrow \w{t}$
            \FOR{local step $k=0, \cdots, K-1$}
                \STATE compute stochastic gradient $\sg f_i(\w{t,k}^i)$
                \STATE $\w{t,k+1}^i \leftarrow \w{t,k}^i - \eta_t \sg f_i(\w{t,k}^i)$
            \ENDFOR
            \STATE Return $\frac{1}{\eta_t}(\w{t}-\w{t,K}^i)$ to the server
    \end{algorithmic}
    \end{minipage}
\end{algorithm}

\algoname[] efficiently progresses without waiting for inactive devices  and re-uses their latest updates as the surrogate for missing responses. Being impatient accelerates  convergence, whereas memory augmentation corrects the update bias. Our algorithm differs from asynchronous algorithms in traditional distributed optimization \citep{stich2020error,arjevani2020tight,NIPS2015_452bf208,feyzmahdavian2016asynchronous,zhang2020taming,agarwal2012distributed, glasgow2020asynchronous,basu2019qsparse} in that we utilize the \emph{noisy} updates of inactive devices \emph{more than once} to avoid biasing against stragglers. In the following part of the paper, we show that \algoname[] successfully exploits information about the descent direction contained in the stale and noisy gradients.

\textbf{Discussion on implementation.} In practice, to implement \algoname[], the server needs to maintain a huge array to store the latest update for each device, which scales with the model size and the total number of devices. 
To avoid exhausting the server's memory, one strategy is to distribute the memory consumption among devices. 
Specifically, each device, say the $i$-th, stores its previous update ${G}^{i}_{t_i'}$ computed at round $t'_i$ in its local memory. When it becomes active and computes $G^i_t$, the device sends $G^i_t -{G}^{i}_{t_i'}$ to the server, which is the difference between the current update and the previous one. In this case, the server only needs to maintain the average $\bar{G}$ in the memory and updates it by $\bar{G}_t = \bar{G}_{t-1} +\frac{1}{N}\sum_{i\in \mathcal{A}(t)} (G^i_t -{G}^{i}_{t_i'})$ at round $t$. Then the server updates the global model by $w_{t+1} = w_t - \eta_t \bar{G}_t$.

\section{Convergence Analysis for strongly convex objective functions}
\label{sec:cvx}

In this section, we present the convergence results for \algoname[] on $\mu$-strongly convex $L$-smooth functions. 
Typical examples for the strongly convex case are $\ell_2$ regularized logistic regression and linear regression problems.

In order to capture how the unavailability of devices affects algorithm performance, we introduce the following notion to quantify the dynamics of devices in our setting.

\begin{definition}[Number of inactive rounds]
We define the \textbf{number of inactive rounds} of device $i$ at round $t$ as $\tau(t,i) =t- \max\{t'\mid t'\leq t, i\in \mathcal{A}(t')\}$, which is the difference between current round $t$ and the latest round when device $i$ is active.
\end{definition}

It can be seen that $\tau(t,i)=0$ if device $i$ is active at round $t$ and $\tau(t,i)=\tau(t-1,i)+1$ otherwise. Also, $t - \tau(t,i)$ is the latest round when the device $i$ is active. 
Next, we present the assumptions made for establishing our convergence theorem.

\begin{assumption}\label{assumption:smooth}
$f_1, \cdots, f_N$ are all $L$-smooth, i.e., for all $w$ and $v$, $f_i(v)\leq f_i(w) + \dotp{\nabla f_i(w)}{v-w} + \frac{L}{2}\norm{w-v}^2$.
\end{assumption}

\begin{assumption}\label{assumption:noise}
$\sg f_i(w)$ is an unbiased estimator of $\nabla f_i$ with variance bounded by $\sigma^2$, i.e.,  $\EE[\xi]{\sg f_i(w)}=\nabla f_i(w)$, $\EE[\xi]{\norm{\sg f_i(w) - \nabla f_i(w)}^2}\leq \sigma^2$.
\end{assumption}

\begin{assumption}\label{assumption:cvx}
$f_1, \cdots, f_N$ are all $\mu$-strongly convex: for all $w$ and $v$, $f_i(v)\geq f_i(w) + \dotp{\nabla f_i(w)}{v-w} + \frac{\mu}{2}\norm{w-v}^2$.
\end{assumption}

\begin{assumption}\label{assumption:delay}
There exists a constant $t_0 > 0$, such that for all $t \ge 1$ and $i\in [N]$, the number of inactive rounds of device $i$ at communication round $t$ satisfies $\tau(t,i) \leq t_0 + \frac{1}{b} t$, where $b=40\left(L/\mu \right)^{1.5}$.
\end{assumption}

Assumptions~\ref{assumption:smooth}, \ref{assumption:noise}, and \ref{assumption:cvx} are standard and common in the FL literature, e.g., \citep{Li2020On, pmlr-v119-karimireddy20a, khaled2020tighter, yang2021achieving, stich2019local}. 
In Assumption~\ref{assumption:noise}, we relax the bounded gradient assumption that is often required in prior work, e.g., \citep{basu2019qsparse, Li2020On,xie2019asynchronous,agarwal2012distributed}. 
Lastly, \Cref{assumption:delay} is a very  mild assumption on device availability, since it allows the number of inactive rounds to grow as $\mathcal{O}(t)$. In contrast, existing results on asynchronous updates mostly assume a bounded or fixed latency, e.g., \citep{basu2019qsparse,agarwal2012distributed,aytekin2016analysis,xie2019asynchronous, stich2020error,arjevani2020tight}.

We are now ready to present our first convergence result. 
Define $D=\frac{1}{N}\nsum \norm{\nabla f_i(\wopt)}^2$ to measure data dissimilarity, where $\wopt = \arg\min f(w)$ is the global optimum. Also, 
define $\taubart$ and $\tau_{\max,T}$ to be the average and  maximum numbers of inactive rounds $\tau(t,i)$ across all devices and rounds, respectively. That is, 
\[
\taubart =\frac{1}{N(T-1)}\sum_{t=1}^{T-1}\nsum \tau(t,i), \ \  \tau_{\max,T} = \max_{i \in [N]} \max_{1\leq t\leq T-1}\tau(t,i).
\]

The following theorem summarizes the performance of $\mathtt{MIFA}$ in this case. 
\begin{theorem}\label{theorem:maintheorem}
Assume that Assumptions~\ref{assumption:smooth} to~\ref{assumption:cvx} hold. Further assume that the device availability sequence $\tau(t,i)$ satisfies \Cref{assumption:delay} and $\tau(1,i)=0$ for all $i \in [N]$. By setting the learning rate $\eta_t = \frac{4}{\mu K(t+a)}$ with $a=\max\{100, 40t_0\}(L/\mu)^{1.5}$, after $T-1$ communication rounds, \algoname[] satisfies:
\begin{align*}
  \EE[\xi]{f(\overline{w}_T)}-f(\wopt)&=\mathcal{O}\left(\frac{\taubart+1}{\mu NKT}\sigma^2 + \frac{\tau_{\max,T}^2 A_1 + (K-1)^2A_2+A_3}{\mu^2 T^2}\right),
\end{align*}

where $\overline{w}_T$ is a weighted average of $\w{t}$ defined as: 
\begin{align*}
    \overline{w}_T &= \frac{1}{W_T}\sum_{t=1}^T (t+a-1)(t+a-2) w_t, \ 
    W_T = \sum_{t=1}^T (t+a-1)(t+a-2),
\end{align*}

and
$A_1 = L (D+L \sigma^2/\mu)$, 
  $A_2 = L(D/K^2+\sigma^2/K^3)$, $A_3  = t_0^2 L^3 \norm{\w{1}-\wopt}^2$.
\end{theorem}

Our results hold under Assumption~\ref{assumption:delay}, which allows for arbitrary device availability sequences with $\tau_{\max,T} = \mathcal{O}(T)$. However, for \algoname[] to converge, we require $\tau_{\max,T} = o(T)$ and  $t_0 = o(T)$.
When $T=\Omega(\frac{NK(\tau_{\max,T}^2+t_0^2)}{\taubar + 1})$, the first term dominates and the impact of the second $\mathcal{O}(1/T^2)$ term is negligible. In fact the first term in Theorem~\ref{theorem:maintheorem} is minimax optimal by our information-theoretic lower bound for the problem in the next proposition.
\begin{proposition}
\label{thm:lower}
Let $c_0 > 0$ be a universal constant. For any potentially randomized algorithm,  there exists a stochastic strongly convex problem satisfying Assumptions~\ref{assumption:smooth} to~\ref{assumption:cvx}, such that the output $w_T$ after T rounds of communication has expected sub-optimality lower bounded by
\begin{align*}
    \mathbb{E}[f(w_T) - f(w^*)] \ge c_0 \frac{\bar{\tau}_T \sigma^2}{\mu NKT}. 
\end{align*}
\end{proposition}
The proof is based on the observation that the number of gradient evaluation can scale inversely with $\bar{\tau}_T$ and that the oracle complexity is tight even for centralized stochastic optimization problems. The optimality of the first term in Theorem~\ref{theorem:maintheorem} is independent of the distributed or the FL setup.

The second term in Theorem~\ref{theorem:maintheorem} converges at the rate $\mathcal{O}(1/T^2)$ and consists of three parts, where the first part reflects the slowdown caused by device unavailability through $\tau_{\max,T}$, the second part shows the effect of multiple ($K>1$) local steps, and the third part tells how the initial error decreases.
\begin{remark}
When $\tau(t,i)=0$ for all $i$ and $t$, our setup reduces to FedAvg with full device participation, and we have $\taubart = 0$ and $\tau_{\max,T} = 0$. In this case, \Cref{theorem:maintheorem} yields bound $\mathcal{O}\left(\tfrac{\sigma^2}{\mu  KNT}+\frac{L(\sigma^2/K+D+L^2\norm{\w{1}-\wopt}^2)}{\mu^2 T^2}\right)$, matching the rate $\mathcal{O}\Big(\tfrac{\sigma^2\log T}{\mu  KNT}+\tfrac{L(\sigma^2/K+D)(\log T)^2}{\mu^2 T^2} + \mu \norm{\w{1}-\wopt}^2\exp(-\tfrac{\mu}{48L}T)\Big)$ in \citep{pmlr-v119-karimireddy20a} (Thm. V.  $B^2=2$,$\eta_g=1$) up to logarithmic terms. 
Besides, in the general case, our $\mathcal{O}(\tau_{\max,T}^2A_1 /T^2)$ term matches the last term in \citep{basu2019qsparse} (Cor. 5).
\end{remark}

\begin{remark}
Our analysis relies on the technical assumption that all devices respond in the first round. Intuitively, this is because we need at least one valid stochastic gradient evaluation for each device to get a complete picture of the global objective, or otherwise any update would be biased. In practice, this can be achieved by waiting for the updates from all devices on $\w{1}$ at the very beginning. 
\end{remark}

\subsection{Case Study: i.i.d. Bernoulli participation}
\label{sec:case_study}

Though our algorithm can be applied to non-stationary and non-independent response patterns, we show in this subsection that even in the simple i.i.d. Bernoulli participation scenario our algorithm can achieve considerable improvement compared to known algorithms. In particular, we consider a setup where each device becomes active independently with a fixed probability $p_i$. It serves as the first motivating example towards modeling the participation patterns of devices, and provides a clean view of how the heterogeneity of the device participation influences the Federated optimization algorithms. 

We will show that in this scenario, \Cref{assumption:delay} holds with high probability, and the terms involving the inactive rounds $\tau(t,i)$ in Theorem~\ref{theorem:maintheorem} can also be bounded. Furthermore, we theoretically demonstrate that algorithms such as \fedavg[] \citep{mcmahan2017communication} and \scaffold[] \citep{pmlr-v119-karimireddy20a}, which sample $S$ devices for each global update, are more prone to stragglers than our algorithm.
\begin{definition}
\label{def:bernoulli}
Assume that for all $i \in [N]$, the $i$-th device is assigned with a probability $p_i$. We say the participation of the devices follows \textbf{i.i.d. Bernoulli participation} model with participation probabilities $\{p_i\}$, if (1). at the first round, all devices are active, and (2). at round $t > 1$, device $i$ is active with probability $p_i$, which is independent of the history and other devices.
\end{definition}

Next theorem shows that under i.i.d. Bernoulli participation scenario, with high probability, $\tau(t,i)$ only grows logarithmically in $t$. Also \Cref{assumption:delay} holds for a mild choice of $t_0$. 
\begin{theorem}
\label{thm:bernoulli_delay}
For i.i.d. Bernoulli participation model defined in Definition~\ref{def:bernoulli}, given any $\delta > 0$, with probability at least $1-\delta$, we have the following holds for all $t \geq 1$ and $i \in [N]$ simultaneously,
\[
    \tau(t,i) \leq  \cO \Big( \frac{1}{p_i} (\log (Nt/\delta) + 1 ) \Big) .
\]
Furthermore, 
(1). \Cref{assumption:delay} holds true if $t_0 = \Omega \Big (\frac{1}{p_{min}} \log\frac{bN}{p_{min} \delta} \Big) $, where $p_{min} = \min \{ p_i \}$, and $b = 40 (L/\mu) ^{1.5} $;(2). $\tau_{\max,T}$ can be upper bounded as 
\[
    \tau_{\max,T} \leq \cO\Big(  \frac{1}{p_{min}} \cdot \big( \log (TN/\delta) + 1 \big) \Big).
\]
\end{theorem}
The next theorem provides a high probability upper bound for $\bar\tau_T$.
\begin{theorem}
\label{thm:bernoulli_tau_T}
For i.i.d. Bernoulli participation model defined in Definition~\ref{def:bernoulli}, given any $\delta > 0$ and $T > 1$, with probability at least $1-\delta$, we have 
\[
    \bar\tau_T \leq \Big( \frac{1}{N} \sum_{i=1}^N \frac{1}{p_i} \Big) \cdot \cO \Big(1 + \log\frac{1}{\delta} \Big).
\]
\end{theorem}

By Theorem~\ref{thm:bernoulli_delay} and Theorem~\ref{thm:bernoulli_tau_T}, we conclude that the dominant term of our convergence bound is $ \widetilde{\cO} \Big(  \frac{1}{N} \sum_{i=1}^N \frac{1}{p_i} \cdot \frac{\sigma^2}{\mu N K T }  \Big)$. Therefore, to achieve $\epsilon$ accuracy, the dominant term of the number of the required rounds is
\begin{equation}
    T_\epsilon^{(\algoname[])} = \widetilde{\cO} \Big(  \frac{1}{N} \sum_{i=1}^N \frac{1}{p_i} \cdot \frac{\sigma^2}{\mu N K \epsilon }  \Big). \label{eq:teps:ours}
\end{equation}
For both \fedavg[] and \scaffold[] that sample $S$ devices uniformly at random, \cite{pmlr-v119-karimireddy20a} (Thm I. \& III.) showed that the dominant term of the number of global updates needed to achieve $\epsilon$ accuracy is $ R_\epsilon = \widetilde{\cO} \Big( \frac{\sigma^2}{\mu S K \epsilon }  \Big)$. Notice that in our setting, to accomplish each global update, the server needs to wait for a few rounds for the $S$ devices to respond. Let $T(\cS)$ be the expected rounds for which the server needs to wait for the selected devices $\cS$ to be active. Then the expected total rounds to achieve $\epsilon$ accuracy is $R_\epsilon \cdot \EE[\cS]{T(\cS)}$. For i.i.d. Bernoulli participation model, we have $ T(\cS) \geq \frac{1}{\min \{p_i | i \in \cS  \}}$, and we can further show that $ \EE[\cS]{T(\cS)} \geq \frac{1}{p_{min}} \frac{S}{N}$ (see Appendix~\ref{app:fedavg:expected_wait} for details). Therefore, 
\begin{equation}
    \EE {T_\epsilon^{(\texttt{FedAvg}, \texttt{SCAFFOLD})} } \geq  \frac{S}{N} \frac{1}{p_{min}} \widetilde{\cO} \Big(  \frac{\sigma^2}{\mu  S K \epsilon }  \Big)  = \widetilde{\cO} \Big(  \frac{1}{p_{min}} \cdot \frac{\sigma^2}{\mu N K \epsilon }  \Big). \label{eq:teps:fedavg}
\end{equation}
By comparing Eqn.~\ref{eq:teps:ours} and Eqn.~\ref{eq:teps:fedavg}, we see that both \fedavg[] and \scaffold[] are more vulnerable to stragglers, that is, the devices with very small participation probabilities; on the contrary, the convergence rate of \algoname[] only depends on the average of $1/p_i$ instead of $1/p_{min}$. We also provide empirical experiments showing that \algoname[] converges faster than FedAvg in Section~\ref{sec:exp}.

\section{Convergence result for non-convex objective functions}
\label{sec:non_cvx}
In this section, we present the convergence guarantee of $\mathtt{MIFA}$ for the non-convex case. First we list the additional assumptions as below.
\begin{assumption}[Hessian Lipschitz]\label{assumption: hessian} $f_1, \cdots, f_N$ are all $\rho$-Hessian Lipschitz: for all $w$ and $v$,
$\norm{ \nabla^2 f_i(w)-\nabla^2 f_i(v)}\leq \rho \norm{w-v}$. 
\end{assumption}
\begin{assumption}[Bounded noise]\label{assumption:bounded noise} The noise of the local stochastic gradients is upper bounded by a constant $\delta$ almost surely:  
    $\norm{\sg f_i(w) - \nabla f_i(w)}\leq \delta$ a.s., $\forall\  i\in [N]$.
\end{assumption}
\begin{assumption}[Bounded gradient dissimilarity]\label{assumption: bounded dissimilarity for noncvx}
There exist $\alpha > 0$ and $\beta_i>0$ such that for all $w$ and $i \in [N]$: $\norm{\nabla f_i(w) }^2 \leq \alpha \norm{\nabla f(w)}^2 + \beta_i$. Furthermore, we define $\beta=\frac{1}{N}\nsum \beta_i$.
\end{assumption}
\begin{assumption}\label{assumption:constant delay}
There exists a constant $\nu_i$ such that $\tau(t,i)\leq \nu_i$, for all $i \in [N]$ and $t\geq 1$. Furthermore, define $\bar{\nu}=\frac{1}{N}\nsum \nu_i$ and $\nu_{\max} = \max_{i\in [N]} \nu_i$.
\end{assumption}

The analysis of non-convex functions is much more technically involved, and our results rely on strong assumptions that provide a finer control of the gradient difference (\Cref{assumption: hessian}), gradient noise (\Cref{assumption:bounded noise}), gradient dissimilarity among devices (\Cref{assumption: bounded dissimilarity for noncvx}), and device unavailability (\Cref{assumption:constant delay}). We remark that \Cref{assumption: hessian} is also made in \citep{carmon2018accelerated, jin2017escape}, and  \Cref{assumption: bounded dissimilarity for noncvx} is also made in \citep{pmlr-v119-karimireddy20a, wang2020tackling}. We leave it as future work to study whether and how \algoname[] converges for non-convex functions with weaker assumptions.

\begin{theorem}\label{theorem:noncvx}
Assume that Assumptions \ref{assumption:smooth}, \ref{assumption:noise},  and \ref{assumption: hessian} to \ref{assumption: bounded dissimilarity for noncvx} hold.   Further assume that the device availability sequence $\tau(t,i)$ satisfies 
 \Cref{assumption:constant delay} and $\tau(1,i)=0$ for all $i \in [N]$. By using a learning rate $\eta = \sqrt{\frac{N}{KTL(1+\bar{\nu})}}$, for $T\geq \max \{32\alpha LNK, 16L N K, \frac{8 KN\nu_{\max}^2(L^2+\rho \delta)}{L}\}$, after $T-1$ communication rounds, \algoname[] satisfies:
 %=\frac{1}{T}\sum_{t=1}^T \EE{\norm{\nabla f(\w{t})}^2 }
\begin{align*}
    \min_{1\leq t \leq T} \EE[\xi]{\norm{\nabla f({w}_t)}^2} = \mathcal{O}\left(\sqrt{\frac{(1+\bar{\nu})L}{TKN}}(f(\w{1}) - f^*+\sigma^2)+\frac{A_4+A_5}{T}\right),
\end{align*}
where $f^*$ is the optimal value, and:
\begin{align*}
    A_4 &= N K L \left(\alpha \sigma^2\bar{\nu}+\frac{\sigma^2\nu_{\max}}{\sqrt{KN}}+\sigma\nu_{\max}\sqrt{\beta}\right) +\frac{(L^2+ \rho\delta) \sigma^2 \nu_{\max}}{L},\\
    A_5 &= \frac{(K-1)N L(\beta  +\sigma^2/K)}{\bar{\nu}+1}.
\end{align*}
\end{theorem}

Next, we show that the leading $\mathcal{O}(1/\sqrt{T})$ term is theoretically optimal for zero-respecting algorithms.

\begin{proposition}
\label{thm:lower-noncvx}
Let $c_0 > 0$ be a universal constant. For any randomized zero-respecting algorithm,  there exists a stochastic non-convex problem satisfying Assumption\ref{assumption:smooth}, \ref{assumption:noise},  \ref{assumption: hessian} and \ref{assumption: bounded dissimilarity for noncvx}, such that the output $w_T$ after T rounds of communication has expected sub-optimality lower bounded by
\begin{align*}
    \mathbb{E}[\|\nabla f(w_T) \|^2]  \ge \mathbb{E}[\|\nabla f(w_T) \|]^2 \ge c_0 \sqrt{\frac{\bar{\nu} L \sigma^2(f(w_0) - f^*)}{NKT}}. 
\end{align*}
\end{proposition}

The above proposition show that when $\sigma \sqrt{(f(w_0) - f^*)} \sim \sigma^2 + (f(w_0) - f^*)
$, the result in Theorem~\ref{theorem:noncvx} is tight. However, note that the counter example we used requires the quantity $\delta$ in Assumption~\ref{assumption:bounded noise} to scale with $T$, hence requiring $\delta$ to be large enough. This does not change the optimality of the first term as the first term is independent of $\delta$. Whether this requirement can be relaxed is left as an open problem.

\begin{remark}
When all $\nu_i=0$ (i.e. all the devices are active), our convergence bound reduces to  
$
\textstyle \mathcal{O}\Big(\sqrt{\frac{L}{TKN}}(f(\w{1}) - f^*+\sigma^2)+\frac{ (K-1)N L(\beta  +\sigma^2/K)  }{T}\Big)
$. This matches the result in \citep{yang2021achieving} (Thm.~1, $\eta=1, \eta_L = \sqrt{\frac{N}{KTL}}$).
\end{remark}

\section{Numerical Experiments}
\label{sec:exp}

In this section, we conduct numerical experiments to verify our theoretical results and investigate how the heterogeneity of the device availability influences the Federated optimization algorithms. We compare the performance of the following four algorithms: \fedavg[] with importance sampling (\texttt{FedAvg}-IS), Biased \fedavg[], \fedavg[] with device sampling, and our proposed \algoname[]. For the detailed discussions of the algorithms, we refer the readers to Sections~\ref{sec:setup} and~\ref{sec:alg}. We remark that for a fair comparison, we deliberately include the first few rounds that \algoname[] needs to wait to receive responses from all devices for initializing the update-array $\{G^i\}$.

Following \cite{Li2020On, MLSYS2020_38af8613}, we construct non-i.i.d. datasets from two commonly used computer vision datasets --- MNIST \citep{lecun1998gradient} and CIFAR-10 \citep{krizhevsky2009learning} . Specifically, we divide the data into $N=100$ devices with each device holding samples of only two classes, which creates a high level of data heterogeneity. For simplicity, we ensure that each device holds the same number of samples. We do not use any data augmentation.
We use multinomial logistic regression as the convex model and Lenet-5 \citep{lecun2015lenet} with ReLU activations as the non-convex model. For all experiments, we use weight decay in the training process, which corresponds to adding $\ell_2$ penalty. We use logistic models for MNIST dataset, while we use Lenet-5 for CIFAR-10. Our code is adapted from \cite{Li2020On}, which is under MIT License.

We model the availability of the devices as independent Bernoulli random trials. The $i$-th device is assigned with a probability $p_i$, where at each time step, the device becomes active with probability $p_i$. In our experiments, the $p_i$'s are chosen such that devices holding data of smaller labels participate less frequently. Specifically, if the $i$-th device holds the data of label $j$ and $k$, we set $p_i = p_{\min} \min(j,k) / 9  + (1-p_{\min})$, where $p_{\min}$ controls the lower bound of the participation probabilities. The correlation between the participation patterns and local datasets increases the difficulty of the problem \cite{kairouz2019advances}.  To investigate this phenomenon, we repeat the experiments for $p_{\min}=0.1$ and $0.2$. We control the randomness of device participation when testing different algorithms.

In all the experiments, we set the initial learning rate to be $\eta_0 = 0.1$ and decay the learning rate as $\eta_t =\eta_0 \cdot \frac{1 }{t}$. We set the weight delay to be $0.001$. The local batch size is $100$ and each local update consists of $2$ epochs. Therefore, the actual number of local steps $K$ depends on the size of the dataset. We run all the experiments with 4 GPUs of type GeForce RTX 2080 Ti. We repeat the experiments for 5 different random seeds, and all of the experiments exhibit similar training curves. We report the averaged training loss and test accuracy with error bars in Figure~\ref{fig:exp1}.

We observe that \fedavg[] with device sampling (FedAvg ($S=50$) and FedAvg ($S=100$) in Figure~\ref{fig:exp1}) is severely influenced by the straggling devices and makes progress relatively slowly compared to the other algorithms. Although biased \fedavg[] converges fast at the beginning, this simple algorithm is biased, and the optimality gaps are prominent for the harder CIFAR-10 dataset and when $p_{\min}$ is small. On the contrary, our proposed \algoname[] avoids waiting for stragglers, converges fast without bias, and is competitive with \fedavg[] with importance sampling,  which requires knowledge of the participation probabilities.
%even if it needs to wait for several rounds to start training.

\begin{figure}[htb!]
\vspace{-0.15in}
\begin{center}
    \subfigure[ $p_{\min}=0.1$]{
        \includegraphics[width=0.245\textwidth]{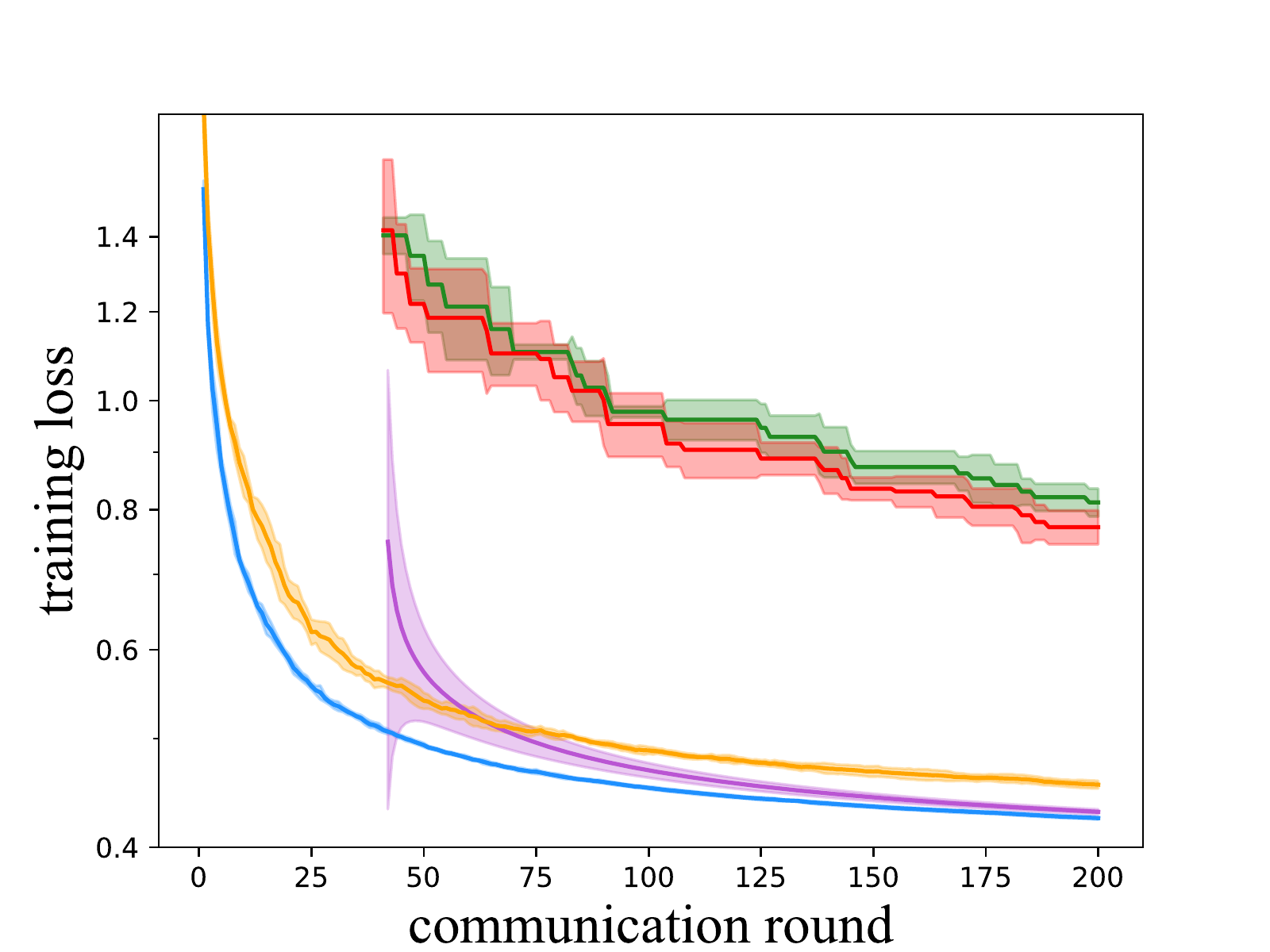}\label{fig:a1}
        \hspace{-0.13in}
    }
    \subfigure[ $p_{\min}=0.1$ ]{
        \includegraphics[width=0.245\textwidth]{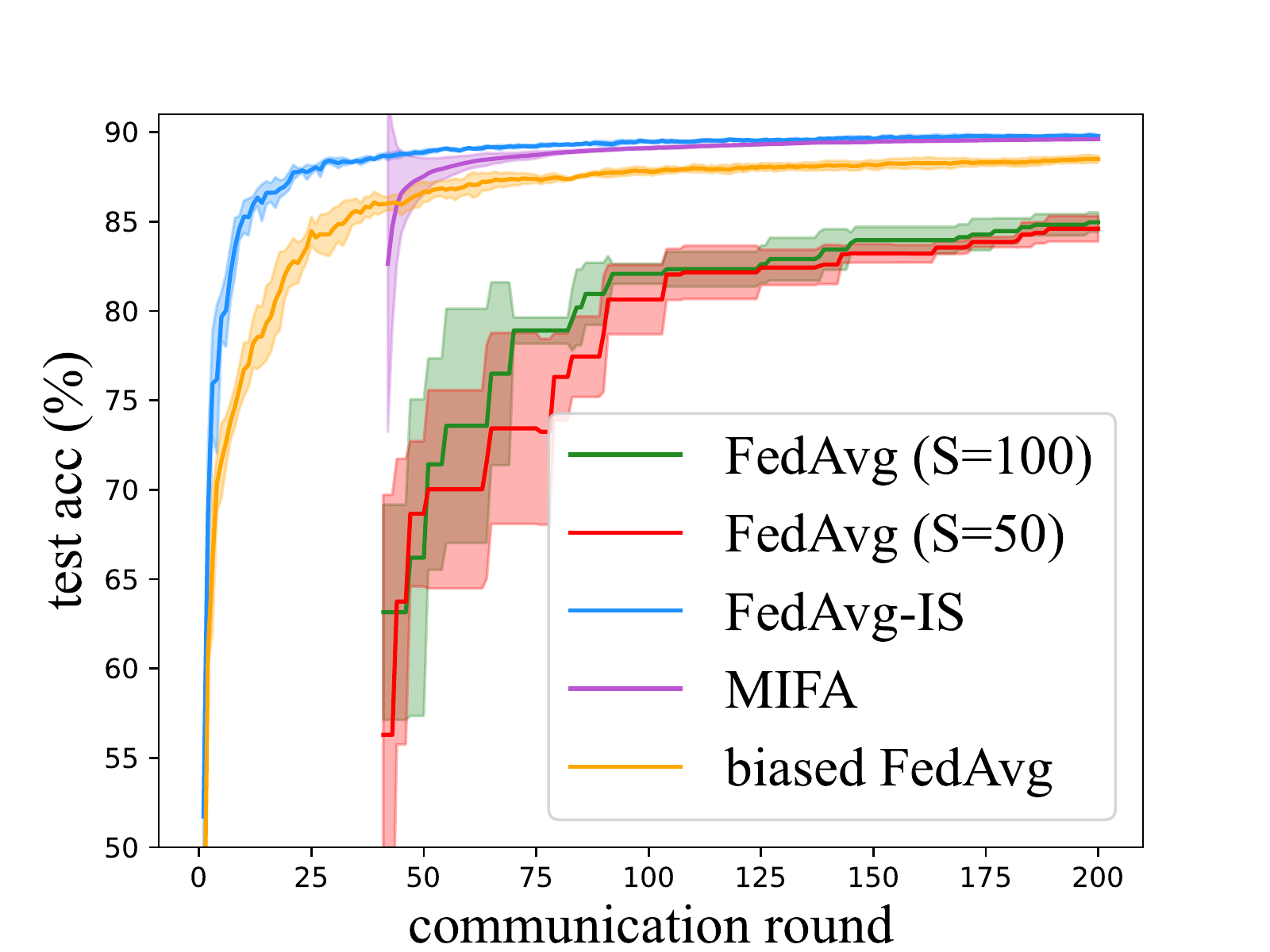}\label{fig:a2}
        \hspace{-0.13in}
    }
    \subfigure[$p_{\min}=0.2$]{
        \includegraphics[width=0.245\textwidth]{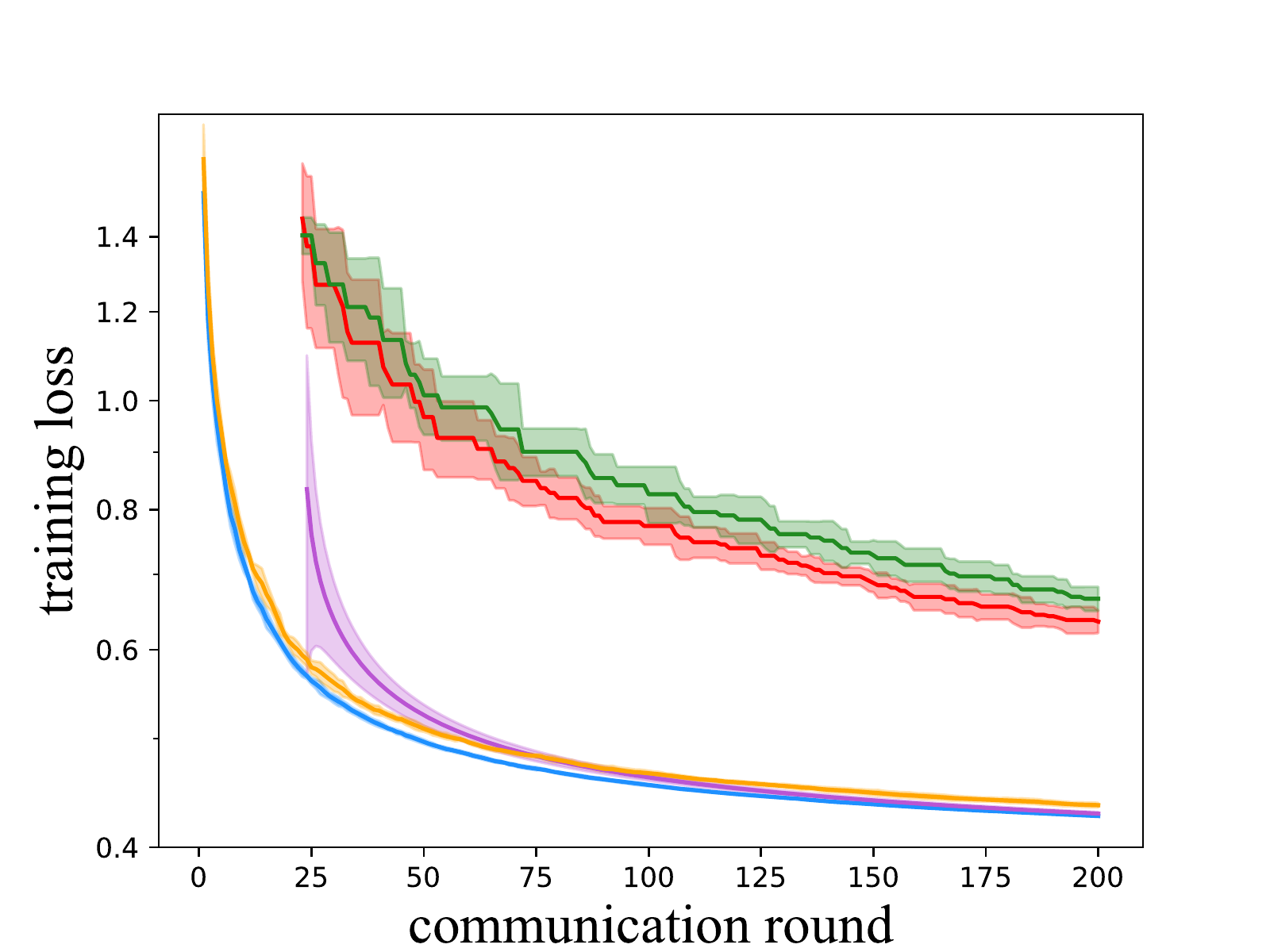}\label{fig:a3}
        \hspace{-0.13in}
    }
    \subfigure[ $p_{\min}=0.2$]{
        \includegraphics[width=0.245\textwidth]{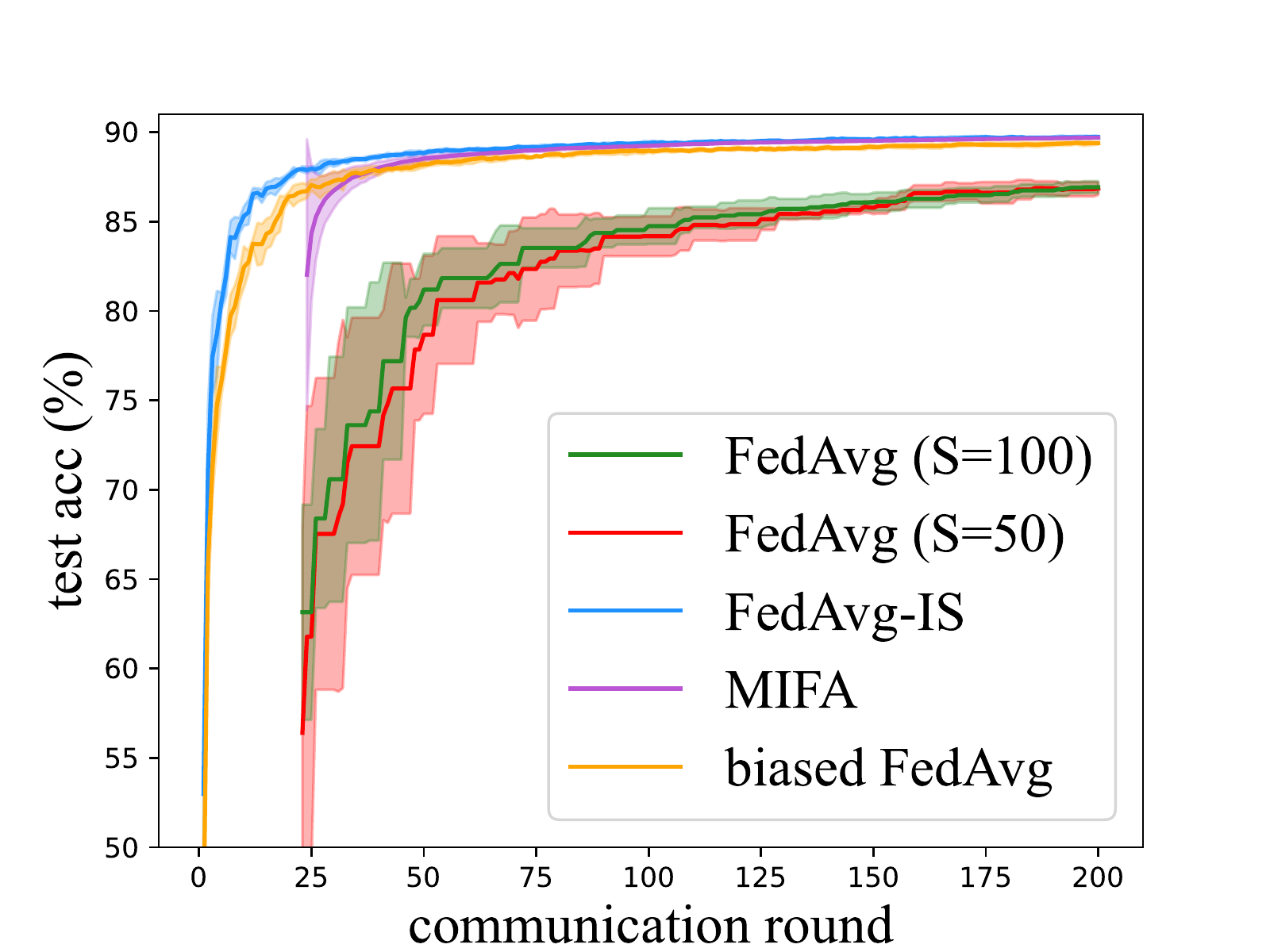}\label{fig:a4}
    }
     \subfigure[ $p_{\min}=0.1$]{
        \includegraphics[width=0.245\textwidth]{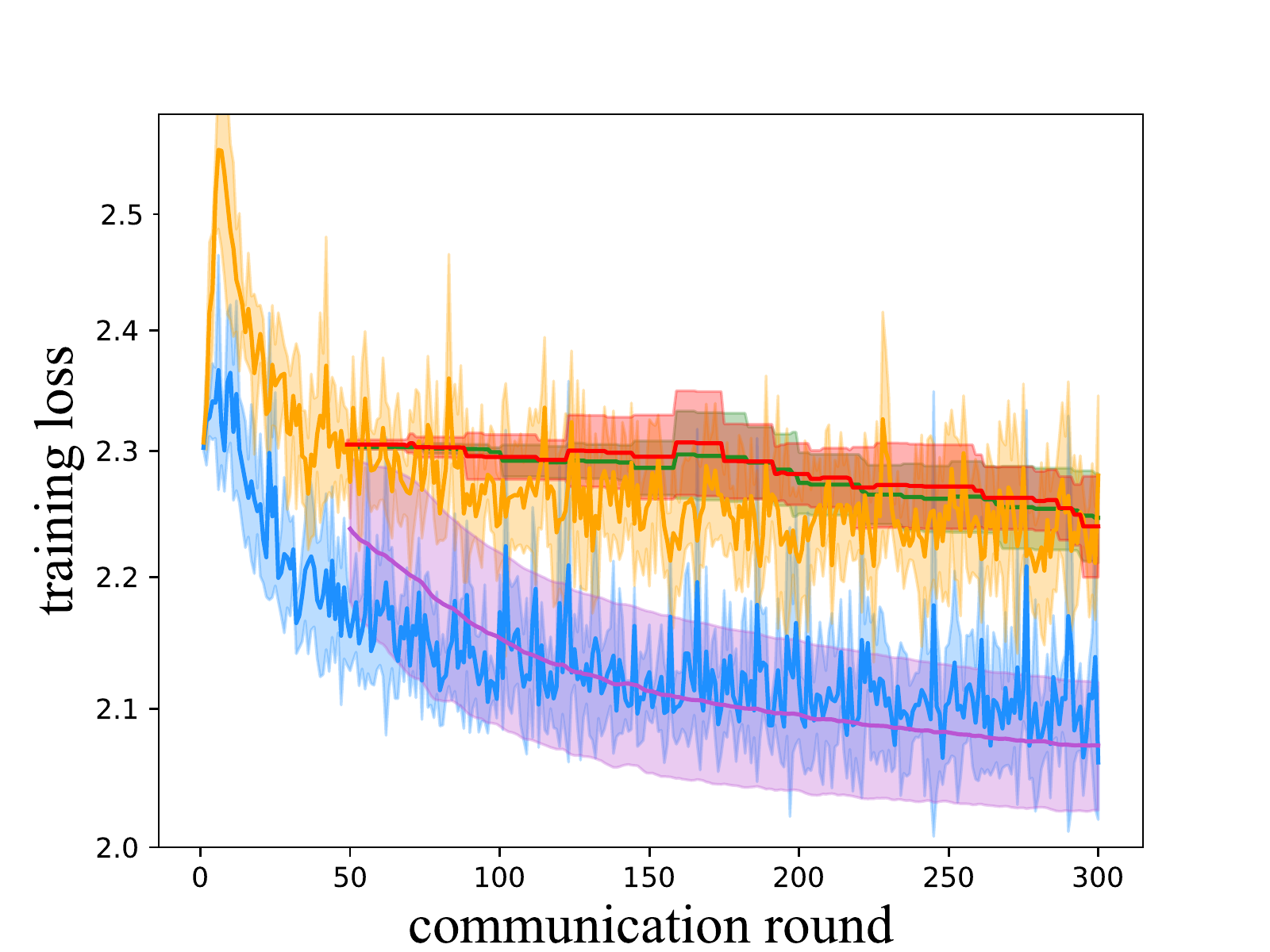}\label{fig:b1}
        \hspace{-0.13in}
    }
    \subfigure[ $p_{\min}=0.1$ ]{
        \includegraphics[width=0.245\textwidth]{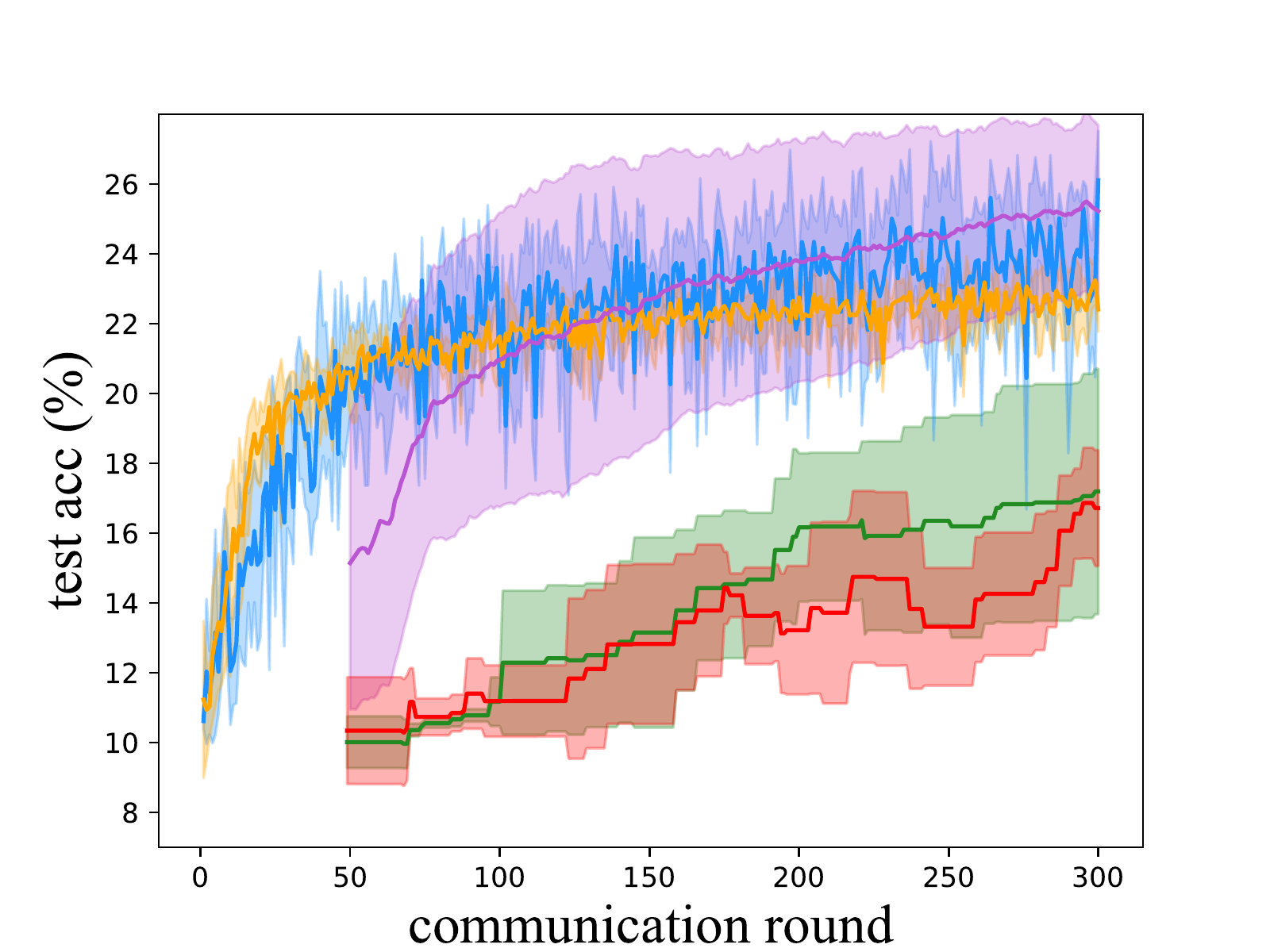}\label{fig:b2}
        \hspace{-0.13in}
    }
    \subfigure[$p_{\min}=0.2$]{
        \includegraphics[width=0.245\textwidth]{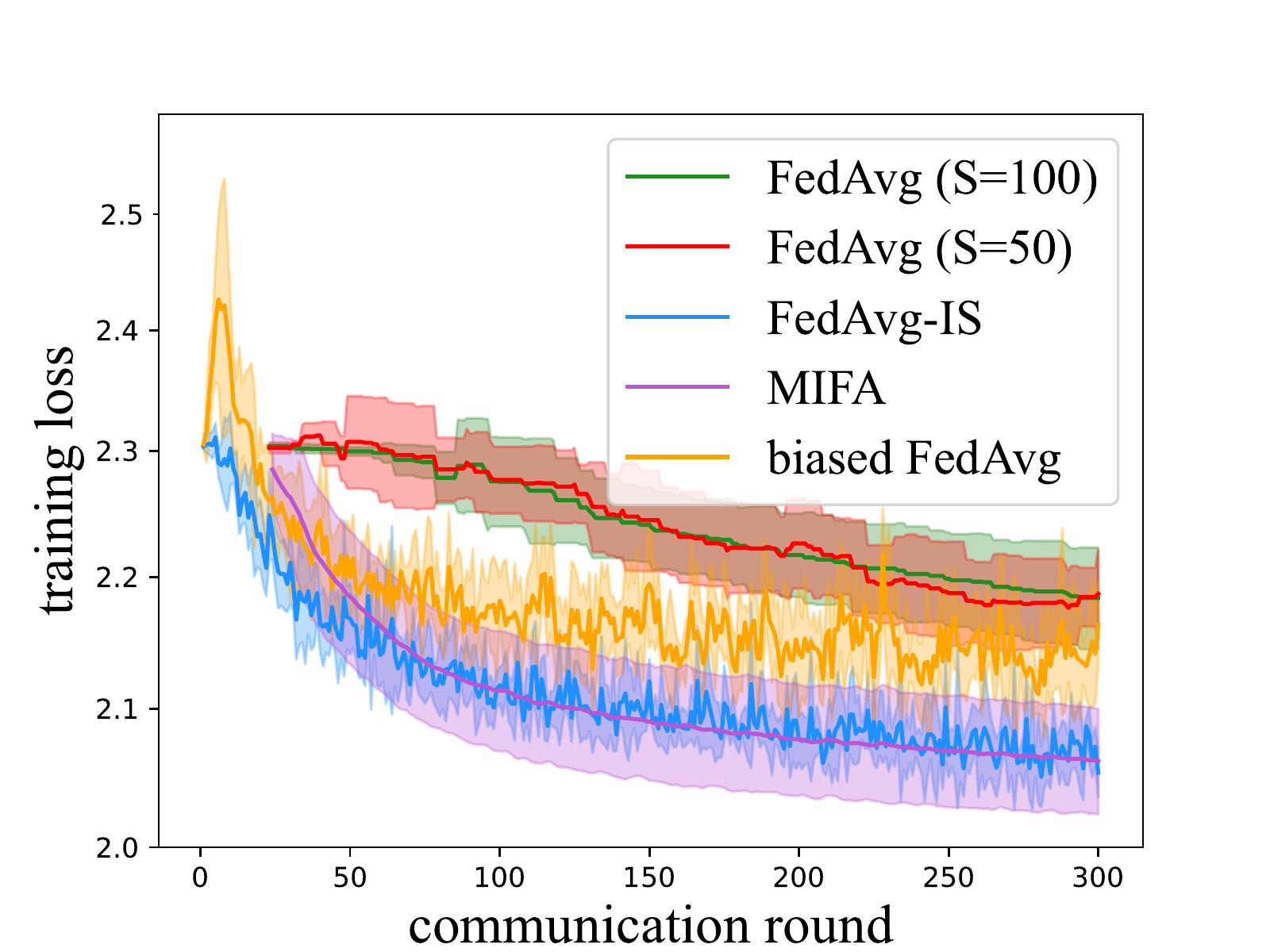}\label{fig:b3}
        \hspace{-0.13in}
    }
    \subfigure[ $p_{\min}=0.2$]{
        \includegraphics[width=0.245\textwidth]{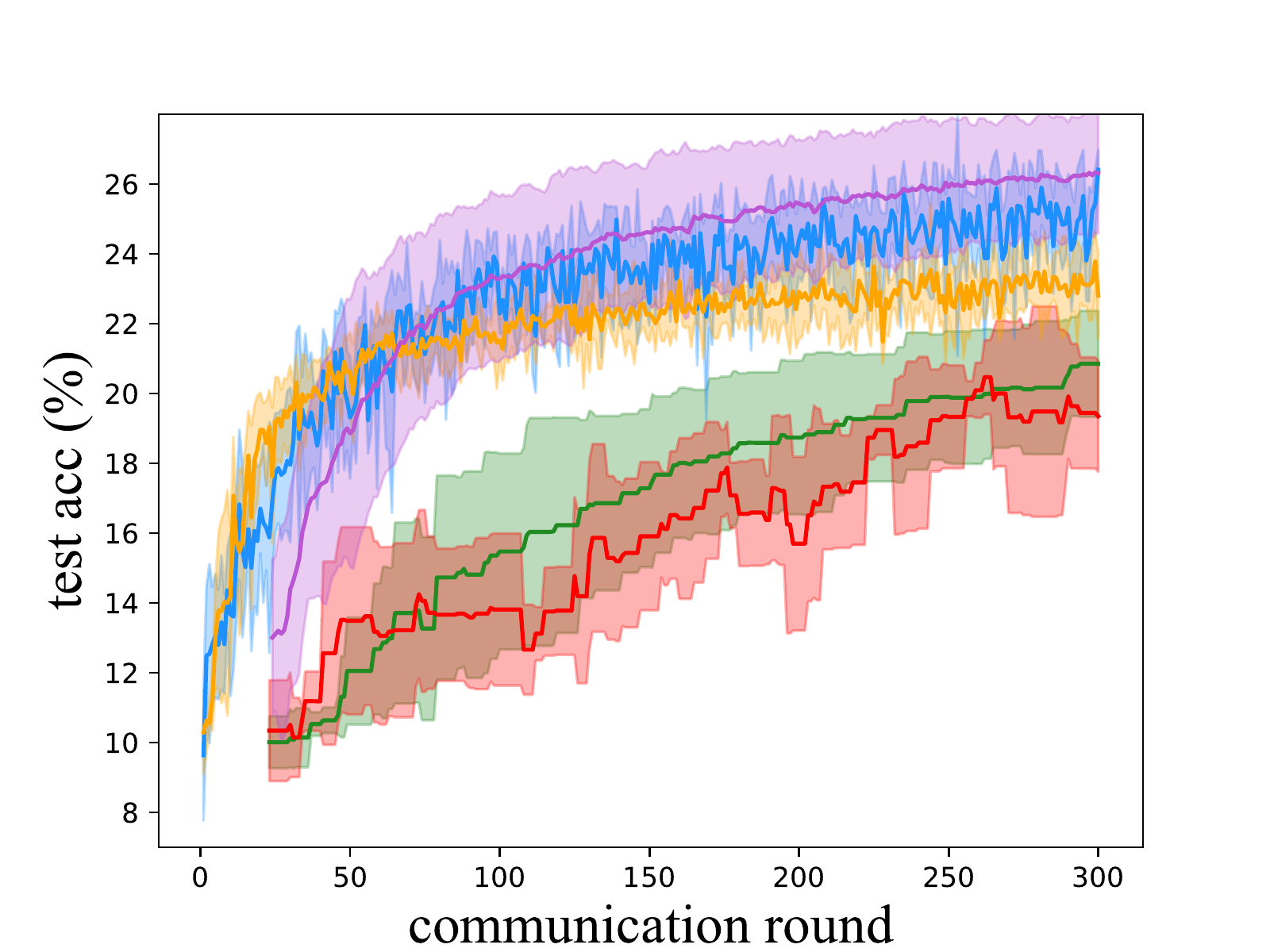}\label{fig:b4}
        \vspace{-0.125in}
    }
    \vspace{-0.1in}
    \caption{Training losses and test accuracies. Fig.~\ref{fig:a1}--\ref{fig:a4}: logistic models on non-iid MNIST. Fig.~\ref{fig:b1}--\ref{fig:b4}: Lenet-5 on non-iid CIFAR-10. FedAvg ($S=50$) and FedAvg ($S=100$) refer to FedAvg with device sampling that samples $S$ devices for each global update. FedAvg-IS is short for FedAvg with importance sampling, which requires knowledge of the participation probabilities.}
        \label{fig:exp1}
    \vspace{-0.15in}
\end{center}
\end{figure}

\section{Conclusions and Discussions}
\label{sec:conclu}

In this paper, we study FL algorithms in the presence of arbitrary device unavailability and propose \algoname[], which avoids waiting for straggling devices and re-uses the memorized latest updates as the surrogate when the device is unavailable. We theoretically analyze \algoname[] without any structural assumptions on the device availability and prove the convergence for strongly convex and non-convex smooth functions.
Different from the literature that studies oracle complexity in terms of stochastic gradient evaluations, we argue that in federated learning system, the bottleneck lies in the non-stationary and possibly adversarial pattern of device participation. Therefore, it is important to study how the number of inactive rounds influences the convergence rate. In Theorem~\ref{theorem:maintheorem}, the dependency upon $\tau_{\max, T}$ might be an artifact of our analysis, and a future direction is to study whether we can remove this dependency. Another important direction is to analyze algorithms for non-convex functions under weaker assumptions.

\bibliography{main_arxiv}
\bibliographystyle{plain}

\newpage
\appendix

\section{Baseline algorithms}\label{sec:baseline}

The three different baseline algorithms discussed in Section~\ref{sec:setup} are summarized in the algorithm box below.

\newcommand{\colorline}[2]{\colorbox{#1}{\makebox[\linewidth][l]{#2}}}
\begin{algorithm}[h!]
    \centering
    \caption{\fedavg[] variants}\label{algo: baseline}
    \begin{algorithmic}[1]
    \STATE \textbf{Input}: initial $w$, learning rates $\eta_t$, $t'\leftarrow 1$
        \STATE  \textbf{Server executes}:
        \FOR{$t=1, \cdots, T-1$}
            \STATE broadcast $w$ to all active devices $i \in \mathcal{A}(t)$ 
            \FOR{each active device $i$}
            \STATE $G^i \gets $ DeviceUpdate($i,w, \eta_t$)
            \ENDFOR
            %\STATE $w_{t+1} = w_t - \eta_t \cdot $ {aggregate}($G^i_t$)
        \begin{tcolorbox}[colback=green!30,boxrule=0mm,sharp corners, grow to left by=5mm]
            \STATE   $\displaystyle \tiny w \gets w - \eta_t / {|\mathcal{A}(t) |}   \cdot \sum_{i \in \mathcal{A}(t) } G^i$  \hfill  biased \fedavg[] 
        \end{tcolorbox}
          \begin{tcolorbox}[colback=yellow!30,boxrule=0mm,sharp corners, grow to left by=5mm]
            \STATE  $\displaystyle \tiny w \gets w - \eta_t  / {|\mathcal{A}(t) |}\cdot \sum_{i \in \mathcal{A}(t) } \frac{1}{p_i}G^i$ \hfill  \fedavg[] with importance sampling
            \end{tcolorbox}
            \begin{tcolorbox}[colback=purple!20,boxrule=0mm,sharp corners, grow to left by=5mm]
             \IF{updates from the randomly selected $S$ devices are received}
                \STATE  $\displaystyle \tiny w \gets w - \eta_{t'} /S \cdot \sum_{i \in \cS} G^i$ \hfill \fedavg[] with device sampling
                \STATE $t' \gets t' + 1 $
            \ENDIF
            \end{tcolorbox}
        \ENDFOR

    \end{algorithmic}
\end{algorithm}

\section{Proof of convergence for smooth and strongly convex objective functions}\label{app:proof_sc}
In this section, we analyze the convergence of \algoname[] for smooth and strongly convex problems. Let $\taubart$ be defined the same as in \Cref{sec:cvx}. Also, we introduce $$\dmaxt=\frac{1}{N}\nsum \left [\max_{1\leq t\leq T-1}\tau(t,i)\right ]^2,
$$
which takes the maximum number of inactive rounds in round $1, \cdots, T-1$ for each device and averages its square over devices.  The following theorem is a more general version of \Cref{theorem:maintheorem}.

% \kaixuan{ BEGIN : Xinran, please move the following to appendix.}

% Here $G_t^i$ is defined as:
% \begin{align*}
%     G_t^i \triangleq \sum_{k=0}^{K-1}\sg f_i(\w{\dl{t} k}^i) =   \frac{1}{\eta_{\dl{t}}}(\w{\dl{t}, K}^i - \w{\dl{t}}) 
% \end{align*}
% That is, for an active device indexed $i$ whose inactive duration $\tau(t,i)=0$, the server uses its updates computed at round $t$. For an unavailable device indexed $i'$, the server uses its latest updates sent back at round $t-\tau(t,i')$ instead. 

% \kaixuan{ END : Xinran, please move the above to appendix.}

% Apart from the notation defined previously, we define 

\begin{theorem}\label{theorem:maintheorem_general}
Assume that  Assumptions~\ref{assumption:smooth} to~\ref{assumption:cvx} hold.  Further assume that the device availability sequence $\tau(t,i)$ satisfies \Cref{assumption:delay} and $\tau(1,i)=0$ for all $i \in [N]$. By setting the learning rate $\eta_t = \frac{4}{\mu K(t+a)}$ with $a=\max\{100, 40t_0\}(L/\mu)^{1.5}$. After $T-1$ communication rounds, \algoname satisfies:
\begin{align*}
  \EE[\xi]{f(\bar{w}_T)}-f(\wopt)&=\mathcal{O}\left(\frac{\taubart+1}{\mu NKT}\sigma^2 + \frac{\dmaxt A_1' + (K-1)^2A_2'+A_3'}{\mu^2 T^2}\right),
\end{align*}
where $\overline{w}_T$ is a weighted average of $\w{t}$ defined as:
\begin{align*}
    \overline{w}_T &= \frac{1}{W_T}\sum_{t=1}^T (t+a-1)(t+a-2) w_t,
    W_T = \sum_{t=1}^T (t+a-1)(t+a-2),
\end{align*}
and
%=\frac{1}{3}T^3+(a-1)T^2+(a^2-2a+\frac{2}{3})T
$A_1'  = L (D+L \sigma^2/\mu), 
  A_2' = L(D/K^2+\sigma^2/K^3), A_3' = t_0^2 L^3 \norm{\w{1}-\wopt}^2.$
 \end{theorem}
Note that the only difference between \Cref{theorem:maintheorem_general} and \Cref{theorem:maintheorem} lies in  $\dmaxt$ and $\tau_{\max, T}^2$. \Cref{theorem:maintheorem_general} yields \Cref{theorem:maintheorem} since $\dmaxt \leq \tau_{\max, T}^2$.
\subsection{Additional notation}
Define $\teta = K\eta_t$.
The update rule of \algoname[] can be summarized as
\begin{align}
    \w{t+1} = \w{t} - \frac{\eta_t}{N}\knsum \sg f_i(\w{t,k}^i) = \w{t} - \frac{\teta}{KN}\knsum \sg f_i(\w{\dl{t},k}^i).\label{eq:update rule}
\end{align}
Further, let $e^i_{t, k} = \sg f_i(\w{t,k}^i) - \nabla f_i(\w{t, k}^i)$ be the sampling noise of device $i$ at round $t$ and local step $k$. Define $\Delta_t = \EE[\xi]{\norm{\w{t}-\wopt}^2}$.
Next, we introduce the following notation about device unavailability. 
 Define $\tau_t$ and $d_t$ to be the average of the number and squared number of inactive rounds over all devices at round $t$. That is,
 $$\tau_t = \frac{1}{N}\nsum \tau(t,i), \quad d_t = \frac{1}{N}\nsum [\tau(t,i)]^2.$$
 Denote by the sum of $\tau_t$ as $s_T$, i.e., $s_T = \sum_{t=1}^{T-1}\tau_t$. Lastly, define
\begin{align*}
    l_t = \max_{i,j}\{\tau(t,i)+\tau(t-\tau(t,i), j)\}.
\end{align*}
That is, the ``oldest'' response used to update $\w{\dl{t}}$ into $\w{t}$ is received in round $t-l_t$.
For convenience, all expectations in this section are taken over sampling noise $\xi$, and the summation $\knsum$ is taken over $i=1,\cdots, N$ and $k=0,1,\cdots, K-1$.
\subsection{Preliminary lemmas}
Before starting the proof, we introduce some preliminary lemmas in this subsection.
\begin{lemma}[Property of smooth functions]\label{lemma:smooth}
For all functions $f$ that are $L$-smooth with domain $\mathcal{X}$, if $\exists \inf_{x \in \mathcal{X}}f(x):= f^*$, we have:
$$\frac{1}{2L}\norm{\nabla f(x)}^2 \leq f(x) - f^*.$$
\end{lemma}
\begin{proof}
By definition of $L$-smoothness
\begin{align*}
    f^*&\leq f(x-\frac{1}{L}\nabla f(x))\\
   & \leq f(x) - \dotp{\nabla f(x)}{\frac{1}{L}\nabla f(x)} + \frac{1}{2L}\norm{\nabla f(x)}^2\\
   &= f(x) -\frac{1}{2L}\norm{\nabla f(x)}^2.
\end{align*}
Rearrange the terms on both sides and we complete the proof.
\end{proof}
The following lemma bounds the norm of local gradient $\norm{\nabla f_i(w)}$ by how close the $w$ is to to the global optimum $\wopt$.
\begin{lemma}[Bounding the local gradient] \label{lemma:diff} For all functions $f_i$ satisfying Assumptions~\ref{assumption:smooth} and \ref{assumption:cvx}$, \norm{\nabla f_i(w)}^2$ can be bounded by $\norm{w-\wopt}^2$. That is,
$$ \left \|\nabla f_i(w)\right\| ^2 \leq 2L^2 \|w-\wopt\|^2+2\|\nabla f_i(\wopt)\|^2.$$
\begin{proof}
By Jensen's inequality and $L$-smoothness, we have
\begin{align*}
    \left \|\nabla f_i(w)\right\| ^2 &\leq 2\|\nabla f_i(w) - \nabla f_i(\wopt)\|^2 + 2\|\nabla f_i(\wopt)\|^2\\
    &\leq 2L^2 \|w-\wopt\|^2+2\|\nabla f_i(\wopt)\|^2.
\end{align*}
\end{proof}
\end{lemma}
The following lemma comes from Lemma 5 in \citep{pmlr-v119-karimireddy20a}.
\begin{lemma}[Perturbed strong convexity]\label{lemma: perturb}
The following holds for any $L$-smooth and $\mu$-strongly convex function $f$ and any $x, y, z$ in the domain of $f$:
$$\left<\nabla f(x), z-y \right> \geq f(z) - f(y) + \frac{\mu}{4} \norm{y-z}^2 - L \norm{z-x}^2.$$
\end{lemma}
\begin{proof}
In order for the paper to be self-contained, we restate the proof here.

By smoothness:
$$f(z) \leq f(x) + \left<\nabla f(x), z-x \right> + \frac{L}{2}\norm{z-x}^2 \Rightarrow \left<\nabla f(x), z-x\right> \geq f(z) - f(x) - \frac{L}{2}\norm{z-x}^2.$$
By strong convexity:
$$f(y) \geq f(x) + \left<\nabla f(x), y-x \right> + \frac{\mu}{2}\norm{y-x}^2 \Rightarrow \left<\nabla f(x), x-y\right> \geq f(x) - f(y) + \frac{\mu}{2}\norm{y-x}^2.$$
Combining the above inequalities, we have:
$$\left<\nabla f(x), z-y\right> \geq f(z) - f(y) -\frac{L}{2}\norm{z-x}^2 + \frac{\mu}{2} \norm{y-x}^2.$$
By triangle inequality:
$$\norm{y-x}^2 \geq \frac{1}{2}\norm{y-z}^2 - \norm{x-z}^2.$$
Thus, 
\begin{align*}
\left<\nabla f(x), z-y\right>&\geq f(z) - f(y) + \frac{\mu}{4}\norm{y-z}^2 - \frac{L+\mu}{2}\norm{x-z}^2\\
&\geq f(z) - f(y)
+ \frac{\mu}{4}\norm{y-z}^2 - L\norm{x-z}^2,
\end{align*}
where the second inequality only uses $L \geq \mu$.
\end{proof}

The following lemma is slightly modified from Lemma 8 in \citep{pmlr-v119-karimireddy20a}.
\begin{lemma}[Bounded drift for strongly convex and smooth objective functions]\label{lemma: bounded drift}
For all $ K\geq 1$ and $ 0\leq k\leq K-1$, when $\teta \leq\tfrac{1}{10L}$, we have bounded drift:
\begin{align*}
    \EE{\norm{\w{t,k}^i - \w{t}}^2} 
    & \leq \frac{8\teta^2 L^2(K-1)}{K} \Delta_t + \frac{8(K-1)\teta^2}{K}\norm{\nabla f_i(\wopt)}^2+\frac{2(K-1)\teta^2\sigma^2}{K^2}.
\end{align*}
\end{lemma}
\begin{proof}
For $K=1$, the bound trivially holds since $\w{t,0}^i = \w{t}$. For $K\geq 2$,
\begin{align*}
    \EE{\norm{\w{t,k}^i - \w{t}}^2} &= \EE{\norm{\w{t,k-1}^i - \w{t} -\eta_t \sg f_i(\w{t,k-1}^i)}^2}\\
    &  = \EE{\norm{\w{t,k-1}^i - \w{t} -\eta_t \nabla f_i(\w{t,k-1}^i)}^2}+\eta_t^2\sigma^2\\
    & \leq\left (1+\frac{1}{K-1}\right)\EE{\norm{\w{t,k-1}^i - \w{t}}^2} + K\eta_t^2 \EE{\norm{\nabla f_i(\w{t, k-1}^i)}^2} + \eta_t^2\sigma^2\\
    & \leq \left(1+\frac{1}{K-1}\right)\EE{\norm{\w{t,k-1}^i - \w{t}}^2} + 2K\eta_t^2 \EE{\norm{\nabla f_i(\w{t})}^2}\\
    &\quad+ 2K\eta_t^2\EE{\norm{\nabla f_i(\w{t, k-1}^i) - \nabla f_i(\w{t})}^2} + \eta_t^2\sigma^2\\
    &\leq \left(1+\frac{1}{K-1} + 2KL^2\eta_t^2\right)\EE{\norm{\w{t,k-1}^i - \w{t}}^2}\\
    &\quad+ 2K\eta_t^2 \EE{\norm{\nabla f_i(\w{t})}^2} + \eta_t^2\sigma^2.
\end{align*}
The first inequality uses $\norm{x+y}^2 \leq (1+\frac{1}{\nu})\norm{x}^2 + (1+\nu)\norm{y}^2,\forall \nu > 0 $ with $\nu = \frac{1}{K-1}$.
For $\teta \leq\frac{1}{10L}$, i.e., $\eta_t \leq \frac{1}{10KL}$, we have $2K L^2\eta_t^2 \leq \frac{1}{50(K-1)}$. Plug in the definition of $\teta$, we have
\begin{align*}
     \underbrace{\EE{\norm{\w{t,k}^i - \w{t}}^2}}_{Y_k}  \leq \underbrace{\left(1+\frac{51}{50(K-1)}\right)}_{h_1} \underbrace{\EE{\norm{\w{t,k-1}^i - \w{t}}^2}}_{Y_{k-1}} + \underbrace{\frac{2\teta^2}{K}\EE{\norm{\nabla f_i(\w{t})}^2} + \frac{\teta^2\sigma^2}{K^2}}_{h_2}.
\end{align*}
Unrolling the recursion $Y_k \leq h_1 Y_{k-1} + h_2$, where $Y_0=0$, we have
\begin{align*}
    Y_k &\leq h_1^k Y_{0}+h_2\sum_{j=0}^{k-1}h_1^j
     = \frac{ h_2(h_1^k-1)}{h_1-1} \leq \frac{h_2(h^{K-1}_1-1)}{h_1-1}.
\end{align*}
Since $h_1^{K-1} = (1+\frac{51}{50(K-1)})^{\frac{50(K-1)}{51}\cdot \frac{51}{50}}\leq \exp(\frac{51}{50})<3$ and $h_1-1 = \frac{51}{50(K-1)}>\frac{1}{K-1}$, plugging in the value of $h_2$, we have
\begin{align}
    \EE{\norm{\w{t,k}^i - \w{t}}^2} &\leq 2(K-1)\left(\frac{2\teta^2}{K}\EE{\norm{\nabla f_i(\w{t})}^2} + \frac{\teta^2\sigma^2}{K^2}\right)\label{ineq: bounded drift for sc}\\
    & \leq \frac{8\teta^2 L^2(K-1)}{K} \Delta_t + \frac{8(K-1)\teta^2}{K}\norm{\nabla f_i(\wopt)}^2+\frac{2(K-1)\teta^2\sigma^2}{K^2}\notag,
\end{align}
where the second inequality uses \Cref{lemma:diff}.
\end{proof}
\subsection{The descent lemma for smooth and strongly convex problems}
In this subsection, we state the descent lemma and provide a proof.
\begin{lemma}[Descent lemma for smooth and strongly convex problems]\label{lemma:descent lemma}
Assume that  Assumptions~\ref{assumption:smooth} to~\ref{assumption:cvx} hold.  Further assume that $\tau(1,i)=0$ for all $i \in [N]$. For any learning rate satisfying $\eta_t \leq \frac{1}{25KL}$, i.e., $\teta \leq \frac{1}{25L}$, the updates of \algoname[] satisfy:
\begin{align}
\begin{aligned}
    \Delta_{t+1} 
   & \leq \left(1-\frac{1}{2}\mu \teta\right)\Delta_t-\frac{44}{25}\teta\left(\EE{f(\w{t})}-f(\wopt)\right)\\
   &\quad+ \frac{2\teta\sigma^2}{KN^2}\nsum\sum_{j=\dl{t}}^{t-1}\teta[j] + \frac{3\teta^2\sigma^2}{KN}
    +\frac{53}{50}\calH+\calS\mathcal{Q},
\end{aligned}\label{ineq:descent lemma}
\end{align}
where
\begin{align*}
   \calH
    &=\frac{64L^3(K-1)^2}{K^2N}\teta \nsum \teta[\dl{t}]^2\Delta_{\dl{t}}+\frac{16L^3}{N^2}\teta \nsum \tau(t,i) \left(\sum_{j=\dl{t}}^{t-1}\psum \teta[j]^2 \Delta_{\dl{j}}\right)\\
    &\quad + \frac{16L D}{N}\teta \nsum \tau(t,i)\left(\sum_{j=\dl{t}}^{t-1}\teta[j]^2 \right) + \frac{64(K-1)^2 L}{K^2 N}\teta \nsum \teta[\dl{t}]^2\norm{\nabla f_i(\wopt)}^2\\
    &\quad + \frac{16(K-1)^2L\sigma^2}{K^3 N}\teta \nsum \teta[\dl{t}]^2
    + \frac{8L\sigma^2}{KN}\teta\nsum \tau(t,i)\left(\sum_{j=\dl{t}}^{t-1}\teta[j]^2 \right)\\
    &\quad +\frac{64L^5(K-1)^2}{K^2N^2}\teta\nsum \tau(t,i)\left( \sum_{j=\dl{t}}^{t-1} \psum \teta[j]^2\teta[\pdl{j}]^2\Delta_{\pdl{j}}\right)\\
    & \quad  + \frac{64L^3(K-1)^2}{K^2N^2}\teta\nsum\tau(t,i)\left(\sum_{j=\dl{t}}^{t-1}\psum \teta[j]^2 \teta[\pdl{j}]^2 \norm{\nabla f_{i'}(\wopt)}^2\right)\\
    & \quad + \frac{16L^3(K-1)^2\sigma^2}{K^3 N^2}\teta\nsum\tau(t,i)\left(\sum_{j=\dl{t}}^{t-1}\teta[j]^2\psum \teta[\pdl{j}]^2 \right),
\end{align*}
and
\begin{align*}
    &\calS\mathcal{Q} = \frac{2\sigma L\teta}{N}\nsum \sqrt{\frac{\tau(t,i)}{N}\sum_{j=\dl{t}}^{t-1}\teta[j]^2 \left(\psum \Delta_{\pdl{j}}\right)}+\\
   &  \frac{2\sigma L\teta}{N}\nsum \sqrt{\frac{\tau(t,i)(K-1)^2}{K^2N}\sum_{j=\dl{t}}^{t-1}\teta[j]^2\psum \teta[\pdl{j}]^2\left(8L^2\Delta_{\pdl{j}}+8\norm{\nabla f_{i'}(\wopt)}^2+\frac{2\sigma^2}{K}\right)}.
\end{align*}
\end{lemma}
\paragraph{Proof of the descent lemma.}
According to the update rule in \eqref{eq:update rule}, we can expand $\norm{\w{t+1}-\wopt}^2$ as
\begin{align*}
    \norm{\w{t+1} - \wopt}^2 &= \norm{\w{t} - \wopt- \frac{\teta}{KN}\knsum \sg f_i(\w{\dl{t}, k}^i)}^2\\
    & = \norm{\w{t}-\wopt}^2 \underbrace{- \frac{2\teta}{KN}\knsum \dotp{\sg f_i(\w{\dl{t}, k}^i)}{\w{t}-\wopt}}_{\calA{1}}\\
    & \quad +\underbrace{\frac{\teta^2}{K^2N^2}\norm{\knsum \sg f_i(\w{\dl{t}, k}^i)}^2}_{\calA{2}}.
\end{align*}
To bound the expectation of $\norm{\w{t+1}-\wopt}^2$, we bound expectations of $\calA{1}$ and $\calA{2}$ respectively.
\subsubsection{Bounding the first term}\label{sec: A1}
Note that $\sg f_i(\w{\dl{t},k}^i)$ can be expanded as $\nabla f_i(\w{\dl{t},k}^i) + e^i_{\dl{t},k}$. Thus, $\calA{1}$ can be split as
\begin{align*}
    \calA{1}  &  = -\frac{2\teta}{KN}\knsum \dotp{\nabla f_i(\w{\dl{t}, k}^i)}{\w{t}-\wopt}  -\frac{2\teta}{KN}\knsum \dotp{e^i_{\dl{t}, k}}{\w{t}-\wopt}.
\end{align*}
Due to reuse of noisy updates, $e^i_{\dl{t},k}$ is correlated with $\w{t}$ and $\EE{\dotp{e^i_{\dl{t}, k}}{\w{t}-\wopt}}$ is not necessarily zero. Further expanding $\w{t}-\wopt$ as $(\w{t}-\w{\dl{t},k}^i)+(\w{\dl{t},k}^i-\wopt)$, we obtain
\begin{align*}
      \calA{1} &=\underbrace{ -\frac{2\teta}{KN}\knsum \dotp{\nabla f_i(\w{\dl{t}, k}^i)}{\w{t}-\wopt}}_{\calB{1}}\underbrace{-\frac{2\teta}{KN}\knsum \dotp{e^i_{\dl{t}, k}}{\w{t}-\w{\dl{t}, k}^i}}_{\calB{2}} \\
    &\quad \underbrace{- \frac{2\teta}{KN}\knsum \dotp{e^i_{\dl{t}, k}}{\w{\dl{t}, k}^i-\wopt}}_{\calB{3}}.
\end{align*}
Due to independence of $e^i_{\dl{t}, k}$ and $\w{\dl{t}, k}^i$, we have $\EE{\calB{3}}=0$.
By \Cref{lemma: perturb}, 
\begin{align*}
   \calB{1} & \leq -2\teta \left(f(\w{t}) - f(\wopt)\right) - \frac{\mu\teta}{2}\norm{\w{t}-\wopt}^2
   + \underbrace{\frac{2L\teta}{KN}\knsum \norm{\w{\dl{t}, k}^i-\w{t}}^2}_{\calC{1}}.
\end{align*}
To estimate the bound for $\calC{1}$, we take a closer look at one summand of $\calC{1}$. Note that $(\w{\dl{t}, k}^i - \w{t})$ can be split in the following way.
\begin{align*}
    & \quad \w{\dl{t}, k}^i - \w{t} \\
    &= \w{\dl{t}, k}^i - \w{\dl{t}} + \w{\dl{t}} - \w{t}\\
    &= \w{\dl{t}, k}^i - \w{\dl{t}} - \frac{1}{KN}\sum_{j=\dl{t}}^{t-1} \teta[j]\pknsum \sg f_{i'}(\w{\pdl{j}, k'}^{i'})\\
    & = \w{\dl{t}, k}^i - \w{\dl{t}} - \frac{1}{KN}\sum_{j=\dl{t}}^{t-1} \teta[j]\pknsum\left( \nabla  f_{i'}(\w{\pdl{j}, k'}^{i'})-\nabla f_{i'}(\w{\pdl{j}})\right)\\
    &\quad-\frac{1}{KN}\sum_{j=\dl{t}}^{t-1}\teta[j]\psum \nabla f_{i'}(\w{\pdl{j}}) - \frac{1}{N}\sum_{j=\dl{t}}^{t-1} \teta[j]\pknsum e^{i'}_{\pdl{j}}.
\end{align*}
By Jensen's inequality, we expand $\norm{\w{\dl{t}, k}^i - \w{t}}^2$ as four parts.
\begin{align*}
     \norm{\w{\dl{t}, k}^i - \w{t}}^2 &\leq \underbrace{4\norm{\w{\dl{t},k}^i - \w{\dl{t}}}^2}_{\calD{1}} \\
     & \quad + \underbrace{4\norm{\frac{1}{KN}\sum_{j=\dl{t}}^{t-1} \teta[j]\left[\pknsum\left( \nabla  f_{i'}(\w{\pdl{j}, k'}^{i'})-\nabla f_{i'}(\w{\pdl{j}})\right)\right]}^2}_{\calD{2}}\\
     & \quad +\underbrace{ 4\norm{\frac{1}{KN}\sum_{j=\dl{t}}^{t-1}\teta[j]\pknsum \nabla f_{i'}(\w{\pdl{j}})}^2}_{D_3} \\
     &\quad + \underbrace{4\norm{\frac{1}{KN}\sum_{j=\dl{t}}^{t-1} \teta[j]\pknsum e^{i'}_{\pdl{j},k'}}^2}_{\calD{4}}.
\end{align*}
According to \Cref{lemma: bounded drift}, for $k=0$, $\calD{1}=0$. For $k\geq 1$,
\begin{align*}
    \EE{\calD{1}} & \leq \frac{32 L^2(K-1)\teta[\dl{t}]^2}{K} \Delta_{\dl{t}}\\
    &\quad + \frac{32(K-1)\teta[\dl{t}]^2}{K}\norm{\nabla f_i(\wopt)}^2+\frac{8(K-1)\teta[\dl{t}]^2\sigma^2}{K^2}.
\end{align*}
Repeatedly applying Jensen's inequality and further using $L$-smoothness,
\begin{align*}
    \EE{\calD{2}} & \leq \frac{4\tau(t,i)}{K^2N^2}\sum_{j=\dl{t}}^{t-1}\teta[j]^2\norm{\pknsum\left( \nabla  f_{i'}(\w{\pdl{j}, k'}^{i'})-\nabla f_{i'}(\w{\pdl{j}})\right)}^2\\
     & \leq \frac{4\tau(t,i)}{K N}\sum_{j=\dl{t}}^{t-1}\teta[j]^2\pknsum\norm{ \nabla  f_{i'}(\w{\pdl{j}, k'}^{i'})-\nabla f_{i'}(\w{\pdl{j}})}^2\\
    & \leq \frac{4L^2\tau(t,i)}{KN} \sum_{j=\dl{t}}^{t-1}\teta[j]^2 \pknsum \norm{\w{\pdl{j}, k'}^{i'}-\w{\pdl{j}}}^2.
\end{align*}
By \Cref{lemma: bounded drift},
\begin{align*}
    \EE{\calD{2}} &\leq \frac{32L^4\tau(t,i)(K-1)^2}{K^2N}\sum_{j=\dl{t}}^{t-1}\teta[j]^2 \psum \teta[\pdl{j}]^2\Delta_{\pdl{j}}\\
    & \quad  + \frac{32L^2\tau(t,i)(K-1)^2}{K^2N}\sum_{j=\dl{t}}^{t-1}\teta[j]^2\psum \teta[\pdl{j}]^2 \norm{\nabla f_{i'}(\wopt)}^2\\
    & \quad + \frac{8L^2\tau(t,i)(K-1)^2\sigma^2}{K^3 N}\sum_{j=\dl{t}}^{t-1}\teta[j]^2\psum \teta[\pdl{j}]^2. 
\end{align*}
Expanding $\calD{3}$ by Jensen's inequality and applying \Cref{lemma:diff},
\begin{align*}
    \EE{\calD{3}} &\leq \frac{4\tau(t,i)}{KN}\sum_{j=\dl{t}}^{t-1}\teta[j]^2\pknsum \norm{\nabla f_{i'}(\w{\dl{j}})}^2\\
    & \leq \frac{8\tau(t,i)L^2}{N}\sum_{j=\dl{t}}^{t-1}\teta[j]^2 \psum \Delta_{\dl{j}} + 8\tau(t,i) D \sum_{j=\dl{t}}^{t-1}\teta[j]^2.
\end{align*}
Due to independence of $e^i_{j,k}$ and $e^{i'}_{j,k'}$ for $i'\neq i$ or $k\neq k'$, $\EE{\norm{\pknsum e^{i'}_{\pdl{j}, k'}}^2}\leq KN\sigma^2$. Still by Jensen's inequality, the expectation of $\calD{4}$ can be bounded as follows.
\begin{align*}
    \EE{\calD{4}}& \leq \frac{4\tau(t,i) }{K^2N^2}\sum_{j=\dl{t}}^{t-1}\norm{\teta[j]\pknsum e^{i'}_{\pdl{j}, k'}}^2\leq \frac{4\tau(t,i)\sigma^2}{KN}\sum_{j=\dl{t}}^{t-1}\teta[j]^2.
\end{align*}
Intuitively, $\calD{1}$ quantifies the drift induced by multiple local steps. $\calD{3}$ and $D_4$ correspond to errors caused by inactivity.  $\calD{2}$ is induced by both local steps and inactivity. Note that $\calD{2}$ to $\calD{4}$ vanish when $\tau(t,i) = 0$ and $\calD{1}$ and that $\calD{2}$ vanish when $K=1$. Combining the expectation of $\calD{1}$ to $\calD{4}$, we have
\begin{align*}
    &\quad \EE{\norm{\w{\dl{t}, k}^i - \w{t}}^2 } \\
    &\leq  \frac{32 L^2(K-1)^2\teta[\dl{t}]^2}{K^2} \Delta_{\dl{t}}
    + \frac{8\tau(t,i)L^2}{N}\sum_{j=\dl{t}}^{t-1}\teta[j]^2 \psum \Delta_{\dl{j}}\\
    &\quad +  8\tau(t,i) D \sum_{j=\dl{t}}^{t-1}\teta[j]^2 + \frac{32(K-1)^2\teta[\dl{t}]^2}{K^2}\norm{\nabla f_i(\wopt)}^2\\
    &\quad +\frac{8(K-1)^2\teta[\dl{t}]^2\sigma^2}{K^3} + \frac{4\tau(t,i)\sigma^2}{KN}\sum_{j=\dl{t}}^{t-1}\teta[j]^2\\
    &\quad +\frac{32L^4\tau(t,i)(K-1)^2}{K^2N}\sum_{j=\dl{t}}^{t-1}\teta[j]^2 \psum \teta[\pdl{j}]^2\Delta_{\pdl{j}}\\
    & \quad  + \frac{32L^2\tau(t,i)(K-1)^2}{K^2N}\sum_{j=\dl{t}}^{t-1}\teta[j]^2\psum \teta[\pdl{j}]^2 \norm{\nabla f_{i'}(\wopt)}^2\\
    & \quad + \frac{8L^2\tau(t,i)(K-1)^2\sigma^2}{K^3 N}\sum_{j=\dl{t}}^{t-1}\teta[j]^2\psum \teta[\pdl{j}]^2 .
\end{align*} 
Since when $i$ and $t$ are fixed, $\EE{\norm{\w{\dl{t}, k}^i - \w{t}}^2 }$ can be uniformly bounded for all $0 \leq k \leq K-1 $, we can bound the expectation of $\calC{1}$.
\begin{align*}
    &\quad\EE{\calC{1}}\\
    &\leq\frac{64L^3(K-1)^2}{K^2N}\teta \nsum \teta[\dl{t}]^2\Delta_{\dl{t}}+\frac{16L^3}{N^2}\teta \nsum \tau(t,i) \left(\sum_{j=\dl{t}}^{t-1}\psum \teta[j]^2 \Delta_{\dl{j}}\right)\\
    &\quad + \frac{16L D}{N}\teta \nsum \tau(t,i)\left(\sum_{j=\dl{t}}^{t-1}\teta[j]^2 \right) + \frac{64(K-1)^2 L}{K^2 N}\teta \nsum \teta[\dl{t}]^2\norm{\nabla f_i(\wopt)}^2\\
    &\quad + \frac{16(K-1)^2L\sigma^2}{K^3 N}\teta \nsum \teta[\dl{t}]^2
    + \frac{8L\sigma^2}{KN}\teta\nsum \tau(t,i)\left(\sum_{j=\dl{t}}^{t-1}\teta[j]^2 \right)\\
    &\quad +\frac{64L^5(K-1)^2}{K^2N^2}\teta\nsum \tau(t,i)\left( \sum_{j=\dl{t}}^{t-1} \psum \teta[j]^2\teta[\pdl{j}]^2\Delta_{\pdl{j}}\right)\\
    & \quad  + \frac{64L^3(K-1)^2}{K^2N^2}\teta\nsum\tau(t,i)\left(\sum_{j=\dl{t}}^{t-1}\psum \teta[j]^2 \teta[\pdl{j}]^2 \norm{\nabla f_{i'}(\wopt)}^2\right)\\
    & \quad + \frac{16L^3(K-1)^2\sigma^2}{K^3 N^2}\teta\nsum\tau(t,i)\left(\sum_{j=\dl{t}}^{t-1}\teta[j]^2\psum \teta[\pdl{j}]^2 \right).
\end{align*}
Here we denote the RHS of the above inequality as $\calH$. Therefore,
\begin{align}
    \EE{\calB{1}}\leq -\frac{1}{2}\mu\teta\Delta_t-2\teta \left(\EE{f(\w{t})} - f(\wopt)\right) + \calH.
\end{align}\label{ineq:B1}
Next we estimate the bound for $\calB{2}. $Unrolling one summand of $\calB{2}$,
\begin{align*}
-\dotp{e^i_{\dl{t}, k}}{\w{t}-\w{\dl{t}, k}^i} &=  \underbrace{-\dotp{e^i_{\dl{t}, k}}{\w{t}-\w{\dl{t}}}}_{\calC{2}}\\
&\quad \underbrace{-\dotp{e^i_{\dl{t}, k}}{ \w{\dl{t}}
 -\w{\dl{t}, k}^i}}_{\calC{3}}.
\end{align*}
Due to independence of $e^i_{\dl{t}, k}$ and $\w{\dl{t}}
 -\w{\dl{t}, k}^i$, $\EE{\calC{3}}=0$.
Then we turn to $\calC{2}$,
\begin{align*}
    \calC{2} &= \frac{1}{KN}\dotp{e_{\dl{t}, k}^i}{\sum_{j=\dl{t}}^{t-1}\teta[j]\pknsum \left(\nabla f_{i'}(\w{\pdl{j}, k'-1}^{i'})+e_{\pdl{j}, k'}^{i}\right)}\\
   & =  \underbrace{\frac{1}{KN} \dotp{e_{\dl{t}, k}^i}{\sum_{j=\dl{t}}^{t-1}\teta[j] e_{\dl{j}, k}^i}}_{\calD{5}} \\
   &\quad + \underbrace{\frac{1}{KN}\dotp{e^i_{\dl{t}, k}}{\sum_{j=\dl{t}}^{t-1}\teta[j] \sum_{\substack{k'\neq k\\ \text{or\ } i'\neq i}}e_{\pdl{j}, k'}^{i'}}}_{\calD{6}}\\
    & \quad +\underbrace{\frac{1}{KN}\dotp{e_{\dl{t}, k}^i}{\sum_{j=\dl{t}}^{t-1}\teta[j]\pknsum \nabla f_{i'}(\w{\pdl{j}, k'}^{i'})}}_{{\calD{7}}}.
\end{align*}
Applying the identity $e^i_{\dl{t}, k} = e^i_{\dl{t-1},k}=\cdots=e^i_{\dl{\dl{t}}, k}$, the expectation of $\calD{5}$ can be bounded by
\begin{align*}
   \EE{\calD{5}} \leq \frac{\sigma^2}{KN}\sum_{j=\dl{t}}^{t-1}\teta[j].
\end{align*}
Due to independence of $e^i_{j,k}$ and $e^{i'}_{j,k'}$ for $i'\neq i$ or $k\neq k'$, $\EE{\calD{6}}=0$. Note that $\nabla f_{i'}(\w{\pdl{j}, k'}^{i'})$ can be split as $(\nabla f_{i'}(\w{\pdl{j}, k'}^{i'}) - \nabla f_{i'}(\w{\pdl{j}}))+(\nabla f_{i'}(\w{\pdl{j}}) - \nabla f_{i'}(\wopt))$, where the first part is the difference between the gradient on the local parameter and on the global parameter, and the second part is the difference between the gradient on the global parameter and  on the global optimum. By Cauchy-Schwartz inequality $\EE{\dotp{X}{Y}}\leq \sqrt{\EE{\norm{X}^2}\EE{\norm{Y}^2}}$,  we bound the expectation of $\calD{7}$,
\begin{align*}
   &\quad\EE{\calD{7} }\\
   &= \frac{1}{KN}\EE{\dotp{e_{\dl{t}, k}^i}{\sum_{j=\dl{t}}^{t-1}\teta[j]\pknsum \left(\nabla f_{i'}(\w{\pdl{j}, k'}^{i'}) - \nabla f_{i'}(\w{\pdl{j}}) \right)} }\\
   &\quad + \frac{1}{KN}\EE{\dotp{e_{\dl{t}, k}^i}{\sum_{j=\dl{t}}^{t-1}\teta[j]\pknsum \left(\nabla f_{i'}(\w{\pdl{j}}) - \nabla f_{i'}(\wopt) \right)} }\\
   & \leq \frac{1}{KN} \sigma \sqrt{\EE{\norm{\sum_{j=\dl{t}}^{t-1}\teta[j]\pknsum \left(\nabla f_{i'}(\w{\pdl{j}, k'}^{i'}) - \nabla f_{i'}(\w{\pdl{j}}) \right)}^2 }}\\
   & \quad+ \frac{1}{KN}\sigma \sqrt{\EE{\norm{\sum_{j=\dl{t}}^{t-1}\teta[j]\pknsum \left( \nabla f_{i'}(\w{\pdl{j}})-\nabla f_{i'}(\wopt) \right)}^2 }}.
\end{align*}
By Jensen's inequality and $L$-smoothness, the term inside the first square root can be bounded as follows.
\begin{align*}
    & \quad \norm{\sum_{j=\dl{t}}^{t-1}\teta[j]\pknsum \left(\nabla f_{i'}(\w{\pdl{j}, k'}^{i'}) - \nabla f_{i'}(\w{\pdl{j}}) \right)}^2\\
    &  \leq \tau(t,i)\sum_{j=\dl{t}}^{t-1}\teta[j]^2\norm{\pknsum \left(\nabla f_{i'}(\w{\pdl{j}, k'}^{i'}) - \nabla f_{i'}(\w{\pdl{j}}) \right)}^2 \\
    & \leq KN\tau(t,i)\sum_{j=\dl{t}}^{t-1}\teta[j]^2 \left(\pknsum\norm{\nabla f_{i'}(\w{j,k'}^{i'}) - \nabla f_{i'}(\w{j})}^2 \right)\\
    & \leq KNL^2 \tau(t,i)\sum_{j=\dl{t}}^{t-1}\teta[j]^2 \left(\pknsum\norm{\w{\pdl{j},k'}^{i'}- \w{\pdl{j}}}^2\right).
\end{align*}
Similarly, the term inside the second square root can be bounded as follows.
\begin{align*}
    &\quad \norm{\sum_{j=\dl{t}}^{t-1}\teta[j]\pknsum \left( \nabla f_{i'}(\w{\pdl{j}})-\nabla f_{i'}(\wopt) \right)}^2 \\
    & \leq  K^2NL^2 \tau(t,i)\sum_{j=\dl{t}}^{t-1}\teta[j]^2 \left(\psum\norm{\w{\pdl{j}} - \wopt}^2\right).
\end{align*}
Therefore,
\begin{align*}
    &\quad\EE{\calD{7}}\\
   &\leq \sigma L \sqrt{\frac{\tau(t,i)}{KN}\sum_{j=\dl{t}}^{t-1}\teta[j]^2\pknsum\EE{ \norm{\w{\pdl{j},k'}^{i'}- \w{\pdl{j}}}^2}}\\
   &\quad + \sigma L \sqrt{\frac{\tau(t,i)}{N}\sum_{j=\dl{t}}^{t-1}\teta[j]^2 \left(\psum \Delta_{\pdl{j}}\right)}\\
   &\leq \sigma L \sqrt{\frac{\tau(t,i)(K-1)^2}{K^2N}\sum_{j=\dl{t}}^{t-1}\teta[j]^2\psum \teta[\pdl{j}]^2\left(8L^2\Delta_{\pdl{j}}+8\norm{\nabla f_{i'}(\wopt)}^2+\frac{2\sigma^2}{K}\right)}\\
   &\quad + \sigma L \sqrt{\frac{\tau(t,i)}{N}\sum_{j=\dl{t}}^{t-1}\teta[j]^2 \left(\psum \Delta_{\pdl{j}}\right)},
\end{align*}
where the second inequality uses \Cref{lemma: bounded drift}. Combining the expectation of $\calD{5}$ to $\calD{7}$, we have
\begin{align*}
   &\quad \EE{\calC{2} }\\
   &\leq \frac{\sigma^2}{KN}\sum_{j=\dl{t}}^{t-1}\teta[j]+\sigma L \sqrt{\frac{\tau(t,i)}{N}\sum_{j=\dl{t}}^{t-1}\teta[j]^2 \left(\psum \Delta_{\pdl{j}}\right)}\\
   & \quad+ \sigma L \sqrt{\frac{\tau(t,i)(K-1)^2}{K^2N}\sum_{j=\dl{t}}^{t-1}\teta[j]^2\psum \teta[\pdl{j}]^2\left(8L^2\Delta_{\pdl{j}}+8\norm{\nabla f_i(\wopt)}^2+\frac{2\sigma^2}{K}\right)}.
\end{align*}
The first term can be interpreted as the accumulated noise due to reuse of noisy gradients. The expression inside the square root of the second term stands  for the effect of inactivity and it vanishes when $\tau(t,i)=0$. The expression inside the square root of the second term stands for the effect of unavailability and local updates and it vanishes when $\tau(t,i) = 0$ or $K=1$. Then the expectation of $\calB{2}$ can be bounded by
\begin{align}
\begin{aligned}
   &\quad \EE{\calB{2} }\\
   &\leq \frac{2\teta\sigma^2}{KN^2}\nsum\sum_{j=\dl{t}}^{t-1}\teta[j]+\underbrace{\frac{2\sigma L\teta}{N}\nsum \sqrt{\frac{\tau(t,i)}{N}\sum_{j=\dl{t}}^{t-1}\teta[j]^2 \left(\psum \Delta_{\pdl{j}}\right)}}_{\calS\mathcal{Q}_1}+\\
   & \underbrace{\frac{2\sigma L\teta}{N}\nsum \sqrt{\frac{\tau(t,i)(K-1)^2}{K^2N}\sum_{j=\dl{t}}^{t-1}\teta[j]^2\psum \teta[\pdl{j}]^2\left(8L^2\Delta_{\pdl{j}}+8\norm{\nabla f_i(\wopt)}^2+\frac{2\sigma^2}{K}\right)}}_{\calS\mathcal{Q}_2}.
\end{aligned}\label{ineq:B2}
\end{align}
Combining \eqref{ineq:B1} and \eqref{ineq:B2}, we bound the expectation of $\calA{1}$.
\begin{align*}
     \EE{\calA{1} } 
    & \leq \EE{\calB{1} }+ \EE{\calB{2} }\\
    &\leq  - \frac{\mu\teta}{2}\Delta_t -2\teta \left(\EE{f(\w{t})} - f(\wopt)\right) +\frac{2\teta\sigma^2}{KN^2}\nsum\sum_{j=\dl{t}}^{t-1}\teta[j] + \calH + \calS\mathcal{Q},
\end{align*}
where $\calS\mathcal{Q} = \calS\mathcal{Q}_1+\calS\mathcal{Q}_2$.
\subsubsection{Bounding the second term}\label{sec: A2}
Note that $\sg f_i(\w{\dl{t}, k}^i)$ can be split into three terms, i.e., 
\begin{align*}
    \sg f_i(\w{\dl{t}, k}^i) 
    & =\left( \nabla f_i(\w{\dl{t}, k}^i) - \nabla f_i(\w{t}) \right) +\nabla  f_i(\w{t}) + e^i_{\dl{t}, k}.
\end{align*}
By Jensen's inequality,
\begin{align*}
    \EE{\calA{2} } &\leq \underbrace{\frac{3\teta^2}{K^2N^2}\EE{\norm{\knsum \left( \nabla f_i(\w{\dl{t}, k}^i)-\nabla f_i(\w{t})\right)}^2}}_{\calB{4}}\\
    & \quad +  \underbrace{\frac{3\teta^2}{K^2N^2}\EE{\norm{\knsum \nabla f_i(\w{t})}^2 }}_{\calB{5}} + \underbrace{\frac{3\teta^2}{K^2N^2}\EE{\norm{\knsum e_{\dl{t},k}^i}^2}}_{\calB{6}}.
\end{align*}
Due to independence of $e^i_{\dl{t},k}$ and $e^{i'}_{\dl{t},k'}$ for $i\neq i'$ or $k\neq k'$, we have $\calB{6} \leq \frac{3\teta^2\sigma^2}{KN}$. 
Recall $\calC{1} = \frac{2L\teta}{KN}\knsum \norm{\w{\dl{t}, k}^i - \w{t}}^2$. By Jensen's inequality and $L$-smoothness,
 we then bound $\calB{4}$.
\begin{align*}
    \calB{4} 
    & \leq \frac{3L^2\teta^2}{KN}\knsum \EE{\norm{ \w{\dl{t}, k}^i)-\w{t}}^2}=\frac{3}{2}L\teta\EE{\calC{1}}\leq \frac{3}{2}L\teta \calH.
\end{align*}
By \Cref{lemma:smooth}, we have
\begin{align*}
    \calB{5} &\leq 3\teta^2 \EE{\norm{\nabla f(\w{t})}^2 }  \leq 6L\teta^2 \left(\EE{f(\w{t})} - f(\wopt) \right).
\end{align*}
Therefore
\begin{align*}
   \EE{\calA{2} } \leq \frac{3}{2}L\teta\calH + \frac{3\teta^2\sigma^2}{KN} +  6L\teta^2 \left(\EE{f(\w{t}) } - f(\wopt) \right).
\end{align*}
Combining \Cref{sec: A1} and \Cref{sec: A2}, we have
\begin{align*}
   \Delta_{t+1} 
   &\leq \left(1-\frac{1}{2}\mu \teta\right)\Delta_t-2\teta\left(1-3L\teta\right)\left(\EE{f(\w{t})}-f(\wopt)\right) \\
    &\quad + \frac{2\teta\sigma^2}{KN^2}\nsum\sum_{j=\dl{t}}^{t-1}\teta[j] + \frac{3\teta^2\sigma^2}{KN}
    +\left(1+\frac{3}{2}L\teta\right)\calH+\calS\mathcal{Q}.
\end{align*}
Since when $\teta \leq \frac{1}{25L}$, $-2\teta (1-3L\teta) \leq -\frac{44}{25}\teta$ and $\frac{3}{2}L\teta\leq \frac{3}{50}$,   \Cref{lemma:descent lemma} holds.

\subsection{Deriving the convergence bound}
In this subsection, we obtain \Cref{theorem:maintheorem_general} based on the descent lemma. We provide a bound for $\Delta_t$ in \Cref{sec:w} $\forall 1 \leq t \leq T$  and further bound $\EE{f(\taubart)}-f(\wopt)$ in \Cref{sec:proof_general}.
\subsubsection{Bounding the distance from the global optimum}\label{sec:w}
\begin{lemma}[A bound for the expected squared $l_2$-distance from the global optimum]\label{lemma:delta}  Assume that Assumptions~\ref{assumption:smooth} to~\ref{assumption:cvx} hold. Further assume that the device availability sequence $\tau(t,i)$ satisfies \Cref{assumption:delay} and $\tau(t,i)=0$, for all $i\in[N]$. By setting the learning rate $\eta_t = \frac{4}{\mu K(t+a)}$ with $a=\max\{100, 40t_0\}(\frac{L}{\mu})^{1.5}$. For all $ 1\leq t \leq T$, after $ t-1$ communication rounds, $\Delta_{t}$  satisfies:
\begin{align}\label{ineq:delta}
    \Delta_{t} \leq \frac{E s_{t}\sigma^2}{ (t+a)^2} + \frac{G}{t+a} + \frac{F}{(t+a)^2} := B_t,
\end{align}
where 
\begin{align*}
E&=\frac{35\sigma^2}{\mu^2 N K}, G=\frac{32\sigma^2}{\mu^2 N K}, F = \frac{\dmaxt C_1 + (K-1)^2C_2+C_3}{\mu^3 K^2},
\end{align*}
and 
\begin{align*}
    C_1 & = 2500L K^2(D+2L \sigma^2/\mu), 
  C_2 = 5000L(D+\sigma^2/K), C_3  = \max\{1600t_0^2, 10000\} L^3 K^2\Delta_1^2.
\end{align*}
\end{lemma}
 We prove \Cref{lemma:delta} by induction. We first show that \eqref{ineq:delta} holds when $t=1$. Then assuming that $\Delta_{t'}\leq B_{t'}$ holds for all $ 1\leq t' \leq t$, we prove $\Delta_{t+1} \leq B_{t+1}$ by verifying
\begin{align}
    B_{t+1} \stackrel{(a)}{\geq} F(B_{t}, \teta) \stackrel{(b)}{\geq} \text{RHS of }  \eqref{ineq:descent lemma} \geq \Delta_{t+1},
\end{align}
where $F$ is a function of $B_t$ and $\teta$. To validate (b), we prove that for all $ 0\leq m \leq l_t$, $B_{t-m}$ and $\teta[t-m]$ can be bounded by $B_t$ and $\teta$ respectively in  \Cref{sec:relationship}. We simplify terms of higher degree in \Cref{sec:higher degree} and simplify terms with square roots in \Cref{sec:sqrt}. Finally, relation (a) is verified in \Cref{sec:verify(a)}. A formal proof is provided as follows. 
\begin{comment}
leq c B_t$, $\teta[t-k] \leq d \teta$, $\forall 0\leq k \leq l_t$
there exist constant $c$ and $d$ such that $B_{t-k} \leq c B_t$, $\teta[t-k] \leq d \teta$, $\forall 0\leq k \leq l_t
\end{comment}
\paragraph{Proof of \Cref{lemma:delta}.} Note that \eqref{ineq:delta} holds trivially when $t=1$ since $\frac{C_3}{\mu^3 K^2} \geq a^2 \Delta_1$. Now we assume $\forall 1\leq t' \leq t$, $\Delta_{t'}\leq B_{t'}$ holds. 

\subsubsection{Connecting bounds and learning rates at different rounds}\label{sec:relationship}

According to  \Cref{assumption:delay}, $\tau(t,i) \leq t_0 + \frac{1}{40}t$, $l_t \leq 2t_0 + \frac{1}{20}t$. Combining with $t_0\leq \frac{1}{40} a$, we have
\begin{align*}
     \frac{1}{t+a-\tau(t,i)} \leq \frac{40}{39(t+a)}\text{\ and\ } \frac{1}{t+a-l_t}  \leq \frac{20}{19(t+a)}.
\end{align*}

Therefore, for all $0 \leq n \leq \tau(t,i)$, we have
\begin{align*}
  B_{t-n}& = \frac{s_{t-n}E }{(t+a-n)^2} + \frac{G}{t+a-n} +  \frac{F}{(t+a-n)^2} \\
  &\leq \frac{s_{t}E}{(t+a-\tau(t,i))^2} + \frac{G}{t+a-\tau(t,i)}+ \frac{F}{(t+a-\tau(t,i))^2}\\
  & \leq \left(\frac{40}{39}\right)^2 B_t.
\end{align*}
For all $0\leq m \leq l_t$,
\begin{align*}
    B_{t-m} \leq  \frac{s_{t}E}{(t+a-l_t)^2} + \frac{G}{t+a-l_t}+ \frac{F}{(t+a-l_t)^2}\leq \left(\frac{20}{19}\right)^2 B_t,
\end{align*}
and
\begin{align*}
 \teta[t-n] &\leq \teta[\dl{t}] \leq \frac{40}{39} \teta, \forall 0 \leq n \leq \tau(t,i),\\   
 \teta[t-m]& \leq \teta[t-l_t] \leq \frac{20}{19}\teta, \forall 0\leq m \leq l_t.
\end{align*}
Also, we have
\begin{align}
    \frac{2\teta \sigma^2}{KN^2}\nsum \sum_{j=\dl{t}}^{t-1}\teta[j] \leq \frac{80\tau_t \sigma^2}{39KN}\teta^2\label{ineq:sigma square}.
\end{align}
\subsubsection{Simplifying terms of higher degree}\label{sec:higher degree}
In this section, we simplify $\calH$ in \eqref{ineq:descent lemma} and bound it by $B_t$ and $\teta$. Rearranging $\calH$, we have
\begin{align*}
    \calH & =\underbrace{ \frac{16L^3}{N^2}\teta \nsum \tau(t,i) \left(\sum_{j=\dl{t}}^{t-1}\psum \teta[j]^2 \Delta_{\dl{j}}\right)}_{\calI[1]}+ \underbrace{\frac{16L D}{N}\teta \nsum \tau(t,i)\left(\sum_{j=\dl{t}}^{t-1}\teta[j]^2 \right)}_{\calI[2]} \\
    & \quad+ \underbrace{\frac{8L\sigma^2}{KN}\teta\nsum \tau(t,i)\left(\sum_{j=\dl{t}}^{t-1}\teta[j]^2 \right)}_{\calI[3]}\\
    & \quad +\underbrace{ \frac{64L^3(K-1)^2}{K^2N}\teta \nsum \teta[\dl{t}]^2\Delta_{\dl{t}}}_{\calI[4]}\\ &\quad +\underbrace{\frac{64L^5(K-1)^2}{K^2 N^2}\teta\nsum \tau(t,i)\left( \sum_{j=\dl{t}}^{t-1} \psum \teta[j]^2\teta[\pdl{j}]^2\Delta_{\pdl{j}}\right)}_{\calI[5]}\\
    & \quad + \underbrace{\frac{64(K-1)^2L}{K^2N}\teta \nsum \teta[\dl{t}]^2\norm{\nabla f_i(\wopt)}^2}_{\calI[6]}\\
    &\quad + \underbrace{\frac{64L^3(K-1)^2}{K^2 N^2}\teta\nsum\tau(t,i)\left(\sum_{j=\dl{t}}^{t-1}\psum \teta[j]^2 \teta[\pdl{j}]^2 \norm{\nabla f_{i'}(\wopt)}^2\right)}_{\calI[7]}\\
    &  \quad + \underbrace{\frac{16(K-1)^2 L\sigma^2}{K^3 N}\teta \nsum \teta[\dl{t}]^2}_{\calI[8]}\\
    &\quad + \underbrace{\frac{16L^3(K-1)^2\sigma^2}{K^3 N^2}\teta\nsum\tau(t,i)\left(\sum_{j=\dl{t}}^{t-1}\teta[j]^2\psum \teta[\pdl{j}]^2 \right)}_{\calI[9]}.
\end{align*}
We first show that $\calI[1]$, $\calI[4]$ and  $\calI[5]$ can be bounded by $\mu\teta B_t$.  According to \Cref{assumption:delay},
\begin{align}
    \tau(t,i)\teta \leq \frac{4[t_0 + (1/b)t]}{\mu(a + t)}\le  \frac{4[t_0 + (1/b)t]}{\mu(b t_0 + t)}\le  \frac{4}{\mu b}\leq \frac{\mu^{0.5}}{10 L^{1.5}} \leq \frac{1}{10L}. \label{ineq:bound eta}
\end{align}
Combining the result in \Cref{sec:relationship}, we can bound $\calI[1]$ in the following way.
\begin{align*}
     \calI[1] & \leq 16\left(\frac{40}{39}\right)^2\left(\frac{20}{19}\right)^2 L\teta\left( L^2\teta^2\frac{1}{N} \nsum\tau(t,i)^2\right) B_t \leq 0.19\mu \teta B_t .
\end{align*}
Similarly, $\calI[5]$ and $\calI[4]$ can be bounded as follows.
\begin{align*}
    \calI[5] & \leq \frac{64L^5 (K-1)^2}{K^2 N^2}\teta \nsum \tau(t,i) \sum_{j=\dl{t}}^{t-1}\psum \left(\frac{40}{39}\right)^2 \left(\frac{20}{19}\right)^4 \teta^4 B_t 
     \\
     &\leq \frac{83 L^5(K-1)^2}{K^2}\teta^3 \Big(\teta\tau(t,i)\Big)^2 B_t\\
     & \leq \frac{0.83L^3(K-1)^2}{K^2}\teta^3B_t,\\
     \calI[4] & \leq 64 \left(\frac{40}{39}\right)^4 \frac{ L^3(K-1)^2}{K^2}\teta^3B_t
     \leq  \frac{71 L^3(K-1)^2}{K^2}\teta^3B_t.
\end{align*}
Further using  $\teta \le \frac{4}{\mu a} \le \frac{\mu^{0.5}}{25L^{1.5}}$, we have
\begin{align*}
    \calI[4]+\calI[5]\leq \frac{68.13 L^3(K-1)^2}{K^2}\teta^3 B_t \leq \left(\frac{71.83 L^3}{\mu} \teta^2 \right)\mu \teta B_t= 0.12 \mu \teta B_t.
\end{align*}
% We first bound $\calI[5]$ and $\calI[4]$. Note that according to \Cref{assumption:delay}, $\tau(t,i)\teta \leq \frac{4[t_0 + (1/b)t]}{\mu(a + t)}\le  \frac{4[t_0 + (1/b)t]}{\mu(b t_0 + t)}\le  \frac{4}{\mu b}\leq \frac{\mu^{0.5}}{10 L^{1.5}} \leq \frac{1}{10L}$. Applying the result in \Cref{sec:relationship}, we have
% and similarly
% \begin{align*}
%     \calI[4] & \leq  \frac{67.3 L^3(K-1)^2}{K^2}\teta^3B_t
% \end{align*}
%Since $\teta \le \frac{4}{\mu a} \le \frac{\mu^{0.5}}{25L^{1.5}}$, $\calI[4]+\calI[5]\leq \frac{68.13 L^3(K-1)^2}{K^2}\teta^3 B_t \leq \left(\frac{68.13 L^3}{\mu} \teta^2 \right)\mu \teta B_t= 0.11 \mu \teta B_t$ . 
By the same token, 
\begin{align*}
    \calI[6] + \calI[7] &\leq \frac{68.1(K-1)^2L}{K^2}\teta^3 D,\\
    \calI[8]+\calI[9]&\leq \frac{17.03(K-1)^2L}{K^3}\teta^3\sigma^2.
\end{align*}

% $\calI[6] + \calI[7] \leq \frac{68.13(K-1)^2L}{K^2}\teta^3 D$, $\calI[8]+\calI[9]\leq \frac{17.05(K-1)^2L}{K^3}\teta^3\sigma^2$.
%and $\tau(t,i)\teta \le \frac{\mu^{0.5}}{10L^{1.5}}$, we can bound $\calI[1]$ 
Still using the result in \Cref{sec:relationship}, we can bound $\calI[2]$ and $\calI[3]$.
\begin{align*}
    \calI[2] &\leq 16\left(\frac{40}{39}\right)^2  \left(\frac{1}{N}\nsum \tau(t,i)^2\right)LD \teta^3 \le16.84 LD  d_t\teta^3,\\
    \calI[3] &\leq 8\left(\frac{40}{39}\right)^2\left(\frac{1}{N}\nsum \tau(t,i)^2\right) \frac{L\sigma^2}{K}\teta^3\leq \frac{8.42 d_t L\sigma^2 \teta^3}{K}.
\end{align*}
Therefore, 
\begin{align}
\begin{aligned}
    \calH &\leq 0.31\mu\teta B_t  + 16.84 L D d_t\teta^3 + \frac{68.1(K-1)^2L}{K^2}\teta^3 D \\
    & \quad + \frac{8.42 d_t L\sigma^2 \teta^3}{K} + \frac{17.03(K-1)^2L}{K^3}\teta^3 \sigma^2
\end{aligned}    \label{ineq:h}
\end{align}
\subsubsection{Simplifying terms with square roots}\label{sec:sqrt}
In this section, we bound terms with square roots on RHS of \eqref{ineq:descent lemma}, i.e., $\calS\mathcal{Q}$. We apply the results in \Cref{sec:relationship} to bound the first term.
\begin{align*}
  &\quad 2\sigma \teta L\frac{1}{N}\nsum \sqrt{\tau(t,i)\sum_{j=\dl{t}}^{t-1}\teta[j]^2\frac{1}{N}\psum\Delta_{\pdl{j}}}\\
  &\leq\frac{80}{39} \sigma \teta^2 L\frac{1}{N}\nsum \sqrt{\tau(t,i) \frac{1}{N} \sum_{ j=\dl{t}}^{t-1}\psum B_{\pdl{j}}}\\
  &\leq \frac{80}{39}\sigma \teta^2\tau_t L\frac{1}{N}\nsum\sqrt{\left(\frac{20}{19}\right)^2B_{t}}\\
  & \leq 2.16\sigma \teta^2 \tau_t L \sqrt{B_t}.
\end{align*}
%where the last inequality applies the results in \Cref{sec:relationship}. 
Recall
\begin{align*}
    B_t \ge \frac{\dmaxt C_1}{\mu^3 (t+a)^2}\ge \frac{5000 \dmaxt L^2\sigma^2}{\mu^4 (t+a)^2}.
\end{align*}
Since $\tau_t^2 = \left[\frac{1}{N}\nsum \tau(t,i)\right]^2 \leq \frac{1}{N} \nsum \tau(t,i)^2\leq \dmaxt$, we have
\begin{align*}
    \sqrt{B_t}\ge \frac{70\tau_t L \sigma}{\mu^2 (t+a)} \geq \frac{1}{\mu}\cdot 8 \cdot \frac{4}{\mu (t+a)} (2.16 \sigma\tau_t L)=\frac{8}{\mu}(2.16 \sigma \teta\tau_t L).
\end{align*}
Therefore,
\begin{align*}
    \frac{1}{8}\mu \teta B_t \ge 2.16\sigma\teta^2\tau_t L\sqrt{B_t}.
\end{align*}
Next, we bound the second term.
\begin{align*}
    & \frac{2\sigma L\teta}{N}\nsum \sqrt{\frac{\tau(t,i)(K-1)^2}{K^2N}\sum_{j=\dl{t}}^{t-1}\teta[j]^2\psum \teta[\pdl{j}]^2\left(8L^2\Delta_{\pdl{j}}+8\norm{\nabla f_{i'}(\wopt)}^2+\frac{2\sigma^2}{K}\right)}\\
    & \leq \frac{2\sigma L\teta}{N}\nsum \sqrt{\frac{\tau(t,i)(K-1)^2}{K^2N}\sum_{j=\dl{t}}^{t-1}\teta[t-\tau(t,i)]^2\psum \teta[t-l_t]^2\left(8L^2 B_{\pdl{j}}+8\norm{\nabla f_{i'}(\wopt)}^2+\frac{2\sigma^2}{K}\right)}\\
    & \leq \frac{2.16\sigma L\teta^3}{N}\nsum\sqrt{\frac{\tau(t,i)^2(K-1)^2}{K^2}\left(8\left(\frac{20}{19}\right)^2 L^2 B_{t} + 8D + \frac{2\sigma^2}{K} \right)}\\
    & \leq  \frac{(K-1)2.16\sigma L\tau_t\teta^3}{K}\sqrt{\left(8\left(\frac{20}{19}\right)^2L^2 B_{t} + 8D + \frac{2\sigma^2}{K} \right)}.
\end{align*}
To show that $ \frac{(K-1)2.16\sigma L\tau_t\teta^3}{K}\sqrt{\left(8\left(\frac{20}{19}\right)^2L^2 B_{t} + 8D + \frac{2\sigma^2}{K} \right)} \le \frac{1}{8}\mu\teta B_t$,
we only have to prove
\begin{align}
    B_t^2 &\geq \frac{64(2.16)^2(K-1)^2\sigma^2 L^2 \tau_t^2}{\mu^2 K^2}\teta^4 \left[8\left(\frac{20}{19}\right)^2L^2B_t + 8D + \frac{2\sigma^2}{K}\right].\label{ineq:sqrt1}
\end{align}
To let \eqref{ineq:sqrt1} hold, we only have to verify
\begin{align}
    \frac{1}{2}B_t &\geq 2647\frac{(K-1)^2\sigma^2 L^4 \tau_t^2}{\mu^2K^2}\teta^4,\label{ineq:sqrt2}\\
    \frac{1}{4}B_t^2 &\geq 2500\frac{(K-1)^2\sigma^2 L^2 \tau_t^2D}{\mu^2 K^2}\teta^4 \Leftrightarrow B_t \ge \frac{ 100(K-1)\sigma L \tau_t \sqrt{D}}{\mu K}\teta^2=\frac{1600(K-1)\sigma L\tau_t\sqrt{D}}{\mu^3K(t+a)^2},\label{ineq:sqrt3}\\
     \frac{1}{4}B_t^2 &\geq 625 \frac{(K-1)^2\sigma^2 L^2 \tau_t^2}{\mu^2 K^2}\teta^4 (\frac{\sigma^2}{K})\Leftrightarrow B_t \geq \frac{50(K-1)L\sigma\tau_t}{\mu K}\teta^2 \frac{\sigma}{\sqrt{K}}=\frac{800(K-1)L\sigma^2\tau_t}{\mu^3K^{1.5}(t+a)^2}.\label{ineq:sqrt4}
\end{align}
Since $\teta \leq \frac{\mu^{0.5}}{25L^{1.5}}$, we have
\begin{align*}
    2647\frac{(K-1)^2\sigma^2 L^4 \tau_t^2}{\mu^2K^2}\teta^4 \leq \frac{4.24 (K-1)^2 L \dmaxt \sigma^2}{\mu^3 K^2(t+a)^2} \leq \frac{2500 \dmaxt L^2\sigma^2}{\mu^4 (t+a)^2} \leq \frac{1}{2}B_t.
\end{align*}
Therefore \eqref{ineq:sqrt2} holds. Also note that
\begin{align*}
    \frac{1600(K-1)\sigma L \tau_t \sqrt{D}}{\mu^3 K(t+a)2}&=\frac{1600L}{\mu^3(t+a)^2}\left[\left(\frac{K-1}{K}\sqrt{D}\right)\left(\tau_t \sigma\right) \right]\\
    & \leq \frac{800L(K-1)^2D}{\mu^3K^2(t+a)^2}+\frac{800L\dmaxt\sigma^2}{\mu^3(t+a)^2}\\
    & \leq B_t.
\end{align*}
Hence, \eqref{ineq:sqrt3} holds. Similarly,
\begin{align*}
   \frac{800(K-1)L\sigma^2\tau_t}{\mu^3K^{1.5}(t+a)^2}=\frac{800\sigma^2L}{\mu^3(t+a)^2}\left[\left(\frac{K-1}{K^{1.5}}\right)\right]\tau_t\leq \frac{400 L(K-1)^2\sigma^2}{\mu^3 K^3(t+a)^2}+\frac{400L\dmaxt \sigma^2}{\mu^3(t+a)^2}\le B_t.
\end{align*}
Therefore, \eqref{ineq:sqrt4} holds. Now we have obtained a bound for $\mathcal{SQ}$. That is,
\begin{align}
\mathcal{SQ}\leq \frac{1}{4}\mu\teta B_t\label{ineq:s}.
\end{align}
\subsubsection{Verifying relation (b)}
In this subsection, we verify relation (b) by using the results in \Cref{sec:relationship}, \Cref{sec:higher degree} and \Cref{sec:sqrt}.
First apply the definition of strong convexity and therefore,
\begin{align}
    \EE{f(\w{t})}-f(\wopt)  \geq \frac{\mu}{2}\Delta_t\label{ineq:strongc}.
\end{align}
Since $\mu\teta \leq \frac{4}{a} \leq \frac{1}{25}$, $1-1.38\mu\teta \ge 0$. We have
\begin{align}
    \left(1-\frac{1}{2}\mu \teta\right)\Delta_t-\frac{44}{25}\teta\left(\EE{f(\w{t})}-f(\wopt)\right) \leq \left(1-1.38 \mu \teta\right)\Delta_t \leq \left(1-1.38 \mu \teta\right)B_t. \label{ineq:strongcvx}
\end{align}
Combining \eqref{ineq:sigma square}, \eqref{ineq:h}, 
\eqref{ineq:s} and \eqref{ineq:strongcvx}, we obtain
\begin{align}
\begin{aligned}
    \text{RHS of } \eqref{ineq:descent lemma}& \leq \left(1-1.38\mu \teta \right)B_t + 0.25\mu\teta B_t + 0.33\mu\teta B_t +
    \frac{80\tau_t \sigma^2}{39KN}\teta^2 + \frac{3\sigma^2}{KN}\teta^2\\
    &\quad + \left[18 \dmaxt + \frac{73(K-1)^2}{K^2}\right]LD\teta^3 + \left[\frac{9\dmaxt}{K}+\frac{18.1(K-1)^2}{K^3}\right]L\sigma^2\teta^3.
\end{aligned}\label{ineq:forinduc}
\end{align}
Therefore, relation (b) is verified.
\subsubsection{Verifying relation (a)}\label{sec:verify(a)}
To verify relation (a), we only have to show
\begin{align}
\begin{aligned}
    B_{t+1} + 0.8\mu\teta B_t &\geq B_t + \frac{80\tau_t \sigma^2}{39KN}\teta^2 
       + \frac{3\sigma^2}{KN}\teta^2+\left[18 \dmaxt + \frac{73(K-1)^2}{K^2}\right]LD\teta^3 \\
      &\quad+ \left[\frac{9\dmaxt}{K}+\frac{18.1(K-1)^2}{K^3}\right]L\sigma^2\teta^3.\label{ineq:induction}
\end{aligned}
\end{align}
Note that $B_{t+1}$ can be split as
\begin{align*}
   B_{t+1} = \frac{\tau_t E}{(t+a+1)^2} + \frac{s_{t}E}{(t+a+1)^2}+\frac{G}{t+a+1} + \frac{F}{(t+a+1)^2},
\end{align*}
and that
\begin{align*}
    \frac{1}{t+a} - \frac{1}{t+a+1} &= \frac{1}{(t+a)(t+a+1)}\le\frac{1}{(t+a)^2},\\
     \frac{1}{(t+a)^2} - \frac{1}{(t+a+1)^2} &= \frac{2t+2a+1}{(t+a)^2(t+a+1)^2}\le\frac{2}{(t+a)^3}.
\end{align*}
Therefore, to prove \eqref{ineq:induction}, we only have to show
\begin{align}
    \frac{\tau_t E}{(t+a+1)^2} \ge \frac{80\tau_t\sigma^2}{39KN}\teta^2=\frac{1280\tau_t\sigma^2}{39\mu^2 KN(t+a)^2},\label{ineq:E}
\end{align}
and
\begin{align}
\begin{aligned}
0.8\mu\teta B_t  &\geq \frac{2Es_{t}}{(t+a)^3} + \frac{2F}{(t+a)^3}+\frac{G}{(t+a)^2}+\frac{48\sigma^2}{KN\mu^2(t+a)^2}\\
& \quad +\left[18\dmaxt + \frac{73(K-1)^2}{K^2}\right]LD\teta^3 \\
&\quad+ \left[\frac{9\dmaxt}{K}+\frac{18.1(K-1)^2}{K^3}\right]L\sigma^2\teta^3.
\end{aligned}\label{ineq:F}
\end{align}
\eqref{ineq:E} holds since
\begin{align*}
    E&=\frac{35\sigma^2}{\mu^2 N K} \geq \frac{1280(40+1)^2}{39(40^2)\mu^2N K}\geq\frac{1280(t+a+1)^2}{39\mu^2N K(t+a)^2}.
\end{align*}
To show that \eqref{ineq:F} holds, we plug in the value of $B_t$ and $\teta$ and make minor adjustments.
\begin{align*}
    &\quad \frac{1.2Es_{t-1}}{(t+a)^3} + \frac{1.2F}{(t+a)^3} + \frac{2.2G}{(t+a)^2}\\
    &\geq  \frac{48\sigma^2}{KN\mu^2(t+a)^2}
   + \frac{\dmaxt L(1152D+576\sigma^2/K)}{\mu^3 (t+a)^3} + \frac{(K-1)^2}{K^2}\cdot\frac{L(4672D+1158.4\sigma^2/K)}{\mu^3(t+a)^3}.
\end{align*}
Recall
\begin{align*}
G=\frac{22\sigma^2}{\mu^2 N K}, F \ge \frac{\dmaxt L(2500D+5000L\sigma^2/\mu)}{\mu^3} + \frac{(K-1)^2}{K^2}\cdot\frac{5000L(D+\sigma^2/K)}{\mu^3}.
\end{align*}
Thus \eqref{ineq:F} holds. Now we have completed the induction step and obtain \Cref{lemma:delta}.

\subsection{Proof of \Cref{theorem:maintheorem_general}}\label{sec:proof_general}
In this subsection, we provide a bound for $\EE{f(\overline{w}_T)}-f(\wopt)$ based on the bound for $\Delta_T$.
Here we restate the descent lemma.
\begin{align}
\begin{aligned}
    \Delta_{t+1} 
   & \leq \left(1-\frac{1}{2}\mu \teta\right)\Delta_t-\frac{44}{25}\teta\left(\EE{f(\w{t})}-f(\wopt)\right)\\
   &\quad+\underbrace{ \frac{2\teta\sigma^2}{KN^2}\nsum\sum_{j=\dl{t}}^{t-1}\teta[j] + \frac{3\teta^2\sigma^2}{KN}
    +\frac{53}{50}\calH+\calS\calQ}_{\calQ[t]}.
\end{aligned}\label{ineq:descent lemma last}
\end{align}
Interestingly, the proof in \Cref{sec:w} generates a bound for $\calQ[t]$. Combining \eqref{ineq:forinduc} and \eqref{ineq:induction}, we find
\begin{align*}
    B_{t+1} &\geq (1-1.38\mu\teta)B_t + 0.25\mu\teta B_t + 0.33\mu\teta B_t + \frac{80\tau_t \sigma^2}{39KN}\teta^2\\
    &\quad + \left[18 \dmaxt + \frac{73(K-1)^2}{K^2}\right]LD\teta^3 + \left[\frac{9\dmaxt}{K}+\frac{18.1(K-1)^2}{K^3}\right]L\sigma^2\teta^3\\
   & \geq (1-1.38\mu\teta)B_t + \calQ[t].
\end{align*}
Hence
\begin{align*}
    \calQ[t] \leq B_{t+1} - B_t + 1.38\mu \teta B_t\leq \frac{E\tau_t}{(t+a+1)^2} + 1.38\mu \teta B_t.
\end{align*}
Rearrange \eqref{ineq:descent lemma last}, we have
\begin{align*}
    \frac{44}{25}\teta\left(\EE{f(\w{t})}-f(\wopt)\right) \leq  \left(1-\frac{1}{2}\mu \teta\right)\Delta_t - \Delta_{t+1} + \mathcal{Q}_t.
\end{align*}
Apply \eqref{ineq:strongc} and subtract $0.25\mu \teta\Delta_t$ on the RHS and $0.5\teta(\EE{f(\w{t})}-f(\wopt))$ on the LHS. Then we have
\begin{align*}
    \frac{63}{50}\teta(\EE{f(\w{t})}- f(\wopt))\leq \left(1-\frac{3}{4}\mu\teta\right)\Delta_t - \Delta_{t+1}+\calQ[t].
\end{align*}
Dividing $\teta$ on both sides and multiplying both sides by $(t+a-1)(t+a-2)$, we have
\begin{align*}
    &\quad (t+a-1)(t+a-2)(\EE{f(\w{t})}- f(\wopt))\\
    &\leq \frac{\mu(t+a-3)(t+a-2)(t+a-1)}{4}\Delta_t - \frac{\mu(t+a-2)(t+a-1)(t+a)}{4}\Delta_{t+1}\\
    &\quad +\frac{\mu(t+a-2)(t+a-1)(t+a)}{4}\calQ[t]\\
    & \leq\frac{\mu(t+a-3)(t+a-2)(t+a-1))}{4}\Delta_t - \frac{\mu(t+a-2)(t+a-1)(t+a)}{4}\Delta_{t+1}\\
    &\quad +  E'' \tau_t(t+a) + E's_{t} + F' + (t+a)G'.
\end{align*}
where $E''=\frac{\mu}{4}E,E'=0.345\mu E, F'=0.345\mu F, G'=0.345\mu G$. Telescoping from $t=1$ to $T-1$, we have
\begin{align}
\begin{aligned}
    &\quad \sum_{t=1}^{T-1} (t+a-1)(t+a-2)\left\{\EE{f(\w{t})}-f(\wopt)\right\}+\frac{\mu(T+a-3)(T+a-2)(T+a-1)}{4}\Delta_{T}\\
    & \leq \frac{\mu a^3}{4}\Delta_1 + E''\sum_{t=1}^{T-1}\tau_t (t+a)+E'\sum_{t=1}^{T-1} s_{t} + F' T + G'\sum_{t=1}^{T-1} (t+a).
\end{aligned}\label{ineq:telescope}
\end{align}
By $L$-smoothness, $f(\w{t})-f(\wopt) \leq \tfrac{L}{2}\Delta_t$. Since $a\geq 100(\tfrac{L}{\mu})^{1.5}$, $\tfrac{\mu(a-2)}{4}\geq \tfrac{L}{2}$. Therefore,
\begin{align*}
    \frac{\mu(T+a-3)(T+a-2)(T+a-1)}{4}\Delta_{T} & \geq \frac{\mu(a-2)(T+a-2)(T+a-1)}{4}\Delta_{T}\\
    & \geq (T+a-2)(T+a-1)\left\{\EE{f(\w{T})}-f(\wopt)\right\}.
\end{align*}
Then \eqref{ineq:telescope} can be further simplified as
\begin{align}
\begin{aligned}
     &\quad \sum_{t=1}^{T} (t+a-1)(t+a-2)\left\{\EE{f(\w{t})}-f(\wopt)\right\}\\
     & \leq \frac{\mu a^3}{4}\Delta_1 + E''\sum_{t=1}^{T-1}\tau_t (t+a)+E'\sum_{t=1}^{T-1} s_{t} + F' T + G'\sum_{t=1}^{T-1} (t+a).
\end{aligned}\label{ineq:telescope2}
\end{align}
Since $\sum_{t=1}^{T-1} s_{t}=\sum_{t=1}^{T-1}\sum_{t'=1}^{t-1}\tau_{t'}=\sum_{t=1}^{T-1}(T-1-t)\tau_t $, we have
\begin{align*}
    E''\sum_{t=1}^{T-1} \tau_t(t+a) + E'\sum_{t=1}^{T-1}s_{t}& \leq E'(T-1+a)\sum_{t=1}^{T-1}\tau_t\\
    & \leq E'(T+a)s_{T}.
\end{align*}
Therefore,
\begin{align*}
    \text{RHS of }\eqref{ineq:telescope2} \leq \frac{\mu a^3}{4}\Delta_1 + E'(T+a)s_{T}+F'T+G'T(T+a).
\end{align*}

Define $W_T = \sum_{t=1}^T (t+a-1)(t+a-2)=\frac{1}{3}T^3+(a-1)T^2+(a^2-2a+\frac{2}{3})T$. Note that $W_T \ge \frac{1}{3}T^2(T+a)$. Dividing $W_T$ on both sides, we have
\begin{align*}
    & \quad\left\{\frac{1}{W_T}\sum_{t=1}^T (t+a-1)(t+a-2)\EE{f(\w{t})}\right\}-f(\wopt) \\
    & \leq \frac{3\mu a^3}{4(T+a)^3}\Delta_1 + \frac{3E's_{T}}{T^2}+ \frac{3G'}{T}+ \frac{3F'}{T^2}.
\end{align*}
Considering $\frac{\mu a^3}{(T+a)^3}\leq \frac{\mu a^2}{T^2}$ and convexity of $f(w)$, we have
\begin{align*}
    \EE{f(\overline{w}_T)} - f(\wopt) = \mathcal{O}\left(\frac{G'+E'\taubart}{T} + \frac{F'}{T^2}\right).
\end{align*}
where $\overline{w}_T = \frac{1}{W_T}\sum_{t=1}^T (t+a-1)(t+a-2) w_t$. 
Plugging in $G', E' \text{ and } F'$, we obtain \Cref{theorem:maintheorem_general}. Since $\dmaxt \leq \tau_{\max,T}^2$, \Cref{theorem:maintheorem} holds.
\section{Proof of convergence for smooth and non-convex objective functions}\label{appendix:noncvx}
%formally state the convergence bound for smooth and non-convex objective functions. 
In this section, we first state a more general version of \Cref{theorem:noncvx} and then provide a proof. The proof of \Cref{theorem:noncvx} is provided as a corollary (See \Cref{corr:noncvx}). Regarding the number of inactive rounds,  we have the following relaxed assumption.
\begin{assumption}\label{assumption:delay nonconvex}
There exists a constant $t_0$ such that $\forall t\geq 1$ and $i \in [N]$, $\tau(t,i) \leq \frac{1}{4}\sqrt{\frac{L}{(L^2+\rho \delta)KN} }\max\{\sqrt{t}, \sqrt{t_0}\}$.
\end{assumption}
Note that different from \Cref{assumption:constant delay}, \Cref{assumption:delay nonconvex} allows $\tau(t,i)$ to grow as $ \mathcal{O}(\sqrt{t})$. Let $\taubart$ and $\tau_{\max,T}$ be be defined the same as in \Cref{sec:cvx}. Further define $$\taubarmaxt = \frac{1}{N}\nsum \max_{1\leq t\leq T-1}\{\tau(t,i)\},$$
which  takes the maximum number of inactive rounds over rounds $1, \cdots, T-1$ for each device and takes the average across devices.  And define 
$$\bar{d}_T = \frac{1}{T-1}\sum_{t=1}^{T-1}d_t,$$%Define $l_{\max, T} = 2\tau_{\max, T}.$
which is the average of squared number of inactive rounds across all devices and rounds.
The following theorem summarizes the performance of $\algoname[]$ on smooth and non-convex problems.
\begin{theorem}\label{theorem:noncvx_general}
Let Assumptions \ref{assumption:smooth}, \ref{assumption:noise}  and \ref{assumption: hessian} to \ref{assumption: bounded dissimilarity for noncvx} hold.   Further assume that the device availability sequence $\tau(t,i)$ satisfies \Cref{assumption:delay nonconvex} and $\tau(t,i)=0$ for all $i\in [N]$. By setting the learning rate $\eta = c_0\sqrt{\frac{N}{KTL(1+\taubart)}}$, where constant $c_0$ satisfies $0 < c_0 \leq 1$ and $T\geq \max \{\frac{64\alpha ^2K N L^3}{L^2 + \rho \delta}, 16L N K, t_0\}$, after communication rounds $1, \cdots,T-1$, \algoname satisfies:
\begin{align*}
  \min_{1\leq t \leq T} \EE[\xi]{\norm{\nabla f(\w{t})}^2 } = \mathcal{O}\left(\sqrt{\frac{(1+\taubart)L}{TKN}}(f(\w{1}) - f^*+\sigma^2)+\frac{A_6}{T}\right),
  \end{align*}
where
\begin{align*}
    A_6 = \frac{1}{(1+\taubart)}\Big[\sigma^2 \taubarmaxt^2 N K L\left(1+\frac{\alpha \taubarmaxt L^2}{(\rho \delta+L^2)(1+\taubart)}\right) + \frac{(L^2 + \rho \delta)\sigma^2}{L}\bar{d}_T\\
   +(K-1)N L(\beta  +\sigma^2/K)\Big]+ LKN \tau_{\max, T} \sigma \sqrt{\beta + \frac{\sigma^2}{KN}}.
\end{align*}
\end{theorem}
\subsection{Additional notation}
Define $r_T = \sum_{t=1}^{T-1}d_t$, which is the sum of average squared number of inactive rounds over the  first $T-1$ communication rounds. Define $g_t = \frac{1}{KN}\knsum \nabla f_i(\w{\dl{t}, k}^i)$, which is the scaled accumulated true gradients at round 
$t$. Also define $l_{\max, T} = 2\tau_{\max, T}$ and $\teta[] = K\eta$ for convenience.
%Define $$\tau_{\max, T} = \max\{\tau(t,i)\mid i \in [N], 1\leq t \leq T-1\}, .$$ 
\subsection{Preliminary lemmas}
Before starting the proof, we introduce some preliminary lemmas in this subsection.
\begin{lemma}[Property of Hessian Lipschitz functions]\label{lemma: hessian lip}
For a $\rho$-Hessian Lipschitz function $f$ and for all $ w, v$ and $z$, the following holds. 
\begin{align*}
     \quad\dotp{\nabla f(w) - \nabla f(v)}{z}   \leq \dotp{\hessian(v)(w-v)}{z} + \frac{\rho}{2} \norm{z}\norm{w-v}^2.
\end{align*}
\end{lemma}
\begin{proof}
\begin{align*}
     &  \quad\dotp{\nabla f(w) - \nabla f(v)}{z} \\
    & = \dotp{\left[\int_0^1 \hessian(v+\theta(w-v))d\theta\right] (w-v)}{z}\\
    & = \dotp{\hessian(v)(w-v)}{z} + \dotp{\left\{\int_0^1 \left[\hessian(v+\theta(w-v)) - \hessian(v)\right]d\theta\right\}(w-v)}{z}\\
    & \leq \dotp{\hessian(v)(w-v)}{z} + \norm{z}\norm{w-v}\norm{\int_0^1 \left[\hessian(v+\theta(w-v)) - \hessian(v)\right]d\theta}\\
    &  \leq \dotp{\hessian(v)(w-v)}{z} + \norm{z}\norm{w-v}\int_0^1 \norm{\hessian(v+\theta(w-v)) - \hessian(v)}d\theta\\
    & \leq \dotp{\hessian(v)(w-v)}{z} + \rho \norm{z}\norm{w-v}^2\int_0^1 \theta d\theta\\
    & \leq \dotp{\hessian(v)(w-v)}{z} + \frac{\rho}{2} \norm{z}\norm{w-v}^2.
\end{align*}
\end{proof}
\begin{lemma}[Bounded drift for non-convex objective functions]\label{lemma: bounded drift noncvx}
For all $K\geq 1, 0\leq k\leq K-1$, $\teta[] \leq\frac{1}{10L}$, we have bounded drift
\begin{align*}
    \EE{\norm{\w{t,k}^i - \w{t}}^2 } &\leq  \frac{4\alpha \teta[]^2 (K-1)}{K}\EE{\norm{\nabla f(\w{t})}^2}  + \frac{4(K-1)\teta[]^2\beta_i}{K}+\frac{2(K-1)\teta[]^2\sigma^2}{K^2}.
\end{align*}
\end{lemma}
\begin{proof}
Simply combining \eqref{ineq: bounded drift for sc} in \Cref{lemma: bounded drift} and \Cref{assumption: bounded dissimilarity for noncvx}, we have
\begin{align*}
    \EE{\norm{\w{t,k}^i - \w{t}}^2} &\leq 2(K-1)\left(\frac{2\teta[]^2}{K}\EE{\norm{\nabla f_i(\w{t})}^2} + \frac{\teta[]^2\sigma^2}{K^2}\right)\\
    & \leq \frac{4\alpha \teta[]^2 (K-1)}{K}\EE{\norm{\nabla f(\w{t})}^2}  + \frac{4(K-1)\teta[]^2\beta_i}{K}+\frac{2(K-1)\teta[]^2\sigma^2}{K^2}.
\end{align*}
\end{proof}
\begin{lemma}[Bounding the difference of parameters at different rounds]\label{lemma:lmax}
For all $t \ge t', t-t'\leq l$, where $l$ is a constant and $\teta[] \leq \frac{1}{\sqrt{12}L}$, the following inequality holds.
\begin{align*}
    \EE{\norm{\w{t}-\w{t'}}^2}& \leq \frac{4\alpha l\teta[]^2}{N}\sum_{j=\max\{t- l, 1\}}^{t-1}\nsum \EE{\norm{\nabla f(\w{\dl{j}})}^2 }\\
    &\quad +4 l^2 \beta \teta[]^2+\frac{4\teta[]^2 l^2}{KN}\sigma^2.
\end{align*}
\end{lemma}
\begin{proof}
Since $\sg f_i(\w{\dl{j}, k}^i) = \nabla f_i(\w{\dl{j}, k}^i)-\nabla f_i(\w{\dl{j}}) + \nabla f_i(\w{\dl{j}})+e_{\dl{j}, k}^i$,
\begin{align*}
    &\quad \EE{\norm{\w{t}-\w{t'}}^2} \\
    & = \teta[] ^2\EE{\norm{\sum_{j=t'}^{t-1}\frac{1}{KN} \knsum \sg f_i(\w{\dl{j}, k}^i)}^2}\\
    & \leq 3\teta[]^2 \EE{\norm{\frac{1}{KN}\sum_{j=t'}^{t-1} \knsum\left(\nabla f_i(\w{\dl{j}, k}^i)-\nabla f_i(\w{\dl{j}})\right)}^2 }\\
    & \quad + 3\teta[]^2\EE{\norm{\frac{1}{N}\sum_{j=t'}^{t-1}\nsum \nabla f_i(\w{\dl{j}})}^2  } + \frac{3\teta[]^2}{K^2N^2}\EE{\norm{\sum_{j=t'}^{t-1}\Big(\knsum e^i_{\pdl{j},k}\Big )}^2} \\
    & \leq \frac{3(t-t')L^2\teta[]^2}{KN}\sum_{j=t-t'}^{t-1}\knsum \EE{\norm{\w{\dl{j}, k}^i-\w{\dl{j}}}^2}\\
    & \quad+\frac{3\teta[]^2(t-t')}{N}\sum_{j=t-t'}^{t-1}\nsum \EE{\norm{\nabla f_i(\w{\dl{j}})}^2 }+\frac{3\teta[]^2(t-t')^2}{KN}\sigma^2 \\
   &\leq \frac{3\alpha l\teta[]^2}{N}\sum_{j=\max\{t- l, 1\}}^{t-1}\nsum \EE{\norm{\nabla f(\w{\dl{j}})}^2 }+3 l^2 \beta \teta[]^2+\frac{3\teta[]^2 l^2}{KN}\sigma^2\\
   & \quad + \frac{12\alpha l L ^2\teta[]^4(K-1)^2}{N K^2}\sum_{j=\max\{t- l, 1\}}^{t-1}\nsum \EE{\norm{\nabla f(\w{\dl{j}})}^2 }+\frac{12(K-1)^2 l^2L^2\beta\teta[]^4}{K^2}\\
   & \quad + \frac{6(K-1)^2 l^2L^2 \teta[]^4\sigma^2}{K^3}\\
   & \leq \frac{4\alpha l\teta[]^2}{N}\sum_{j=t- l}^{t-1}\nsum \EE{\norm{\nabla f(\w{\dl{j}})}^2 } +4 l^2 \beta \teta[]^2+\frac{4\teta[]^2 l^2}{KN}\sigma^2.
\end{align*}
The first inequality above uses Jensen's inequality. The second one utilizes $L$-smoothness and Jensen's inequality. The third one uses \Cref{lemma: bounded drift noncvx} and the last one holds since $\teta[] \leq \frac{1}{\sqrt{12}L}$.
\end{proof}

%Define $l_{\max, T} = \max\left\{ \tau(t,i) + \tau(t-\tau(t,i), j) \mid i, j \in [N], 1 \leq t \leq T-1\right \}$, which is the longest inactive duration in two adjacent rounds up to round $T-1$. 
\subsection{The descent lemma for smooth and non-convex problems}
In this subsection, we state the descent lemma and provide a proof.
\begin{lemma}[Descent lemma for non-convex problems]
Assume that  Assumptions~\ref{assumption:smooth}, \ref{assumption:noise} and \ref{assumption: hessian} to \ref{assumption: bounded dissimilarity for noncvx} hold.  Further assume that $\tau(1,i)=0$ for all $i \in [N]$. For any learning rate satisfying $\teta[] \leq \frac{1}{\sqrt{12}L}$, i.e., $\eta \leq \frac{1}{\sqrt{12}KL}$, the following holds for all $1\leq t \leq T$.
\begin{align}
\begin{aligned}
     & \quad \EE{f(\w{t+1})} - \EE{f(\w{t})}\\
     & \leq  -\frac{\teta[]}{2}\EE{\norm{\nabla f(\w{t})}^2}  + \frac{L(1+\tau_t)\sigma^2 }{KN}\teta[]^2 + (H_1d_t + H_2 \tau_t + H_3)\teta[]^3\\
     & \quad  +2\tau_t\sigma L^2 \teta[]^3 \sqrt{\frac{\alpha  l }{N} \sum_{j=\max\{t-\lmaxt, 1\}}^{t-1}\psum \EE{\norm{\nabla f(\w{\pdl{j}})}^2}}\\
     &\quad +  \frac{(4L^2 + \rho \delta)}{N}\teta[]^3\nsum \tau(t,i)\left(\sum_{j=\dl{t}}^{t-1} \EE{\norm{g_j}^2}\right) - \frac{\teta[]}{2}\left(1-2L\teta[]\right) \EE{\norm{g_t}^2}\\
    & \quad + \sigma L\teta[]^2\tau_t \sqrt{\EE{\norm{\nabla f(\w{t})}^2 }}+ \frac{8\alpha L^2 (K-1)\teta[]^3 }{KN}\nsum \EE{\norm{\nabla f(\w{\dl{t}})}^2},
\end{aligned}\label{ineq: decent lemma noncvx}
\end{align}
where $H_1 = \frac{(4L^2 + \rho \delta)\sigma^2 }{KN}$,  $H_2 = 2L^2  \lmaxt \sigma \sqrt{\beta + \frac{\sigma^2}{KN}} $ and $H_3 = \frac{4(K-1)L^2(2\beta + \sigma^2/K)}{K}$ . 
\end{lemma}
\paragraph{Proof of the descent lemma.} According to the update rule in \eqref{eq:update rule} and $L$-smoothness, 
\begin{align*}
    &\quad f(\w{t+1}) - f(\w{t}) \\
    & \leq \dotp{\nabla f(\w{t})}{\w{t+1}-\w{t}} + \frac{L}{2}\norm{\w{t+1}-\w{t}}^2\\
    & = -\teta[] \dotp{\nabla f(\w{t})}{\frac{1}{KN}\knsum \sg f_i(\w{\dl{t}, k}^i)}+\frac{L\teta[]^2}{2}\norm{\frac{1}{KN}\knsum \sg f_i(\w{\dl{t}, k}^i )}^2\\
    & = \underbrace{-\teta[] \dotp{\nabla f(\w{t})}{\frac{1}{KN}\knsum e_{\dl{t}, k}^i}}_{\calT[1]} \underbrace{-\teta[] \dotp{\nabla f(\w{t})}{\frac{1}{KN}\knsum \nabla f_i(\w{\dl{t}, k}^i)}}_{\calT[2]}\\
    & \quad +\underbrace{\frac{L\teta[]^2}{2}\norm{\frac{1}{KN}\knsum \sg f_i(\w{\dl{t}, k}^i)}^2}_{\calT[3]}.
\end{align*}
\subsubsection{Bounding the first term}\label{sec: noncvx T1}
Due to reuse of noisy updates, $e^i_{\dl{t},k}$ is correlated with $\w{t}$ and $\EE{\calT[1]}$ is not necessarily zero. Unrolling one summand of $\calT[1]$,
\begin{align*}
    -\teta[]\dotp{\nabla f(\w{t})}{e_{\dl{t}, k}^i} & = \underbrace{-\teta[]\dotp{\nabla f(\w{t}) - \nabla f(\w{\dl{t}})}
    {e_{\dl{t}, k}^i}}_{\calU[1]} \\
    & \quad \underbrace{-\teta[] \dotp{\nabla f(\w{\dl{t}})}{e_{ \dl{t}, k}^i}}_{\calU[2]}.
\end{align*}
Since $\w{\dl{t}}$ and $e^i_{\dl{t}, k}$ are independent, we have
$\EE{\calU[2] }=0$. Plugging $z = -e^i_{\dl{t},k}$ into \Cref{lemma: hessian lip},
\begin{align*}
    & \quad \EE{\calU[1]} \\
    & \leq  \EE{-\teta[] \dotp{\hessian(\w{\dl{t}})(\w{t}-\w{\dl{t}})}{e_{\dl{t}, k}^i}} + \frac{1}{2}\rho \delta\teta[]\EE{\norm{\w{t}-\w{\dl{t}}}^2 }\\
    &=  \underbrace{\teta[]^2   \EE{\dotp{\hessian(\w{\dl{t}})\frac{1}{KN}\sum_{j=\dl{t}}^{t-1} \pknsum \nabla f_{i'}(\w{\pdl{j}}^{i'})}{e_{\dl{t}, k}^i}}}_{\calV[1]}\\
    & \quad +\underbrace{\teta[]^2  \frac{1}{KN}\sum_{j=\dl{t}}^{t-1} \pknsum\EE{\dotp{\hessian(\w{\dl{t}})e_{\dl{j}, k'}^{i'}}{e_{ \dl{t}, k}^i} }}_{\calV[2]}\\
    & \quad + \frac{1}{2}\rho \delta\teta[]\underbrace{\EE{\norm{\w{t}-\w{\dl{t}}}^2 }}_{\calV[3]}.
\end{align*}
Using the identity $e_{\dl{t}, k}^i = e_{\dl{t-1}, k}^i=\cdots=e_{ \dl{\dl{t}}, k}^i$ and independence of $e_{j,k}^i$ and $e_{j', k'}^{i'}$ for all $i\neq i'$ or $k\neq k'$, we can bound $\calV[2]$.
\begin{align*}
    \calV[2] = \frac{\teta[]}{KN}\tau(t,i)\EE{\dotp{\hessian(\w{\dl{t}})e_{ \dl{t}, k}^i}{e_{\dl{t}, k}^i} }\leq \tau(t,i)\frac{\teta[]^2 L}{KN}\sigma^2,
\end{align*}
where the second inequality uses $L$-smoothness of $f$. 
Note that $\nabla f_{i'}(\w{\pdl{j}})$ can be split as $\nabla f_{i'}(\w{\pdl{j}})- \nabla f_{i'}(\w{t})+\nabla f_{i'}(\w{t})$. Further using Cauchy-Schwartz inequality $\EE{\dotp{X}{Y}}\leq \sqrt{\EE{\norm{X}^2}\EE{\norm{Y}^2}}$ and $L$-smoothness, we can bound $\calV[1]$ in the following way.
\begin{align*}
\calV[1]& =
    \teta[]^2 \frac{1}{KN}\sum_{j=\dl{t}}^{t-1}\pknsum\EE{\dotp{\hessian(\w{\dl{t}})\left(\nabla f_{i'}(\w{\pdl{j}}) - \nabla f_{i'}(\w{t})\right)}{e_{\dl{t}, k}^i}}\\
    & \quad + \teta[]^2\tau(t,i)\EE{\dotp{\hessian(\w{\dl{t}})\nabla f(\w{t})}{e_{i, \dl{t}}} }\\
    &\leq \underbrace{ \frac{\sigma L\teta[]^2}{N}\sum_{j=\dl{t}}^{t-1}\psum \sqrt{\EE{\norm{\nabla f_{i'}(\w{t})-\nabla f_{i'}(\w{\pdl{j}})}^2 }}}_{\calV[4]}\\
    & \quad + \sigma L\teta[]^2\tau(t,i) \sqrt{\EE{\norm{\nabla f(\w{t})}^2 }}.
\end{align*}
Note that for all $ t-\tau(t,i)\leq j \leq t-1$ and $i'\in [N]$, $t - (\pdl{j}) \leq \lmaxt$. By $L$-smoothness and \Cref{lemma:lmax}, we obtain an upper bound for $\calV[4]$.
\begin{align*}
    \EE{\calV[4]} & \leq \frac{\sigma L^2 \teta[]^2}{N}\sum_{j=\dl{t}}^{t-1}\psum \sqrt{\EE{\norm{\w{t}- \w{\pdl{j}}}^2 }}\\
    & \leq 2\tau(t,i)\sigma L^2 \teta[]^3 \sqrt{\frac{\alpha \lmaxt }{N} \sum_{j=\max\{t-\lmaxt, 1\}}^{t-1}\psum \EE{\norm{\nabla f(\w{\pdl{j}})}^2}}\\
    & \quad + 2\sigma L^2 \lmaxt \tau(t,i) \teta[]^3\sqrt{\beta + \frac{\sigma^2}{KN}},
\end{align*}
where the last in equality uses $\sqrt{x+y}\leq \sqrt{x}+\sqrt{y}, \forall x,y \geq 0$. Now we can obtain an upper bound for $\calV[1]$.
 \begin{align*}
     \calV[1] &\leq 2\tau(t,i)\sigma L^2 \teta[]^3 \sqrt{\frac{\alpha \lmaxt }{N} \sum_{j=\max\{t-\lmaxt, 1\}}^{t-1}\psum \EE{\norm{\nabla f(\w{\pdl{j}})}^2}}\\
    & \quad + \sigma L\teta[]^2\tau(t,i) \sqrt{\EE{\norm{\nabla f(\w{t})}^2 \mid 
    \calF}}+2\sigma L^2 \lmaxt \tau(t,i) \teta[]^3\sqrt{\beta + \frac{\sigma^2}{KN}}.
 \end{align*}
 We proceed to bound $\calV[3]$ by Jensen's inequality.
 \begin{align*}
     \calV[3] & =\EE{ \norm{\frac{\teta[]}{KN} \sum_{j=\dl{t}}^{t-1}\pknsum \sg f_{i'}(\w{\pdl{j}, k'}^{i'})}^2}\\
     & = \EE{\norm{\teta[]\sum_{j=\dl{t}}^{t-1} g_{j} + \frac{\teta[]}{KN}\sum_{j=\dl{t}}^{t-1}\pknsum  e^{i'}_{\pdl{j}, k'}}^2}\\
     & \leq 2\teta[]^2\EE{\norm{\sum_{j=\dl{t}}^{t-1} g_j}^2} + 2\teta[]^2\EE{\norm{\frac{1}{KN}\sum_{j=\dl{t}}^{t-1}\pknsum  e^{i'}_{\pdl{j}, k'}}^2}\\
     & \leq 2\tau(t,i)\teta[]^2 \sum_{j=\dl{t}}^{t-1} \EE{\norm{g_j}^2} + \frac{2\tau(t,i)^2\sigma^2\teta[]^2}{KN}.
 \end{align*}
Combining $\calV[1]$ to $\calV[3]$, we have
\begin{align*}
    \EE{\calU[1]  } &\leq \frac{\tau(t,i)L\sigma^2 }{KN}\teta[]^2+ \frac{\rho \delta\tau(t,i)^2\sigma^2}{KN}\teta[]^3+2\sigma L^2 \lmaxt \tau(t,i) \teta[]^3\sqrt{\beta + \frac{\sigma^2}{KN}}\\
    &\quad +2\tau(t,i)\sigma L^2 \teta[]^3 \sqrt{\frac{\alpha \lmaxt }{N} \sum_{j=\max\{t-\lmaxt, 1\}}^{t-1}\psum \EE{\norm{\nabla f(\w{\pdl{j}})}^2}}\\
    & \quad + \sigma L\teta[]^2\tau(t,i) \sqrt{\EE{\norm{\nabla f(\w{t})}^2 }}+ \rho \delta\tau(t,i)\teta[]^3 \sum_{j=\dl{t}}^{t-1} \EE{\norm{g_j}^2} .
\end{align*}
Finally we bound the expectation of $\calT[1]$ and conclude this section.
\begin{align*}
    \EE{\calT[1] }  &\leq \frac{\tau_t L\sigma^2 }{KN}\teta[]^2+ \frac{\rho \delta d_t\sigma^2}{KN}\teta[]^3+2\sigma L^2 \lmaxt \tau_t \teta[]^3\sqrt{\beta + \frac{\sigma^2}{KN}}\\
    &\quad +2\tau_t\sigma L^2 \teta[]^3 \sqrt{\frac{\alpha \lmaxt }{N} \sum_{j=\max\{t-\lmaxt, 1\}}^{t-1}\psum \EE{\norm{\nabla f(\w{\pdl{j}})}^2}}\\
    & \quad + \sigma L\teta[]^2\tau_t \sqrt{\EE{\norm{\nabla f(\w{t})}^2 }}+ \frac{\rho \delta\teta[]^3}{N}\nsum \tau(t,i)\sum_{j=\dl{t}}^{t-1} \EE{\norm{g_j}^2}.
\end{align*}

\subsection{Bounding the second term} \label{sec: noncvx T2}
Since $\dotp{x}{y} = \frac{1}{2}\norm{x}^2 + \frac{1}{2}\norm{y}^2 - \frac{1}{2}\norm{x-y}^2$,
\begin{align*}
   \calT[2]& = -\frac{\teta[]}{2}\norm{\nabla f(\w{t})}^2-\frac{\teta[]}{2}\norm{g_t}^2 +\frac{\teta[]}{2}\underbrace{\norm{\nabla f(\w{t}) - \frac{1}{KN}\knsum \nabla f_i(\w{\dl{t}, k}^i)}^2}_{\calU[3]}.
\end{align*}
Next we bound $\calU[3]$. Note that $\nabla f(\w{t}) - \frac{1}{KN}\knsum \nabla f_i(\w{\dl{t}, k}^i)$ can be split as
\begin{align*}
    & \quad \nabla f(\w{t}) - \frac{1}{KN}\knsum \nabla f_i(\w{\dl{t}, k}^i) \\
    & = \frac{1}{N} \nsum \left (\nabla f_i(\w{t}) - \nabla f_i(\w{\dl{t}})\right) + \frac{1}{KN}\knsum \left(\nabla f_i(\w{\dl{t}, k}^i )-\nabla f_i(\w{\dl{t}})\right).\\
\end{align*}
By Jensen's inequality and $L$-smoothness, 
\begin{align*}
    & \quad \EE{\calU[3]}\\ 
    & \leq \frac{2L^2}{N}\nsum \EE{\norm{\w{t}-\w{\dl{t}}}^2} + \frac{2L^2}{KN}\knsum \norm{\w{\dl{t}, k}^i - \w{\dl{t}}}^2\\
    & \leq  \frac{4L^2 \teta[]^2}{N}\nsum \tau(t,i) \sum_{j=\dl{t}}^{t-1} \EE{\norm{g_j}^2} + \frac{4 d_t L ^2\sigma^2\teta[]^2}{KN}\\
    &\quad  + \frac{8\alpha L^2 (K-1)\teta[]^2 }{KN}\nsum \EE{\norm{\nabla f(\w{\dl{t}})}^2}  + \frac{8L^2(K-1)\teta[]^2\beta}{K}+\frac{4L^2 (K-1)\teta[]^2\sigma^2}{K^2},
\end{align*}
where we apply \Cref{lemma: bounded drift noncvx} and plug in the bound for $\calV[3]$ in the second inequality. To sum up, we derive the following bound for the expectation of $\calT[2]$.
\begin{align*}
    & \quad \EE{\calT[2]}\\
    & \leq -\frac{\teta[]}{2}\EE{\norm{\nabla f(\w{t})}^2} - \frac{\teta[]}{2}\EE{\norm{g_t}^2} \\
    & \quad + \frac{4L^2 \teta[]^3}{N}\nsum \tau(t,i) \sum_{j=\dl{t}}^{t-1} \EE{\norm{g_j}^2} + \frac{4 d_t L ^2\sigma^2\teta[]^3}{KN}\\
    & \quad +  \frac{8\alpha L^2 (K-1)\teta[]^3 }{KN}\nsum \EE{\norm{\nabla f(\w{\dl{t}})}^2}  + \frac{8L^2(K-1)\teta[]^3\beta}{K}+\frac{4L^2 (K-1)\teta[]^3\sigma^2}{K^2}.
\end{align*}

\subsection{Bounding the third term}\label{sec: noncvx T3}
By Jensen's inequality, 
\begin{align*}
    \EE{\calT[3]}& = \frac{L\teta[]^2}{2}\EE{\norm{g_t + \frac{1}{KN}\knsum e^i_{\dl{t}, k}}^2} \\
    & \leq L\teta^2 \EE{\norm{g_t}^2} + L\teta[]^2 \EE{\norm{\frac{1}{KN}\nsum e^i_{\dl{t}, k}}^2}\\
    & \leq L\teta[]^2 \EE{\norm{g_t}^2} + \frac{L\sigma ^2 \teta[]^2}{KN}.
\end{align*}

Combining the results in \Cref{sec: noncvx T1},  \Cref{sec: noncvx T2} and \Cref{sec: noncvx T3}, we have
\begin{align*}
     & \quad \EE{f(\w{t+1})} - \EE{f(\w{t})}\\
     & \leq  -\frac{\teta[]}{2}\EE{\norm{\nabla f(\w{t})}^2}  + \frac{L(1+\tau_t)\sigma^2 }{KN}\teta[]^2 + (H_1d_t + H_2 \tau_t + H_3)\teta[]^3\\
     & \quad  +2\tau_t\sigma L^2 \teta[]^3 \sqrt{\frac{\alpha \lmaxt }{N} \sum_{j=\max\{t-\lmaxt, 1\}}^{t-1}\psum \EE{\norm{\nabla f(\w{\pdl{j}})}^2}}\\
     &\quad +  \frac{(4L^2 + \rho \delta)}{N}\teta[]^3\nsum \tau(t,i)\left(\sum_{j=\dl{t}}^{t-1} \EE{\norm{g_j}^2}\right) - \frac{\teta[]}{2}\left(1-2L\teta[]\right) \EE{\norm{g_t}^2}\\
    & \quad + \sigma L\teta[]^2\tau_t \sqrt{\EE{\norm{\nabla f(\w{t})}^2 }}+ \frac{8\alpha L^2 (K-1)\teta[]^3 }{KN}\nsum \EE{\norm{\nabla f(\w{\dl{t}})}^2},
\end{align*}
where $H_1 = \frac{(4L^2 + \rho \delta)\sigma^2 }{KN}$,  $H_2 = 2L^2  \lmaxt \sigma \sqrt{\beta + \frac{\sigma^2}{KN}} $ and $H_3 = \frac{4(K-1)L^2(2\beta + \sigma^2/K)}{K}$ . Now we have proved the descent lemma.
\subsection{Deriving the convergence rate}
Since $\sum_{t=1}^{T-1} \EE{\norm{\nabla f(\w{\dl{t}})}^2} \leq (1+\max_{1 \leq t \leq T-1}\{\tau(t,i)\}) \sum_{t=1}^{T-1} \EE{\norm{ \nabla f(\w{t})}^2}$, the
telescoping sum of \eqref{ineq: decent lemma noncvx} from $t=1$ to $T-1$ satisfies
\begin{align}
\begin{aligned}
    &\quad \EE{f(\w{T})}-\EE{f(\w{1})}\\
    & \leq \underbrace{-\teta[] \left(\frac{1}{2}-  \frac{8\alpha L^2 (K-1)\taubarmaxt \teta[]^2}{K} \right)\sum_{t=1}^{T-1}\EE{\norm{\nabla f(\w{t})}^2}}_{\calV[5]} + \frac{L(T+s_{T})\sigma^2}{KN}\teta[]^2 \\
    & \quad + (H_1 r_{T}+ H_2 s_{T} + H_3T)\teta[]^3 \\
    & \quad +\underbrace{2\sum_{t=1}^{T-1}\tau_t\sigma L^2 \teta[]^3 \sqrt{\frac{\alpha \lmaxt }{N} \sum_{j=\max\{t-\lmaxt, 1\}}^{t-1}\psum \EE{\norm{\nabla f(\w{\pdl{j}})}^2}}}_{\calV[6]}\\
   &\quad + \underbrace{\frac{(4L^2 + \rho \delta)}{N}\teta[]^3\sum_{t=1}^{T-1}\nsum \tau(t,i)\left(\sum_{j=\dl{t}}^{t-1} \EE{\norm{g_j}^2}\right) - \frac{\teta[]}{2}\left(1-2L\teta[]\right) \sum_{t=1}^{T-1}\EE{\norm{g_t}^2}}_{\calV[7]}\\
    & \quad + \underbrace{\sigma L\teta[]^2\sum_{t=1}^{T-1}\tau_t \sqrt{\EE{\norm{\nabla f(\w{t})}^2 }}}_{\calV[8]}.
\end{aligned}\label{ineq: convergence rate raw}
\end{align}
Next, we bound $\calV[5]$ to $\calV[8]$ respectively. When $\teta[] \leq \sqrt{\frac{1}{32\alpha \taubarmaxt[T] L^2}}$, we have
\begin{align*}
    \calV[5] \leq -\frac{\teta[]}{4} \sum_{t=1}^T\EE{\norm{\nabla f(\w{t})}^2}.
\end{align*}
By Jensen's inequality $(\sum_{t=1}^T \sqrt{a_t})^2 \leq T\sum_{t=1}^T a_t$, i.e. $\sum_{t=1}^T \sqrt{a_t} \leq \sqrt{T\sum_{t=1}^T a_t}$,
\begin{align*}
    \calV[6] & \leq 2\taubarmaxt \sigma L^2 \teta[]^3\sum_{t=1}^T \sqrt{\frac{\alpha \lmaxt }{N} \sum_{j=\max\{t-\lmaxt, 1\}}^{t-1}\psum \EE{\norm{\nabla f(\w{\pdl{j}})}^2}}\\
    & \leq 2\taubarmaxt\sigma L^2 \teta[]^3\sqrt{\frac{\alpha \lmaxt^2 T}{N}\sum_{t=1}^T \psum \EE{\norm{\nabla f(\w{\pdl{t}})}^2}}\\
    & \leq 2\taubarmaxt\lmaxt \sigma L^2 \teta[]^3\sqrt{\alpha T \taubarmaxt[T]\sum_{t=1}^T \EE{\norm{\nabla f(\w{t})}^2}}.
\end{align*}
When $
    \teta[] \leq \frac{1}{4L}
   $ and $ \teta[]  \leq \sqrt{\frac{1}{2(4 L^2+\rho \delta)\dbarmaxt[T]}}
$ , we can bound $\calV[7]$ as
\begin{align*}
   \calV[7] \leq -\frac{\teta[]}{2}\left[1-2L\teta[]- (4L^2 + \rho \delta)\dbarmaxt[T]\teta[]^2\right]\left( \sum_{t=1}^T \EE{\norm{g_t}^2}\right)\leq 0.
\end{align*}
Using $\tau_t \leq \bar{\tau}_{\max, T}$ for all $1 \leq t \leq T-1$ and Jensen's inequality, we have
\begin{align*}
    \calV[8] \leq \sigma L \teta[]^2 \taubarmaxt[T] \sqrt{T \sum_{t=1}^T \EE{\norm{\nabla f(\w{t})}^2}}. 
\end{align*}
After minor rearrangement, \eqref{ineq: convergence rate raw} can be simplified as
\begin{align}
\begin{aligned}
    &\quad \sum_{t=1}^{T-1} \EE{\norm{\nabla f(\w{t})}^2} \\
& \leq \frac{4}{\teta[]}(f(\w{1}) - f(\w{T})) +  \frac{4L(T+s_{T})\sigma^2}{KN}\teta[]  + 4(H_1 r_{T} + H_2 s_{T} + H_3T)\teta[]^2 \\
    & \quad +4\left(1 + 2\lmaxt\sqrt{\alpha\taubarmaxt[T]} L\teta[]\right) \sqrt{T} \sigma L \taubarmaxt[T]\teta[]\sqrt{\sum_{t=1}^T \EE{\norm{\nabla f(\w{t})}^2}}.
\end{aligned}\label{ineq:noncvx telescope}
\end{align}
By \Cref{lemma:smooth}, $\EE{\norm{\nabla f(\w{T})}^2} \leq 2L (\EE{f(\w{T})}- f^*)$. Multiplying both sides by $\teta[]$ and further using $\teta[]\leq \tfrac{1}{4L}$, we have
\begin{align*}
    \teta[]\EE{\norm{\nabla f(\w{T})}^2} & \leq 2L\teta[] (\EE{f(\w{T})}- f^*)\\
    & \leq \frac{1}{2}(\EE{f(\w{T})}- f^*).
\end{align*}
The adding $\EE{\norm{\nabla f(\w{T})}^2}$ to the LHS and $\frac{4}{\teta}\EE{f(\w{T})}-f(\wopt)$ to the RHS, \eqref{ineq:noncvx telescope} can be further simplified as
\begin{align*}
    &\quad \sum_{t=1}^{T} \EE{\norm{\nabla f(\w{t})}^2} \\
& \leq \frac{4}{\teta[]}(f(\w{1}) - f^*) +  \frac{4L(T+s_{T})\sigma^2}{KN}\teta[]  + 4(H_1 r_{T} + H_2 s_{T} + H_3T)\teta[]^2 \\
    & \quad +4\left(1 + 2\lmaxt\sqrt{\alpha\taubarmaxt[T]} L\teta[]\right) \sqrt{T} \sigma L \taubarmaxt[T]\teta[]\sqrt{\sum_{t=1}^T \EE{\norm{\nabla f(\w{t})}^2}}.
\end{align*}
Define $\Omega_T = \sum_{t=1}^T \EE{\norm{\nabla f(\w{t})}^2}$, $H_4 =\frac{4}{\teta[]}(f(\w{1}) - f^*) + \frac{4L(T+s_{T})\sigma^2}{KN}\teta[]  + 4(H_1 r_{T} + H_2 s_{T} + H_3 T)\teta[]^2$, $H_5 = 4\left(1 + 2\lmaxt\sqrt{\alpha\taubarmaxt[T]} L\teta[]\right) \sqrt{T} \sigma L \taubarmaxt[T]\teta[]$. Now we solve the following inequality.
\begin{align*}
    \Omega_T \leq H_5 \sqrt{\Omega_T} + H_4 \Rightarrow \sqrt{\Omega_T} \leq \frac{1}{2}(H_5 + \sqrt{H_5^2 + 4 H_4}) \Rightarrow \Omega_T \leq H_5^2 + 2H_4.
\end{align*}
Therefore,
\begin{align}
\begin{aligned}
    \frac{1}{T}\sum_{t=1}^T \EE{\norm{\nabla f(\w{t})}^2 }&  \leq \frac{8}{T\teta[]}(f(\w{1}) - f^*) + \frac{8L(1+\taubart[T])\sigma^2}{KN}\teta[] \\
    &\quad + 4(H_1 \bar{d}_{T} + H_2 \taubart[T] + H_3)\teta[]^2\\
    &\quad+ 32\sigma^2 \taubarmaxt[T]^2 L^2\teta[]^2 + 128\alpha \sigma^2 \taubarmaxt[T] ^3 \lmaxt^2 L^4 \teta[]^4.
    \end{aligned}\label{ineq:ergodic mean}
\end{align}
Let $\teta[] = c_0\sqrt{\frac{KN}{TL(1+\taubart)}}$, where $c_0$ is a constant and $0 < c_0\leq 1$. We will show that for $T\geq \max \{\frac{64\alpha ^2K N L^3}{L^2 + \rho \delta}, 16L N K, t_0\}$, the following holds.
\begin{align}
\teta[] & \leq \sqrt{\frac{1}{2(4 L^2+\rho \delta)\dbarmaxt[T]}},\label{ineq:eta1}\\
\teta[] &\leq \frac{1}{4L},\label{ineq:eta2}\\
\teta[] &\leq\sqrt{ \frac{1}{32\alpha \taubarmaxt[T] L^2}}.\label{ineq:eta3}
\end{align}
By \Cref{assumption:delay nonconvex}, when  $T \geq t_0$, $\tau(t,i) \leq \frac{1}{4}\sqrt{\frac{LT}{NK(\rho \delta + L^2)}}$, $\forall t \leq T$. Thus \eqref{ineq:eta1} holds.
Since $T\geq 16LNK$, \eqref{ineq:eta2} holds. To verify \eqref{ineq:eta3}, we only have to show
\begin{align*}
    \teta[]^2 \leq \frac{1}{32\alpha \taubarmaxt[T] L^2} \Leftrightarrow \taubarmaxt[T] \leq \frac{T(1+\taubart[T])}{32\alpha c_0^2LKN}.
\end{align*}
Still by \Cref{assumption:delay nonconvex}, we only have to show
\begin{align*}
   \frac{1}{4}\sqrt{\frac{LT}{NK(\rho \delta + L^2)}} \leq \frac{T(1+\taubart[T])}{32\alpha LKN}, 
\end{align*}
which holds for $T \geq \frac{64\alpha ^2KNL^3}{L^2 + \rho \delta}$. Now we only have to plug the value of $\teta[]$ into \eqref{ineq:ergodic mean} and make minor adjustments. Still by \Cref{assumption:delay nonconvex}, we have 
\begin{align*}
    (\lmaxt \teta[])^2 = \mathcal{O}\left(\frac{1}{(\rho\delta + L^2)(1+\taubart)}\right).
\end{align*}
Since $ \min_{1\leq t \leq T}\EE{\norm{\nabla f(\w{t})}^2 } \leq \frac{1}{T}\sum_{t=1}^T \EE{\norm{\nabla f(\w{t})}^2 }$, we have
\begin{align*}
   \min_{1\leq t \leq T}\EE{\norm{\nabla f(\w{t})}^2 }  = \mathcal{O}\left(\sqrt{\frac{(1+\taubart)L}{TKN}}(f(\w{1}) - f^*+\sigma^2)+\frac{A_6}{T}\right),
  \end{align*}
where
\begin{align*}
    A_6 = \frac{1}{(1+\taubart)}\left[\sigma^2 \taubarmaxt^2 N K L\left(1+\frac{\alpha \taubarmaxt L^2}{(\rho \delta+L^2)(1+\taubart)}\right) + \frac{(L^2 + \rho \delta)\sigma^2}{L}\bar{d}_T\right.\\
    +(K-1)N L(\beta  +\sigma^2/K)\Big]+ LKN \tau_{\max, T}\sigma \sqrt{\beta + \frac{\sigma^2}{KN}}.
\end{align*}
Now we have completed the proof of \Cref{theorem:noncvx_general}. The following corollary is the same as \Cref{theorem:noncvx}, which holds under the assumption of bounded number of inactive rounds.
\begin{corollary}[Bounded number of inactive rounds]
\label{corr:noncvx}
Assume that Assumptions \ref{assumption:smooth}, \ref{assumption:noise},  and \ref{assumption: hessian} to \ref{assumption: bounded dissimilarity for noncvx} hold.   Further assume that the device availability sequence $\tau(t,i)$ satisfies 
 \Cref{assumption:constant delay} and $\tau(1,i)=0$ for all $i \in [N]$. By using a learning rate $\eta = \sqrt{\frac{N}{KTL(1+\bar{\nu})}}$, for $T\geq \max \{32\alpha LNK, 16L N K, \frac{8 KN\nu_{\max}^2(L^2+\rho \delta)}{L}\}$, after $T-1$ communication rounds, \algoname[] satisfies:
 %=\frac{1}{T}\sum_{t=1}^T \EE{\norm{\nabla f(\w{t})}^2 }
\begin{align*}
    \min_{1\leq t \leq T} \EE[\xi]{\norm{\nabla f({w}_t)}^2} = \mathcal{O}\left(\sqrt{\frac{(1+\bar{\nu})L}{TKN}}(f(\w{1}) - f^*+\sigma^2)+\frac{A_4+A_5}{T}\right),
\end{align*}
where $f^*$ is the optimal value, and:
\begin{align*}
    A_4 &= N K L \left(\alpha \sigma^2\bar{\nu}+\frac{\sigma^2\nu_{\max}}{\sqrt{KN}}+\sigma\nu_{\max}\sqrt{\beta}\right) +\frac{(L^2+ \rho\delta) \sigma^2 \nu_{\max}}{L},\\
    A_5 &= \frac{(K-1)N L(\beta  +\sigma^2/K)}{\bar{\nu}+1}.
\end{align*}
\end{corollary}
% \begin{corollary}[Bounded number of inactive rounds]\label{corr:noncvx}
% Let Assumptions \ref{assumption:smooth}, \ref{assumption:noise},  and \ref{assumption: hessian} to \ref{assumption: bounded dissimilarity for noncvx} hold.   For any sequence of inactive rounds satisfying \Cref{assumption:constant delay} and $\mathcal{A}(1)=[N]$, we establish the following convergence rate. Set $\eta = \sqrt{\frac{N}{KTL(1+\bar{\nu})}}$. Set $T\geq \max \{32\alpha LNK, 16L N K, \frac{KN\nu_{\max}^2(L^2+\rho \delta)}{L}\}$. Taking expectation over local data sampling, after communication rounds $1, \cdots,T-1$, \algoname satisfies:
% \begin{align*}
%     \frac{1}{T}\sum_{t=1}^T \EE{\norm{\nabla f(\w{t})}^2 } = \mathcal{O}\left(\sqrt{\frac{(1+\bar{\nu})L}{TKN}}(f(\w{1}) - f^*+\sigma^2)+\frac{A_4+A_5}{T}\right)
% \end{align*}
% where $f^*$ is the optimal value, and:
% \begin{align*}
%     A_4 &= N K L \left(\alpha \sigma^2\bar{\nu}+\frac{\sigma^2\nu_{\max}}{\sqrt{KN}}+\sigma\nu_{\max}\sqrt{\beta}\right) +\frac{(L^2+ \rho\delta) \sigma^2 \overline{\nu^2}}{L(\bar{\nu}+1)}\\
%     A_5 &= \frac{(K-1)N L(\beta  +\sigma^2/K)}{\bar{\nu}+1}
% \end{align*}
% \end{corollary}
\begin{proof}
We first show that \eqref{ineq:eta1} to \eqref{ineq:eta3} hold. \eqref{ineq:eta1} holds because when $T \geq \frac{8KN\nu_{\max}^2(L^2+\rho\delta)}{L}$,
\begin{align*}
    \teta[] \leq \sqrt{\frac{1}{2(4L^2+\rho\delta)\nu_{\max}^2}}.
\end{align*}
Also, \eqref{ineq:eta2} holds when $T\geq 16LNK$. \eqref{ineq:eta3} holds when $T\geq 32\alpha LNK$. Therefore, \eqref{ineq:ergodic mean} holds and it can be further simplified as
\begin{align}
    \begin{aligned}
      \frac{1}{T}\sum_{t=1}^T \EE{\norm{\nabla f(\w{t})}^2 }&  \leq \frac{4}{T\teta[]}(f(\w{1}) - f^*) + \frac{4L(1+\bar{\nu})\sigma^2}{KN}\teta[] \\
    &\quad + 4\left[H_1\left(\frac{1}{N}\nsum \nu_i^2\right) + H_2 \bar{\nu} + H_3\right]\teta[]^2\\
    &\quad+ 32\sigma^2 \bar{\nu}^2 L^2\teta[]^2 + 512\alpha \sigma^2\bar{\nu} ^3 \nu_{\max}^2 L^4 \teta[]^4.
    \end{aligned}\label{ineq:ergodic mean2}
\end{align}
Since $T\geq \frac{8KN\nu_{\max}^2(L^2+\rho\delta)}{L}$,
\begin{align*}
    \teta[]^2 \nu_{\max}^2=\mathcal{O}\left( \frac{1}{(1+\bar{\nu})(L^2 + \rho \delta)}\right).
\end{align*}   
Therefore,
\begin{align*}
    \alpha \sigma^2 \bar{\nu}^3 \nu_{\max}^2 L^4 \teta[]^4 = \mathcal{O}\left(\frac{\alpha \sigma^2 \bar{\nu}^2 L^4}{L^2+\rho \delta} \teta[]^2\right) = \mathcal{O}\left(\frac{\alpha \sigma^2 \bar{\nu} LKN}{T} \right).
\end{align*}
Besides, 
\begin{align*}
    32\sigma^2\bar{\nu^2}L^2\teta[]^2 &=\mathcal{O}\left(\frac{LKN\bar{\nu}\sigma^2}{T}\right),\\
    H_1 \left(\frac{1}{N}\nsum \nu_i^2\right)&=\mathcal{O}\left(\frac{L^2+\rho\sigma^2}{L(\bar{\nu}+1)}\left(\frac{1}{N}\nsum \nu_i^2\right)\right) = \mathcal{O}\left(\frac{(L^2+\rho\sigma^2)\nu_{\max}}{L}\right), \\
    H_2\bar{\nu}\teta[]^2&=\mathcal{O}\left(\frac{LKN\nu_{\max}}{T}\left(\sigma\sqrt{\beta}+\sigma^2/\sqrt{KN}\right)\right),\\
    H_3\teta[]^2 &= \mathcal{O}\left(\frac{(K-1)NL(\beta+\sigma^2/K)}{1+\bar{\nu}}\right).
\end{align*}
Now we have completed the proof of \Cref{corr:noncvx}.
 \end{proof}
\begin{comment}
\begin{corollary}
If there exists a constant upper bound $\nu_i$ such that $\tau(t,i)\leq \nu_i$ for all $i \in [N]$ and $t\ge 1$, we have:
\begin{align*}
    \frac{1}{T}\sum_{t=1}^T \EE{\norm{\nabla f(\w{t})}^2 } = \mathcal{O}\left(\sqrt{\frac{(1+\taubart)L}{TKN}}(f(\w{1}) - f^*+\sigma^2)+\frac{A_4+A_5}{T}\right)
\end{align*}
where $\bar{\nu}=\frac{1}{N}\nsum \nu_i$ and $\nu_{\max}=\max_{i\in [N]}\nu_i$.
where 
\begin{align*}
    A_4 &= N K L \sigma^2(\bar{\nu}^3+\nu_{\max}/\sqrt{KN}) +(L^2+ \rho\delta) \sigma^2 \overline{\nu^2}/L\\
    A_5 &= (K-1)N L^2(\beta  +\sigma^2/K)
\end{align*}
\end{corollary}
\end{comment}
\begin{comment}
\begin{align*}
     H_1 = \frac{(4L^2 + \rho \delta)\sigma^2 }{KN}\\
     H_2 = 2L^2  \lmaxt \sigma \sqrt{\beta + \frac{\sigma^2}{KN}} \\
     H_3 = \frac{4(K-1)L^2(2\beta +  \sigma^2/K)}{K}
\end{align*}
\end{comment}

\section{Proofs in Section~\ref{sec:case_study}}

Our analysis is based on the observation that $\tau(t,i)$ is a truncated geometric random variable with success probability $p_i$ for the Bernoulli participation model.

\begin{lemma}
\label{lem:prop:tau}
For i.i.d. Bernoulli participation model with participation probabilities $\{p_i\}$, we have $\tau(t,i)$ is a truncated geometric random variable taking values in $\{0, 1, \dots, t-1\}$.
\end{lemma}

\begin{proof}
Notice that for $ k < t$, the event $\{ \tau(t,i) \geq k \}$ is equivalent to the event that device $i$ is not active at round $t, t-1, \dots, t-k+1$, which means
\[
    \PP ({\tau(t,i) \geq k}) = (1-p_i)^k, \text{ for } k < t.
\]
Also, since we have assumed that all devices participate at the first round, we have $\PP ({\tau(t,i) \geq t}) = 0$. 
\end{proof}

\subsection{Proof of Theorem~\ref{thm:bernoulli_delay}}
\begin{proof}
By Lemma~\ref{lem:prop:tau}, we know that for all $k$,
\[
    \PP ({\tau(t,i) \geq k}) \leq (1-p_i)^k.
\]
For any fixed $0 < \delta_t < 1$, by setting $k = \lceil \frac{\log(1/\delta_t)}{\log (1/(1-p_i))} \rceil$, we have  $\PP ({\tau(t,i) \geq k}) \leq \delta_t$. This means with probability at least $1-\delta_t$, we have
\[
    \tau(t,i) \leq 1 + \frac{\log (1/\delta_t)}{ \log (1/(1-p_i))}.
\]
By choosing $\delta_t = \frac{6}{\pi^2} \cdot \frac{\delta}{t^2 N}$ and taking union bound over all $ t \geq 1$ and $i \in [N]$, we have with probability at least $1-\delta$, 
\[
    \tau(t,i) \leq 1 + \frac{\log ( \frac{\pi^2} {6}\cdot \frac{t^2 N}{\delta} )}{ \log (1/(1-p_i))} = 1 + \frac{1}{\log (1/(1-p_{i}))} \Big[ \log(\frac{\pi^2}{6\delta}) + 2\log t + \log N \Big].
\]
Using the inequality that $\frac{1}{\log (1/(1-p_{i}))} \leq 1/p_i$ (which is tight when $p_i \approx 0$), we further have 
\[
    \tau(t,i) \leq  1 + \frac{1}{p_i} \Big( 2\log t + \log N + \log \frac{\pi^2}{6\delta} \Big)  = \cO \Big( \frac{1}{p_i} (1 + \log (Nt/\delta)) \Big) .
\]

For \Cref{assumption:delay} to hold, We need to find a $t_0$ such that for all $t$,
\[
    1 + \frac{1}{p_{min}} \Big[ \log(\frac{\pi^2}{6\delta}) + 2\log t + \log N \Big] \leq t_0 + \frac{t}{b}.
\]
Solving this inequality, we get
\[
    t_0 \geq  \frac{2}{p_{min}} \Big (\log \frac{2b}{p_{min}}  - 1 \Big )  + \frac{1}{p_{min}} \log \frac{\pi^2 N}{6\delta} + 1 ,
\]
which is satisfied if 
\[
 t_0 \geq C \frac{1}{p_{min}} \log\frac{bN}{p_{min} \delta}
\]
for an absolute constant $C>0$. 
\end{proof}

\subsection{Proof of Theorem~\ref{thm:bernoulli_tau_T}}
\begin{proof}
By Lemma~\ref{lem:prop:tau}, we have 
\[
    \EE{\tau(t,i)} = \sum_{k =1 }^\infty \PP \Big( \tau(t,i) \geq k \Big) = \sum_{k =1 }^{t-1} (1-p_i)^k \leq \frac{1}{p_i}.
\]
Therefore, we can upper bound the expectation of $\bar{\tau}_T = \frac{1}{N(T-1)} \sum_{t=1}^{T-1} \sum _{i=1}^N \tau(t,i)$ as
\[
\EE{\bar{\tau}_T } \leq \frac{1}{N} \sum_{i=1}^N \frac{1}{p_i}.
\]

Furthermore, we know that $\tau(t,i)$ is sub-exponential with $\| \tau(t,i) \|_{\psi_1} \leq C_1 \frac{1}{p_i} $ \citep{vershynin2018high}. Then we know that $\bar{\tau}_T - \EE{\bar{\tau}_T}$ is sub-exponential with  $\|\bar{\tau}_T - \EE{\bar{\tau}_T} \|_{\psi_1} \leq C_2 \frac{1}{N} \sum_{i=1}^N \frac{1}{p_i}$. Therefore, by Bernstein's inequality \citep{vershynin2018high}, we have with probability at least $1-\delta$,
\[
    \bar{\tau}_T - \EE{\bar{\tau}_T} \leq C_3 \Big( \frac{1}{N} \sum_{i=1}^N \frac{1}{p_i} \Big) \cdot \max \Big( \log\frac{1}{\delta}, 1 \Big).
\]
We conclude that 
\[
    \bar\tau_T \leq \Big( \frac{1}{N} \sum_{i=1}^N \frac{1}{p_i} \Big) \cdot \cO \Big(1 + \log\frac{1}{\delta} \Big)  .
\]
Remark: $C_1, C_2, C_3 > 0$ are absolute constants.
\end{proof}

\subsection{Additional Discussion on the Expected Waiting Time}
\label{app:fedavg:expected_wait}

To accomplish a single global update, algorithms such as FedAvg and SCAFFOLD need to receive the local updates from a randomly sampled subset $\cS$ of devices. In our setting, the server needs to wait for a few rounds so that all devices in $\cS$ become active and return the computation result during the these rounds. For i.i.d. Bernoulli participation model, the expected rounds for the $i$-th device to become active is $1/p_i$. Therefore, the expected rounds for all the devices in $\cS$ to become active is at least $ \frac{1}{\min \{p_i | i \in \cS  \}} $. 

Denote by $T({\cS})$ the expected rounds for all the devices in $\cS$ to become active, under the setting that $\cS$ is randomly selected from $N$ devices without replacement, we have
\[
    \EE[\cS] {T({\cS})}  \geq \frac{1}{p_{min} } \PP _{\cS} ( \text{ the device with minimal $p_i$ is selected  } ) = \frac{S}{N}\frac{1}{p_{min} }.
\]

\section{Proof of Proposition~\ref{thm:lower}}\label{app:lower}

\begin{proof}
This lower bound actually holds even for centralized algorithms. We first show that a lower bound for centralized optimization implies a lower bound on our case. We then analyze the lower bound for the standard optimization setup.

\paragraph{Number of gradient evaluations.} Assume that we have $N$ devices, and each device respond every $2\tau$ rounds of communication.  Then by definition $\bar{\tau}_T = \Theta(\tau)$, and only $ \Theta(NKT/\bar{\tau}_T)$ stochastic gradients are evaluated. Hence, the theorem is proved if we can show that no algorithms can output a (potentially random) $w_T$ within $\mathcal{T}$ stochastic gradients evaluations satisfying
\begin{align*}
    \mathbb{E}[f(w_T) - f(w^*)] \ge c_0 \frac{ \sigma^2}{\mu \mathcal{T}}. 
\end{align*}

\paragraph{Uncontrained stochastic optimization lower bound.} The constrained version of the above inequality has been formally proved by multiple works (e.g.\cite{agarwal2009information, nemirovskij1983problem}). These results do not readily applied as we did not assume the function to be Lipschitz continuous. The smooth but not Liptschitz continuous case is a \emph{folklore} in optimization community (e.g. see \cite{ghadimi2012optimal} equation 1.3). We provide a short proof \emph{for completeness} following \citep{bubeck2013bandits, zhang2020adaptive}.

For a given $\mu \in (0, 1], \sigma > 1$, we consider the following simple one-dimensional function class parameterized by $b$:
\begin{equation}\label{eqn:function-class}
    \min_{x } \cbr*{f_{b}(x):= \tfrac{\mu}{2}(x - b)^2 }\,, \text{ for } b \in [0,1/2]\,.
\end{equation}

% where 
% \begin{align*}
%     f_b(x) =   
%     \begin{cases}
%       \tfrac{1}{2}(x - b)^2 & \text{if $|x-b| \le 1$}\\
%       |x - b| - 1/2 & \text{if $|x - b| > 1$}\\
%     \end{cases}. 
% \end{align*}

Note that $f_b$ is $1$-smooth and $\mu$-strongly convex.

Also suppose that for $b \in [0,1/2]$ the stochastic gradients are of the form:
\begin{equation}\label{eqn:noise-class}
    g(x) \sim \nabla f_b(x) + \chi_{b}\,, \E[g(x)] = \nabla f_b(x)\,, \text{ and } \E[|g(x) - \nabla f_b(x)|^2] \leq  \sigma^2.
\end{equation}
Note that the function class \eqref{eqn:function-class} has optimum value $f_b(b) = 0$. Thus, we want to prove the following:

\begin{theorem}
There exists a distribution $\chi_{b}$ such that the stochastic gradients satisfy \eqref{eqn:noise-class}. Further, for any (possibly randomized) algorithm $\Ac$, define $\Ac_k\rbr*{f_b + \chi_{b}}$ to be the output of the algorithm $\Ac$ after $k$ queries to the stochastic gradient $g(x)$, then:
\[
    \max_{b \in [0,1/2]} \E[f_b(\Ac_k(f_b + \chi_{b}))] \geq \frac{c_0 \sigma^2 }{k\mu}\,.
\]
\end{theorem}
We assume the algorithm of interest is stable, i.e. $\max_{b \in [0,1/2]} \E[f_b(\Ac_k(f_b + \chi_{b}))] \le \infty.$ Otherwise, the theorem is true.

Let $\Ac_k(f_b + \chi_{b})$ denote the output of any possibly randomized algorithm $\Ac$ after processing $k$ stochastic gradients of the function $f_b$ (with noise drawn i.i.d. from distribution $\chi_{b}$). Similarly, let $\Dc_k(f_b + \chi_{b})$ denote the output of a \emph{deterministic} algorithm after processing the $k$ stochastic gradients. Then from Yao's minimax principle we know that for any fixed distribution $\Bc$ over $[0,1/2]$,
\[
    \min_{\Ac}\max_{b \in [0,1/2]}\E_{\Ac}\sbr*{ \E_{\chi_b}f_b(\Ac_k(f_b + \chi_{b}))} \geq \min_{\Dc} \E_{b \sim \Bc}\sbr*{\E_{\chi_b}f_b(\Dc_k(f_b + \chi_{b}))}\,.
\]
Here we denote $\E_{\Ac}$ to be expectation over the randomness of the algorithm $\Ac$ and $\E_{\chi_b}$ to be over the stochasticity of the the noise distribution $\chi_b$.
Hence, we only have to analyze deterministic algorithms to establish the lower-bound. Further, since $\Dc_k$ is deterministic, for any \emph{bijective} transformation $h$ which transforms the stochastic gradients, there exists a deterministic algorithm $\tDc$ such that $\tDc_k(h(f_b + \chi_{b})) = \Dc_k(f_b + \chi_{b})$. This implies that for any bijective transformation $h(\cdot)$ of the gradients:
\[
    \min_{\Dc} \E_{b \sim \Bc}\sbr*{\E_{\chi_b}f_b(\Dc_k(f_b + \chi_{b}))} =   \min_{\Dc} \E_{b \sim \Bc}\sbr*{\E_{\chi_b} f_b(\tDc_k(h(f_b + \chi_{b})))}\,.
\]
In this rest of the proof, we will try obtain a lower bound for the right hand side above.

We now describe our construction of the three quantities to be defined: the problem distribution $\Bc$, the noise distribution $\chi_{b}$, and the bijective mapping $h(\cdot)$. All of our definitions are parameterized by $\epsilon \in (0,1/8]$ (which represents the desired target accuracy). We will pick $\epsilon$ to be a fixed constant which depends on the problem parameters (e.g. $k$) and should be thought of as being small.
\begin{itemize}
    \item{Problem distribution:} $\Bc$ picks $b_0 = 2\epsilon \sigma/\mu$ or $b_1 = \epsilon \sigma/\mu$ at random i.e. $\nu \in \{0,1\}$ is chosen by an unbiased coin toss and then we pick
    \begin{equation}\label{eqn:lower-prob}
        b_{\nu} = (2-\nu)\epsilon \frac{\sigma}{\mu} \,.
    \end{equation}
    \item{Noise distribution:} Define a constant $\gamma = 4\epsilon / \sigma$ and $p_\nu = (16\epsilon^2 - 8\nu \epsilon^2) $. Simple computations verify that $\gamma \in (0,1/2]$ and that
    \[
        p_\nu =  (4-2\nu)\rbr*{4 \epsilon^2} \in (0,1)\,.
    \]
    Then, for a given $\nu \in \{0,1\}$ the stochastic gradient $g(x)$ is defined as
    \begin{equation}\label{eqn:lower-noise}
        g(x) = \begin{cases}
                \mu x - \frac{1}{2\gamma} &\text{ with prob. } p_\nu\,,\\
                \mu x &\text{ with prob. } 1 - p_\nu\,.
            \end{cases}
    \end{equation}
    To see that we have the correct gradient in expectation verify that 
    \[
        \E[g(x)] = \mu x - \frac{ p_\nu}{2\gamma} =  \mu x -  \mu b_\nu = \nabla f_{b_{\nu}}(x)\,.
    \]
    Next to bound the variance of $g(x)$. We see that
    \[
        \E[\abs{g(x) -  \nabla f_b (x)}^2] \leq p_\nu \rbr*{ \frac{1}{2\gamma}}^{2} + (1- p_\nu ) \mu ^2 b_\nu^{2} \leq \sigma^2\,.
    \]
    Thus $g(x)$ defined in \eqref{eqn:lower-noise} satisfies condition \eqref{eqn:noise-class}.
    \item{Bijective mapping:} Note that here the only unknown variable is $\nu$ which only affects $p_\nu$. Thus the mapping is bijective as long as the \emph{frequencies} of the events are preserved. Hence given a stochastic gradient $g(x_i)$ the mapping we use is:
    \begin{equation}\label{eqn:lower-map}
        h(g(x_i)) = \begin{cases}
                        0 &\text{ if } g(x_i) =   \mu x_i \,,\\
                        1 &\text{ otherwise.}
                        \end{cases}
    \end{equation}
    % To see that the mapping is bijective, just observe that given $x_i$ and $h(g(x_i))$ we can reconstruct $g(x_i)$. Since $\Dc$ is a deterministic function, we can then compute $x_{i+1}$ (the point at which $\Dc$ queries for the next gradient). This way we can inductively reconstruct the entire sequences of $\{x_k\}$ and $\{g(x_k)\}$. % Praneeth: description seems unnecessary and might even add confusion.
\end{itemize}

%The rest of the proof follows exactly the same lines as Appendix

Given the definitions above, the output of algorithm $\Dc_k$ is thus simply a function of $k$ i.i.d. samples drawn from the Bernoulli distribution with parameter $p_\nu$ (which is denoted by $\ber(p_\nu)$). We now show how achieving a small optimization error implies being able to guess the value of $\nu$.

\begin{lemma}\label{lem:guessing-game}
Suppose we are given problem and noise distributions defined as in \eqref{eqn:lower-prob} and \eqref{eqn:lower-noise}, and an bijective mapping $h(\cdot)$ as in \eqref{eqn:lower-map}. Further suppose that there is a deterministic algorithm $\Dc_k$ whose output after processing $k$ stochastic gradients satisfies
\[
    \E_{b \sim \Bc}\sbr*{\E_{\chi_b} f_b(\Dc_k(h(f_b + \chi_{b})))} < \frac{\epsilon^2\sigma^2}{ 64\mu }\,.
\]
Then, there exists a deterministic function $\tDc_k$ which given $k$ independent samples of $\ber(p_\nu)$ outputs $\nu' = \tDc_k(\ber(p_\nu)) \in \{0,1\}$ such that
\[
    \Pr\sbr*{\tDc_k(\ber(p_\nu)) = \nu} \geq \frac{3}{4}\,.
\]
\end{lemma}
\begin{proof}
Suppose that we are given access to $k$ samples of $\ber(p_\nu)$. Use these $k$ samples as the input $h(f_b + \chi_{b}))$ to the procedure $\Dc_k$ (this is valid as previously discussed), and let the output of $\Dc_k$ be $x^{(\nu)}_k$. The assumption in the lemma states that
\[
    \E_\nu\sbr*{\E_{\chi_b} \frac{\mu}{2} \abs{x^{(\nu)}_k - b_\nu}^2} < \frac{ \epsilon^2\sigma^2}{64 \mu} \text{, which implies that } \E_{\chi_b}\abs{x^{(\nu)}_k - b_\nu}^2 < \frac{ \epsilon^2\sigma^2}{16  \mu^2} \text{ almost surely.}
\]
Then, using Markov's inequality (and then taking square-roots on both sides) gives
\[
    \Pr\sbr*{\abs{x^{(\nu)}_k - b_\nu} \geq \frac{\epsilon \sigma}{2 \mu}} \leq \frac{1}{4}\,.
\]
Consider a simple procedure $\tDc_k$ which outputs $\nu' = 0$ if $x^{(\nu)}_k \geq \frac{3\epsilon  \sigma }{2 \mu}$, and $\nu' = 1$ otherwise. Recall that $\abs{b_0 - b_1} = \epsilon \sigma/\mu$ with $b_0 = 2\epsilon  \sigma/\mu$ and $b_1 = \epsilon \sigma/\mu$. With probability $\frac{3}{4}$, $\abs{x^{(\nu)}_k - b_\nu} < \frac{\epsilon}{2}\sigma/\mu$ and hence the output $\nu'$ is correct.
\end{proof}

Lemma~\ref{lem:guessing-game} shows that if the optimization error of $\Dc_k$ is small, there exists a procedure $\tDc_k$ which distinguishes between the Bernoulli distributions with parameters $p_0$ and $p_1$ using $k$ samples. To argue that the optimization error is large, one simply has to argue that a large number of samples are required to distinguish between $\ber(p_0)$ and $\ber(p_1)$.
\begin{lemma}\label{lem:sample-complex}
    For any deterministic procedure $\tDc_k(\ber(p_\nu))$ which processes $k$ samples of $\ber(p_\nu)$ and outputs $\nu'$
    \[
        \Pr\sbr*{\nu' = \nu} \leq \frac{1}{2} + \sqrt{ k  \rbr*{4\epsilon}^{2}}\,.
    \]
\end{lemma}
\begin{proof}

Here it would be convenient to make the dependence on the samples explicitly. Denote $\ss^{(\nu)}_k = \rbr*{s_1^{(\nu)}}$
${,\dots,s_k^{(\nu)}} \in \{0,1\}^k$ to be the $k$ samples drawn from $\ber(p_\nu)$ and denote the output as $\nu' = \tDc(\ss^{(\nu)}_k)$. With some slight abuse of notation where we use the same symbols to denote the realization and their distributions, we have:
\[
    \Pr\sbr*{\tDc(\ss^{(\nu)}_k) = \nu} = \frac{1}{2}\Pr\sbr*{\tDc(\ss^{(1)}_k) = 1} + \frac{1}{2}\Pr\sbr*{\tDc(\ss^{(0)}_k) = 0} = \frac{1}{2} + \frac{1}{2}\E\sbr*{\tDc(\ss^{(1)}_k) - \tDc(\ss^{(0)}_k)}\,.
\]
Next using Pinsker's inequality we can upper bound the right hand side as:
\[\E\sbr*{\tDc(\ss^{(1)}_k) - \tDc(\ss^{(0)}_k)} \leq \abs*{\tDc(\ss^{(1)}_k) - \tDc(\ss^{(0)}_k)}_{TV} \leq \sqrt{\frac{1}{2}\KL\rbr*{\tDc\rbr*{\ss^{(1)}_k}, \tDc\rbr*{\ss^{(0)}_k}}}\,,    
\]
where $\abs{\cdot}_{TV}$ denotes the total-variation distance and $\KL(\cdot, \cdot)$ denotes the KL-divergence. Recall two properties of KL-divergence: i) for a product measures defined over the same measurable space $(p_1,\dots,p_k)$ and $(q_1,\dots,q_k)$, 
\[
    \KL((p_1,\dots,p_k), (q_1,\dots,q_k)) = \sum_{i=1}^k \KL(p_i,q_i)\,,
\] and ii) for any deterministic function $\tDc$, 
\[
    \KL(p,q) \geq \KL(\tDc(p), \tDc(q))\,.
\]
Thus, we can simplify as
\begin{align*}
    \Pr\sbr*{\tDc(\ss^{(\nu)}_k) = \nu} &\leq  \frac{1}{2} +  \sqrt{\frac{k}{8}\KL\rbr*{\ber(p_1), \ber(p_0)}}\\
    &\leq  \frac{1}{2} + \sqrt{\frac{k}{8}\frac{(p_0 - p_1)^2}{p_0 ( 1- p_0)}}\\
    &\leq  \frac{1}{2} + \sqrt{  k  \rbr*{4\epsilon}^2}
\end{align*}
\end{proof}
If we pick $\epsilon$ to be
\[
    \epsilon = \frac{1}{16  k^{1/2}}\,,
\]
we have that
\[
    \frac{1}{2} + \sqrt{ k \rbr*{4\epsilon}^2} = \frac{3}{4}\,.
\]
Given Lemmas~\ref{lem:guessing-game} and~\ref{lem:sample-complex}, this implies that for the above choice of $\epsilon$,
\[
  \E_{b \sim \Bc}\sbr*{\E_{\chi_b} f_b(\Dc_k(h(f_b + \chi_{b})))} \geq \epsilon^2\frac{\sigma^2}{64 \mu}  = \frac{\sigma^2}{\mu 2^{14} k}\,.
\]

\end{proof}

\section{Proof of Proposition~\ref{thm:lower-noncvx}}
This proof is almost the same as the proof in Appendix~\ref{app:lower}, except that we use the result from Theorem 3 of \citep{arjevani2019lower} instead of from \citep{bubeck2013bandits}.

\end{document}